\documentclass{article}

\usepackage[main,final]{neurips_2026}

\usepackage[utf8]{inputenc} 
\usepackage[T1]{fontenc}    
\usepackage{hyperref}       
\usepackage{url}            
\usepackage{booktabs}       
\usepackage{amsfonts}       
\usepackage{nicefrac}       
\usepackage{microtype}      
\usepackage{xcolor}         
\usepackage{fontawesome5}    
\newcommand{\repo}{\href{https://github.com/grafmaxi/bandit_changepoint}{\faGithub\,\texttt{bandit\_changepoint}}}
\usepackage{amsmath}
\usepackage{graphicx} 
\usepackage{amssymb}
\usepackage{amsthm}
\usepackage{mathtools}
\usepackage{algorithm}
\usepackage{algpseudocode}
\definecolor{commentgray}{gray}{0.35}
\algrenewcommand\algorithmiccomment[1]{\hfill{\small\color{commentgray}$\triangleright$ #1}}
\usepackage{natbib}
\usepackage{cleveref}
\usepackage{bbm}
\usepackage{enumitem}

\usepackage{subcaption}

\newtheorem{theorem}{Theorem}

\newtheorem{lemma}{Lemma}
\newtheorem{remark}{Remark}
\newtheorem{corollary}{Corollary}

\newtheorem{definition}{Definition}

\DeclareMathOperator{\kl}{kl}

\DeclareMathOperator{\KL}{KL}

\title{The Sample Complexity of Multiple Change Point Identification under Bandit Feedback}

\author{
  Maximilian Graf~\textsuperscript\dag\\
  Institut für Mathematik, Universität Potsdam, Potsdam, Germany\\
  \texttt{maximilian.graf@warwick.ac.uk} \\
  \And
  Victor Thuot~\textsuperscript\dag\\
  INRAE, Mistea, Institut Agro, Univ Montpellier, Montpellier, France\thanks{Work done at INRAE, Mistea, Institut Agro, Univ Montpellier. Author now at UMPA, ENS Lyon.} \\
  \texttt{victor.thuot@ens-lyon.fr}
  \\[0.5em]
  \normalsize $^\dag$ Equal contribution.
}

\begin{document}

\maketitle

\begin{abstract}
We study multiple change point localization under bandit feedback. An unknown piecewise-constant function on a compact interval can be queried sequentially at adaptively chosen inputs, and each query returns a noisy evaluation of the function. The goal is to identify a prescribed number of discontinuities, known as change points, within a target precision $\eta$ and confidence level $1-\delta$, while using as few samples as possible. We propose an adaptive algorithm that first detects intervals likely to contain change points and then refines their locations to precision $\eta$. We establish non-asymptotic upper bounds on its sample budget, together with corresponding lower bounds.
Prior work shows that jump magnitudes alone determine the asymptotic sample complexity as $\delta\to 0$. We reveal that this picture is incomplete beyond this regime. We demonstrate, both empirically\footnote{Code to reproduce all experiments is available at \repo{}.} and theoretically, that for general $\delta$ and $\eta$, the complexity is jointly governed by the jumps and the relative positions of the change points.
\end{abstract}

\section{Introduction}

Change point detection -- identifying the thresholds at which a system's
behavior changes abruptly -- is a central problem in statistics and machine
learning, with decades of theoretical development~\citep{lai2001sequential,yu2020review,truong2020selective}.
Many practical problems fall within the scope of change point detection.
In materials science, one seeks the control parameters (e.g., temperature and
pressure) at which a material undergoes a phase
transition~\citep{nikolaev2016autonomy,sebastian2011dielectric}; in queuing
and processing systems, one aims to detect abrupt changes in waiting times or
throughput~\citep{hung2007modeling}; in machine learning, one studies the
sensitivity of black-box models to hyperparameter
changes~\citep{lan2009change,hayashi2019active}.
In all these settings, evaluations are costly or time-consuming, and the input
domain is continuous, so adaptive and sequential sampling strategies are
particularly relevant. Moreover, reliable guarantees are required in practice:
both the estimation precision and the probability of error must be controlled.
This motivates a bandit formulation~\citep{lattimore2019bandit}, in which a learner queries the function at
adaptively chosen inputs and aims to identify a prescribed number of change
points, at a given precision $\eta$ and confidence level $1-\delta$, while
minimizing the total number of evaluations.

In this manuscript, we consider the multiple change point detection problem in
the bandit setting, as introduced in~\citet{hayashi2019active} and recently
studied in~\citet{pmlr-v267-lazzaro25a}.
As in these works, we
focus on the one-dimensional piecewise-constant model. Although many practical problems may require more involved settings and
notions of change points, this model is the canonical theoretical baseline in the passive literature~\citep{niu2016multiple,verzelen2023optimal}, and a natural starting point for the bandit setting, from which extensions to piecewise-linear, polynomial, or multivariate models can
be built~\citep{truong2020selective,suk2022tracking,manegueu2021generalized}.
To fix notation before discussing related work (a more detailed presentation is deferred to Section~\ref{sec:setting}), we consider a piecewise-constant function $f:[0,1]\to\mathbb{R}$ with $m$ ($m\geq 1$) change points at positions $0<x_1^*<\cdots<x_m^*<1$, so that $f$ is constant on $[x_{l-1}^*,x_l^*)$ for $l\in\{1,\dots,m+1\}$, with $x_0^*=0$ and $x_{m+1}^*=1$. The learner sequentially and adaptively queries points $x_t\in[0,1]$ and observes $y_t=f(x_t)+\varepsilon_t$, where $\varepsilon_t$ is subGaussian noise; in particular, each new query may depend on past observations. Given a target number of change points $N\leq m$, a precision $\eta\in(0,1)$, and an error rate $\delta\in(0,1)$, the goal is to output $N$ estimates within distance $\eta$ of true change points with probability at least $1-\delta$. The sample complexity, denoted by $\mathcal{T}$, is the total number of evaluations required to achieve this guarantee.

In~\citet{pmlr-v267-lazzaro25a}, the authors propose \texttt{MCPI} (\emph{Multiple Change Point Identification}), a method based on the stopping conditions from~\citet{garivier2016optimal}, and prove its asymptotic optimality in expectation. {Their analysis is conducted in a discrete action space with $K$ arms,
which corresponds to a precision of $\eta \simeq 1/K$ in our continuous
setting}. {They provide non-asymptotic upper bounds, which contain at least
a linear\footnote{Additionally, a polynomial dependence on $K$ (namely, $K^4$), independent of $\delta$, appears to be hidden in the non-explicit terms $T'_1(v)$ of their Prop.~5.6.} dependence on $K$, while the focus of their analysis lies on the asymptotic regime $\delta\to 0$. Indeed, \texttt{MCPI} identifies $N$ change points
with expected sample complexity satisfying $\lim_{\delta \to 0}
\frac{\mathbb{E}[\mathcal{T}]}{\log(1/\delta)} = 8 H_{\mathrm{localize}}^{(N)}$, where}
\begin{equation}\label{eq:H_loc}
	H_{\mathrm{localize}}^{(N)}=\displaystyle \sum_{l=1}^N \frac{1}{\Delta_{(l)}^2}\enspace,
\end{equation}
with $\Delta_i=f(x_i^*)-f(x_{i-1}^*)$ for $i\in\{1,\dots,m\}$ denoting the jump at the $i$-th change point, and $|\Delta_{(1)}|\geq \cdots \geq |\Delta_{(m)}|$ denoting the decreasing rearrangement of jump magnitudes. Thus, the complexity is driven by the $N$ largest jumps in
the asymptotic regime $\delta \to 0$, while dependence on $K$ disappears in the limit. { The analysis of large-scale problems ($K$ large) remains open.}

For comparison, consider the batch change point detection problem, in which the learner evaluates $f$ once at each point of a fixed grid~\citep{lai2001sequential,yu2020review}. In that setting, the relevant difficulty measure is the energy of each change point, defined as $\mathcal{E}_i^2:=s_i \Delta_i^2$, where $s_i$ is the minimal spacing to neighboring change points (see Equation~\eqref{eq:def_energy} in Section~\ref{sec:setting}). One of our contributions is to show that energy also matters in the active setting.  For instance, the guarantees in~\citet{verzelen2023optimal} imply a uniform-sampling detection threshold, and hence a sample-complexity term of order
\begin{equation}\label{def:H_detect}
	H_{\mathrm{detect}}^{(m)}=\max_{i=1,\dots,m} \frac{1}{\mathcal{E}_i^2},
\end{equation}
which is always at least as large as $\tfrac{1}{2}H_{\mathrm{localize}}^{(N)}$; it can in fact be much larger. For instance, consider change points with comparable jump size $\Delta$ and local spacing $s$ (with $s\leq 1/m$). Then $H_{\mathrm{detect}}^{(m)}\asymp \frac{1}{s\Delta^2}$, whereas $H_{\mathrm{localize}}^{(N)}\asymp \frac{N}{\Delta^2}$. When $s$ is small, detection may dominate the budget, showing that the rate~\eqref{eq:H_loc} can miss an important non-asymptotic effect. This suggests a richer interplay between jump magnitudes and local spacings than what is captured by the asymptotic jump-only rate in~\eqref{eq:H_loc}.

We consider a continuous action space that avoids the restrictive assumption $\eta<\min_{i=1}^m s_i$ required by discrete methods such as~\citet{pmlr-v267-lazzaro25a}. In the discrete setting, this induces a linear dependence on $1/\eta$, which can become prohibitive when high localization precision is required. Our algorithm adapts to local spacings $s_i$, which is practically significant when change points may be arbitrarily close and the spacing is unknown. By decomposing detection and localization, we show that the dependence on precision $\eta$ is logarithmic, namely $\log(1/\eta)$, which is exponentially better than the linear dependence on the discretization grid $K\simeq 1/\eta$ inherent to discrete approaches. Deriving guarantees that explicitly capture the optimal dependence on both $\eta$ and $\delta$ in all regimes is an open question and a key motivation for our work.

 We illustrate the energy-based complexity in \textbf{Experiment 1} (Figure~\ref{fig:two_cp}) with two change points ($\Delta_1=-\Delta_2=1$), standard noise, and varying spacing $s$. We compare our Algorithm~\ref{alg:lcp} (\texttt{LCP}) to \texttt{MCPI} from~\citet{pmlr-v267-lazzaro25a}.\footnote{In Appendix~\ref{sec:continuousdiscrete} we show how to transform  our setting into a discrete problem, for comparison with the \texttt{MCPI} Algorithm. } The results validate the theory: as $s$ increases, sample complexity of \texttt{LCP} decreases roughly as $1/s$, matching the prediction $H_{\mathrm{detect}}\asymp 1/(s\Delta^2)$. In contrast, \texttt{MCPI} is insensitive to $s$ because its rate is dominated by a constant term proportional to $1/\eta$. Then, adaptation to local spacings is essential, as it allows us to outperform \texttt{MCPI} for moderate values of $s$.

\begin{figure}[ht]
\centering
\begin{minipage}{0.41\textwidth}
\includegraphics[width = \textwidth]{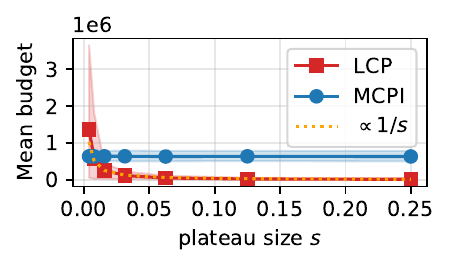}
\end{minipage}
\begin{minipage}{0.48\textwidth}\caption{\textbf{Experiment 1}. We consider two change points. We set $x_1^*\sim\mathcal{U}(0,1/2)$ for each Monte Carlo run and $x_2^*=x_1^*+s$ with $s$ varying and $1000$ Monte Carlo runs. We chose $\Delta_1=-\Delta_2=1$ and compare \texttt{LCP} (Algorithm~\ref{alg:lcp}) with parameters $\eta=2^{-11}$, $\delta = 0.05$ and $\delta_{\mathrm{explore}}=1$ to algorithm \texttt{MCPI}.}\label{fig:two_cp}
\end{minipage}
\end{figure}
\paragraph{Open questions.} These observations motivate the present work and raise the following fundamental questions, relative to the sample complexity of multiple change point identification: (i)~How does the sample complexity depend on jump magnitudes and the local spacings between change points? When is it governed by the energies $\mathcal{E}_i$ rather than only by the jumps $\Delta_i$? (ii)~What is the optimal achievable dependence on the precision $\eta$ and confidence level $\delta$? How does the problem scale for high precision ($\eta$ small)?

\paragraph{Contributions.} In this manuscript, we answer these questions.

\begin{enumerate}
	\item We propose \texttt{LocalizeChangePoints} (Algorithm~\ref{alg:lcp}), a new algorithm for multiple change point identification under bandit feedback. It decomposes the task into (i) detecting intervals likely to contain change points and (ii) sequential refinement by binary search. We prove non-asymptotic sample-complexity guarantees that depend on both jump magnitudes and local spacings through $\mathcal{E}_i$ and $\Delta_i$. In particular, Theorem~\ref{thm:cpl} states that the budget $\mathcal{T}$ for localizing $N$ change points among $m$ is bounded, up to logarithmic terms independent of $\delta$ and $\eta$,\footnote{Typically $\underset{\log}{\leq}$ hides here at most a $\log(H_{\mathrm{detect}}^{(m)})$ multiplicative factor.} as
	\begin{align*}
		& \mathbb{E}[\mathcal{T}] \, \underset{\log}{\leq}\, H_{\mathrm{detect}}^{(m)} + H_{\mathrm{localize}}^{(N)}\log\left(\frac{1}{\delta\eta}\right), \qquad\text{see~\eqref{eq:H_loc},~\eqref{def:H_detect}} \\
		& \mathcal{T} \, \underset{\log}{\leq}\, H_{\mathrm{detect}}^{(m)}\log\left(\frac{1}{\delta}\right) + H_{\mathrm{localize}}^{(N)}\log\left(\frac{1}{\delta\eta}\right)\mathrlap{\qquad\text{w.h.p. }1-\delta.}
	\end{align*}
	\item Additionally, we derive information-theoretic lower bounds (Theorem~\ref{thm:main_LB}). On a local-minimax class of instances,  we prove that the bounds above are rate-optimal in all parameters, including $\delta$ and $\eta$. This shows that the asymptotic rate in~\eqref{eq:H_loc} is overly optimistic.
	\item Finally, we validate our results numerically on synthetic data.
\end{enumerate}
\paragraph{Batch change point detection.}
Change point detection is one of the classical problems in statistics, with a rich literature spanning cumulative-sum (CUSUM) procedures~\citep{hinkley1970inference}, binary segmentation~\citep{scott1974cluster}, wild binary segmentation (WBS)~\citep{fryzlewicz2014wild}, and penalized least-squares methods~\citep{wang2020univariate,verzelen2023optimal}; see~\citet{lai2001sequential,yu2020review,truong2020selective} for comprehensive reviews. Our algorithm is directly inspired by two classical ideas from this literature: multiscale testing as in the foundational work on WBS~\citep{fryzlewicz2014wild}, to identify candidate intervals likely to contain a change point~\citep{kovacs2023seeded,pilliat2023optimal}, and binary search, to localize it to precision $\eta$. It is also well established that the relevant difficulty measure is the energy $\mathcal{E}_i^2 = s_i\Delta_i^2$ of each change point~\citep{yu2020review,verzelen2023optimal}, an observation that we extend to the bandit setting.
The problem has also been studied in multivariate~\citep{wang2018high} and
kernel-based~\citep{arlot2019kernel} settings. The piecewise-constant model has proved a fruitful baseline for broader change point methods~\citep{niu2016multiple}.

A complementary family of methods detects change points by minimizing a
penalized cost over all possible segmentations, with BIC used to select
their number~\citep{zou2014nonparametric,niu2016multiple}. PELT~\citep{killick2012optimal}
solves this problem exactly in $O(n)$ time (under mild conditions) via pruned dynamic programming, and
\citet{zou2014nonparametric} extend the penalized formulation to a nonparametric,
distribution-free setting. Because these methods process the full data
sequence, their main contribution is computational efficiency rather than
sample complexity. Whether such ideas -- or related clustering-based
approaches~\citep{banerjee2026nonparametric} -- can be adapted to the
active-sampling setting is an interesting open question.

\paragraph{Change point detection in piecewise-stationary bandits.}
A related but distinct stream of literature studies piecewise-stationary bandits~\citep{garivier2011upper}, where the reward distributions of a multi-armed bandit change at unknown time steps. Two main tasks are considered: quickest detection, which aims to raise an alarm as soon as possible after a change occurs~\citep{veeravalli2024quickest,zhang2023bandit,xu2023asymptotic}, and regret minimization, where the goal is to adapt to the new environment in order to minimize cumulative regret~\citep{pmlr-v89-cao19a,besson2022efficient,mukherjee2022safety}.
This more established literature illustrates that the piecewise-constant assumption is a natural starting point: \citet{suk2022tracking} go beyond it by targeting only \emph{significant} arm switches, and \citet{manegueu2021generalized} extend the framework further to piecewise-linear and piecewise-polynomial reward regimes.
Our setting differs from this line of work in two fundamental respects: our model is stationary, and the change points are \emph{spatial} features of a fixed function $f$ rather than temporal. Our model is therefore a specific instance of the well-established continuous bandit framework -- see, e.g., unimodal bandits~\citep{combes2014unimodal}, continuous dueling bandits~\citep{kumagai2017regret}, and pure-exploration problems over continuous domains~\citep{bubeck2011pure}.
\paragraph{Active change point detection.}
Active and sequential change point detection studies settings in which the sampling locations are chosen adaptively rather than fixed in advance. Early foundations for this perspective appear in adaptive change point estimation~\citep{hall2003sequential,lan2009change}, where multistage sampling is used to improve localization rates. In the bandit setting, \citet{hayashi2019active} formulates active change point detection explicitly, and recent fixed-budget and fixed-confidence analyses for piecewise-constant bandits~\citep{lazzaro2025fixedbudget,pmlr-v267-lazzaro25a} provide instance-dependent guarantees based on jump magnitudes. In particular, \citet{pmlr-v267-lazzaro25a} combines this formulation with Track-and-Stop ideas from~\citet{garivier2016optimal} and proves asymptotic rates proportional to $\sum_{l=1}^{N}1/\Delta_{(l)}^2$. Our work complements this line by focusing on non-asymptotic guarantees and by highlighting the additional role of spacing between change points.
We highlight the strong connection to other structured bandit problems, such as monotone bandits \citep{cheshire2020influence} or clustering with bandit \citep{yang2024optimal,graf2025clustering}.
\paragraph{Outline.} The rest of the paper is organized as follows. In Section~\ref{sec:setting}, we introduce the formal setting and notation. In Section~\ref{sec:algorithm}, we present our algorithm and its guarantees. In Section~\ref{sec:lower_bounds}, we derive information-theoretic lower bounds. In Section~\ref{sec:experiments}, we illustrate our results numerically on synthetic data. Finally, in Section~\ref{sec:discussion}, we discuss the implications of our results, possible extensions, and limitations. All proofs are deferred to the appendix.

\section{Problem formulation and notation}\label{sec:setting}

A summary of the notation is given in Table~\ref{tab:notation} (Appendix~\ref{sec:notation}). We consider an unknown piecewise-constant function $f:[0,1]\to\mathbb{R}$ with $m\geq 1$ change points at positions $0<x_1^*<\cdots<x_m^*<1$. The change points are the discontinuities of $f$, so $f$ is constant on each interval $[x_{l-1}^*,x_l^*)$ for $l\in\{1,\dots,m+1\}$, with the conventions $x_0^*=0$ and $x_{m+1}^*=1$. Both $f$ and the true number of change points $m$ are unknown to the learner.

We parametrize $f$ as
\begin{equation}\label{eq:model}
	f(x)=\mu_0+\sum_{i=1}^m \Delta_i \mathbbm{1}_{x\geq x_i^*}~,
\end{equation}
where $\mu_0\in\mathbb{R}$ is the baseline level and $\Delta_i=f(x_i^*)-f(x_{i-1}^*)$ for $i\in\{1,\dots,m\}$ is the jump at the $i$-th change point. For any compact interval $I=[l(I),r(I)]$, we introduce the notation $\Delta(I) = f(r(I)) - f(l(I))$ for the jump magnitude across the interval $I$. We assume $\Delta_i\neq 0$ for all $i$, so every change point is identifiable. Let $|\Delta_{(1)}|\geq\cdots\geq|\Delta_{(m)}|$ denote the decreasing ordering of jump magnitudes. These jumps govern asymptotic sample-complexity rates in~\citet{pmlr-v267-lazzaro25a}, through the complexity $H_{\mathrm{localize}}^{(N)}=\sum_{i=1}^N 1/\Delta_{(i)}^2$ (Equation~\eqref{eq:H_loc}).
Additionally, we assume that the jumps are bounded by $1$, i.e., $|\Delta_i|\leq 1$ for all $i\in\{1,\dots,m\}$. This assumption is without loss of generality, as one can always rescale the problem by a known bound on the jumps. It is made for clarity of presentation and to avoid unnecessary constants in the bounds, but it can easily be removed.

Spacing between change points also plays a central role.  For $i\in\{1,\dots,m-1\}$, let $\vartheta_i=x_{i+1}^*-x_{i}^*$ denote the spacing after the $i$-th change point. By convention, define $\vartheta_0=\vartheta_m=1$. We then define the local spacing by $s_i=\vartheta_{i-1}\wedge\vartheta_i$ for $i=1,\dots,m$,
and the associated energies by
\begin{equation}\label{eq:def_energy}
	\mathcal{E}_i^2=s_i\Delta_i^2\qquad\text{for } i=1,\dots,m.
\end{equation}
The quantity $s_i$ plays the role of sparsity level, and the quantity $\mathcal{E}_i$ captures the effect of local spacing on the difficulty of detecting change point $i$, in line with the batch literature~\citep{yu2020review}. The interplay between these quantities is captured in the complexity term $H_{\mathrm{detect}}^{(m)}=\max_{i=1,\dots,m} 1/\mathcal{E}_i^2$ (Equation~\eqref{def:H_detect}). Observe that for $m=N=1$, one has $H_{\mathrm{detect}}=H_{\mathrm{localize}}$.

We work in a bandit setting: at each round $t\geq 1$, the learner chooses $x_t\in[0,1]$ and observes
$y_t=f(x_t)+\varepsilon_t $,
where $(\varepsilon_t)_{t\geq 1}$ are i.i.d. 1-subGaussian random variables, i.e., $\mathbb{E}[\varepsilon_t]=0$ and $\mathbb{E}[\exp(\lambda\varepsilon_t)]\leq\exp(\lambda^2/2)$ for all $\lambda\in\mathbb{R}$. This includes, for instance, Gaussian and bounded noise.\footnote{The analysis extends to $\sigma$-subGaussian noise by a standard rescaling.} The pair $(f,\varepsilon)$ defines an environment, denoted by $\nu$.

Sampling is sequential and adaptive: each query $x_t$ may depend on past observations. An algorithm $\pi$ consists of three components: (i) a sampling rule for choosing next query, (ii) a stopping rule, and (iii) a decision rule that outputs estimated change points. These components may be randomized. The (random) sample complexity of $\pi$, denoted by $\mathcal{T}_{\pi}$, is the total number of evaluations of $f$. We write $\mathbb{P}_{\pi,\nu}$ and $\mathbb{E}_{\pi,\nu}$ for probability and expectation under environment $\nu$ and algorithm $\pi$.

The learner is given a precision $\eta\in(0,1)$, a confidence parameter $\delta\in(0,1)$, and a target number of change points $N\in\{1,\dots,m\}$. The objective is to identify $N$ distinct change points within distance $\eta$ of true change points with probability at least $1-\delta$, while minimizing the sample complexity.
Note that the learner can select any of the $m$ change points, as long as it localizes $N$ distinct change points.

\begin{definition}[$({\delta},{\eta},{N})$-correctness]\label{def:correctness}
An algorithm $\pi$ is $(\delta,\eta,N)$-correct if, for every environment $\nu$ with at least $N$ change points,
\begin{equation}\label{eq:correctness}
	\mathbb{P}_{\pi,\nu}\left(\exists 1\leq i_1<\dots<i_N\leq m:\ \forall l \in [N], |\hat{x}_l-x_{i_l}^*|\leq \eta\right)\geq 1-\delta\enspace ,
\end{equation}
where $\hat{x}_1,\dots,\hat{x}_N$ are the estimates returned by $\pi$.
\end{definition}

\section{Algorithmic method and guarantees}\label{sec:algorithm}

\subsection{Detection of change points}\label{sec:detection}

We describe the detection subroutine at the core of our method. Given a confidence level $\delta$ and a budget $T$, \texttt{DetectIntervals} (Algorithm~\ref{alg:di}) returns a set $\mathcal{I}$ of intervals with disjoint interiors. Using multiscale endpoint tests, the procedure guarantees that every returned interval contains at least one change point, with failure probability at most $\delta$ (Lemma~\ref{lem:upperbounddetection}).

\paragraph{Procedure.} For each depth $d\in\{1,\dots,d_{\max}\}$ (with $d_{\max}$ defined in Line~\ref{line:di:init}), the domain is partitioned into $n_d=2^d$ dyadic intervals of equal length. Each endpoint is sampled $T_d$
times (Line~\ref{line:di:budget-threshold}). Interval $[i/2^d,(i+1)/2^d]$ is flagged when the empirical jump between its endpoints exceeds the Hoeffding threshold $\beta_d$  (Line~\ref{line:di:test}), which accounts for the multiple testing across depths, intervals, and endpoints.\footnote{If $T_d=0$, $\beta_d$ is set to $+\infty$ by convention.}
Flagged intervals are aggregated across depths; if a newly flagged interval is strictly contained in an existing one, the larger one is removed (Line~\ref{line:di:new-interval}). This pruning keeps $\mathcal{I}$ disjoint while favoring the finest detected intervals.

\begin{algorithm}[ht]
		\caption{\texttt{DetectIntervals}}\label{alg:di}
		\begin{algorithmic}[1]
			\Require $\delta$ confidence parameter, $T$ budget
			\Ensure  $\mathcal{I}$ set of disjoint intervals
			\State $\mathcal{I}\gets\emptyset$, $d_{\max}\gets \lfloor\log_2(T/\log(1/\delta))\rfloor $
			\label{line:di:init}
			\For{$d=1,\dots,d_{\max}$} \Comment{Multiple-scale loop}
			\label{line:di:depth-loop}
            \State $n_d\gets 2^d$, $T_d\gets\left\lfloor\frac{ T}{d_{\max}(n_d+1)}\right\rfloor$, $\beta_d\gets\sqrt{\frac{8}{T_d}\log\left(\frac{2d_{\max}(n_d+1)}{\delta}\right)}$
			\label{line:di:budget-threshold}
			\For{$i=0,1,\dots,n_d$}
			\label{line:di:endpoint-loop}
			\State sample $T_d$ times from $x_i^{(d)}=i/n_d$, store means as $\hat y_i^{(d)}$\Comment{Sample uniformly at endpoints}\label{line:di:sample-endpoints}
			\EndFor
			\For{$i=1,\dots,n_d$}
			\label{line:di:adjacent-loop}
			\If{$|\hat y_i^{(d)}-\hat y_{i-1}^{(d)}|>\beta_d$} \Comment{Test for presence of jump between adjacent endpoints}
			\label{line:di:test}
            \State remove from $\mathcal{I}$ any set containing $[x_{i-1}^{(d)},x_i^{(d)}]$ \Comment{Pruning step}
			\State add $[x_{i-1}^{(d)},x_i^{(d)}]$ to $\mathcal{I}$, \Comment{Update $\mathcal{I}$}\label{line:di:new-interval}
			\EndIf
			\EndFor
			\EndFor
		\end{algorithmic}
	\end{algorithm}

\begin{lemma}\label{lem:upperbounddetection}
Let $0<\delta<1/4$. On an event $\xi$ with $\mathbb{P}(\xi)\geq 1-\delta$, Algorithm~\ref{alg:di} only outputs intervals that contain at least one change point.
In addition, there exists a sufficiently large numerical constant $c>0$ such that for any change point $x_i^*$ (with $i\in\{1,\dots,m\}$), if
	$T\geq c\cdot\omega^a_i(\delta)\cdot \frac{1}{\mathcal{E}_i^2}$,
	with
	$
	\omega_i^a(\delta)=(\log(\mathcal{E}_i^{-2})\vee 1)\cdot\log \left(\frac{\log(\mathcal{E}_i^{-2})\vee e}{s_i\cdot \delta}\right)~,
	$ then on the same event $\xi$, there exists an interval $I\in\mathcal{I}$ such that  $l(I)<x_i^*\leq r(I)$ (the change point may coincide with the right endpoint of $I$) and $|I|\leq \frac{s_i}{2}$. In particular, if $T\geq \max_{i=1,\dots,m} c\cdot\omega^a_i(\delta)\cdot \frac{1}{\mathcal{E}_i^2}$, then on the event $\xi$, for any change point $i\in[m]$, there exists a unique interval $I\in \mathcal{I}$ containing $x_i^*$, and this interval has a length smaller than $\frac{s_i}{2}$.
\end{lemma}

\paragraph{Guarantees and energy scaling.}
For intuition on Lemma~\ref{lem:upperbounddetection}, consider for change point $i$ its \emph{natural depth} $d_i^\star\asymp\log_2(1/s_i)$: at this depth, one dyadic interval isolates $x_i^*$ and has left-right jump $|\Delta_i|$. Detection requires $|\Delta_i|\gtrsim\beta_{d_i^\star}$, i.e., $T_{d_i^\star}\gtrsim\Delta_i^{-2}\log(1/\delta)$. Since $T_{d_i^\star}\asymp s_iT$ up to logarithmic factors, this yields $T\gtrsim\mathcal{E}_i^{-2}\log(1/\delta)$, which explains the $H_{\mathrm{detect}}^{(m)}$ term in the upper bound.

\subsection{Main algorithm and guarantees}

Algorithm~\ref{alg:lcp} takes as input the target number of change points $N$, the confidence level $\delta$, and the precision $\eta$. Importantly, \texttt{LCP} does not require knowledge of $m$; only the target $N \leq m$ is needed as input. In addition, it uses a tuning parameter $\delta_{\mathrm{explore}}$, which controls the probability of failure of the exploration phase and, therefore, the budget. No other problem-dependent parameter is required.

The algorithm follows a doubling schedule. At stage $k$ (initialized in Line~\ref{lin:initk}), it is given budget $T_k=2^{k+2}$ and runs four steps that mimic the oracle strategy described above. If, at the end of the stage, all $N$ candidates are certified with error probability at most $\delta^{(k)}$, the algorithm returns the estimates $\mathcal{C}=\{c_1^{(k)},\dots,c_N^{(k)}\}$ (Line~\ref{lin:returncp}). Otherwise, it moves to stage $k+1$, doubles the budget, and repeats (Line~\ref{lin:nextk}).
Before stating the main guarantees of Algorithm~\ref{alg:lcp}, we describe each step.

At a generic stage with budget $T_k$, the four steps are the following.

\paragraph{(i) Detection.}
In Line~\ref{lin:stepdetection}, \texttt{DetectIntervals} (Algorithm~\ref{alg:di}) constructs a set $\mathcal{I}$ of disjoint intervals, each certified to contain at least one change point. This step is run with confidence level $\delta_{\mathrm{explore}}/4$ and is described in Section~\ref{sec:detection}.

\paragraph{(ii) Jump estimation.}
In Line~\ref{lin:stepmultiplejumps}, \texttt{EstimateJumps} (Algorithm~\ref{alg:emj}) estimates the jump magnitudes on the candidate intervals and keeps the $N$ most informative ones. The procedure is designed so that, with high probability, these estimates are accurate up to a constant multiplicative factor. It returns $N$ intervals $\mathcal{J}^{(k)}=\{J_1^{(k)},\dots,J_N^{(k)}\}$ together with the corresponding jump estimates $\mathcal{G}^{(k)}=\{\hat \Delta_1^{(k)},\dots,\hat \Delta_N^{(k)}\}$, where $\hat\Delta_v^{(k)}$ estimates $|\Delta(J_v^{(k)})|$. The intervals are ordered by decreasing estimated jump, so that $\hat\Delta_1^{(k)}\geq \hat\Delta_2^{(k)}\geq \dots \geq \hat\Delta_N^{(k)}$.
\texttt{EstimateJumps} is a standard adaptive estimation routine. We defer its description to Appendix~\ref{app:estimate_jumps}; its guarantees are stated in Lemma~\ref{lem:multiplejumps}.

\paragraph{(iii) Refinement of localization.}
We use \texttt{SHB} (\emph{Sequential Halving with Backtracking}) from~\citet{lazzaro2025fixedbudget}, which localizes a single change point up to precision $\eta$ under a fixed budget constraint. The method consists of a binary-search procedure, with a backtracking mechanism as in~\citet{cheshire2020influence}. In particular, when applied to an interval containing a single change point of jump magnitude $\Delta$, a budget of order $\Delta^{-2}\log(1/(\delta\eta))$ suffices to return an estimate within distance $\eta$ of the true change point with probability at least $1-\delta$~\citep{lazzaro2025fixedbudget}. For completeness, we restate these guarantees in Lemma~\ref{lem:shb}, and the routine itself in Algorithm~\ref{alg:shb}.

In Line~\ref{lin:stepbinsearch}, we apply \texttt{SHB} independently to each selected interval. The budget is allocated proportionally to the inverse squared estimated jump, so interval $J_v^{(k)}$ receives budget $T_v^{(k)}$ with
    \begin{equation}\label{def:alpha_v}
    	T_v^{(k)}=\left\lfloor\alpha_{v}^{(k)}\cdot 2^k\right\rfloor\vee 1\ , \text{with} \qquad \alpha_{v}^{(k)}=\frac{\big(\hat \Delta_v^{(k)}\big)^{-2}}{\sum_{v'=1}^N \big(\hat \Delta_{v'}^{(k)}\big)^{-2}} \enspace .
    \end{equation}
\paragraph{(iv) Verification.}
In Line~\ref{lin:stepverify}, \texttt{VerifyCP} (Algorithm~\ref{alg:vcp}) tests whether the candidates $(c_v^{(k)})_{v\in[N]}$ lie within distance $\eta$ of $N$ distinct change points. For each $v$, it performs a two-sample test at the points
\begin{equation}\label{def:l_v_r_v}
l_v^{(k)}:=\max\big(\min J_v^{(k)},c_v^{(k)}-\eta\big),
\qquad
r_v^{(k)}:=\min\big(\max J_v^{(k)},c_v^{(k)}+\eta\big),
\end{equation}
using confidence level $\delta^{(k)}= \frac{3\delta}{2\pi^2Nk^2}$. It draws $T_v^{(k)}/2$ samples at each endpoint and compares the empirical jump to the threshold $\sqrt{\frac{{16}}{T_v^{(k)}}\log\left(\frac{2}{\delta^{(k)}}\right)}$. If the test is positive, then $c_v^{(k)}$ is certified to be within distance $\eta$ of a change point with probability at least $1-\delta^{(k)}$. The routine is given in Algorithm~\ref{alg:vcp}, and its guarantees are stated in Lemma~\ref{lem:vcp}.

\begin{remark}
The verification step (iv) cannot be performed with the jump estimation from Step (ii). Indeed, an interval returned in Step (i) may contain many change points with small jumps, whose cumulative effect creates a large jump between the left and right endpoints. In such cases, the refinement step is not guaranteed to converge to a change point, and verification is necessary.
\end{remark}

\begin{theorem}\label{thm:cpl}

Let $N\leq m$, $0<\delta<1/4$, $0<\eta <1/4$, and $\delta_{\mathrm{explore}}\leq 1/4$. Consider any environment $\nu$ with $m\geq N$ change points, and run Algorithm~\ref{alg:lcp} with input parameters $(N,\delta,\eta,\delta_{\mathrm{explore}})$.

\begin{enumerate}
        \item (Correctness) Algorithm~\ref{alg:lcp} is $(\delta,\eta,N)$-correct (see~\eqref{eq:correctness}). With probability at least $1-\delta$, Algorithm~\ref{alg:lcp} returns $N$ points $0\leq c_1<c_2<\dots<c_N\leq 1$ such that $|c_i-x_{l_i}^*|\leq \eta$ for $i=1,\dots,N$, and $1\leq l_1<l_2<\dots<l_N\leq m$.
        \item (High-probability bound) If we set $\delta_{\mathrm{explore}}=\delta$, then there exists an event $\xi$ of probability at least $1-\delta$ on which the output is correct (Definition~\ref{def:correctness}), and the budget $\mathcal T$ satisfies
		\begin{align*}
			\mathcal{T}&\leq c \cdot\left( \omega_1 H_{\mathrm{detect}}^{(m)}\left(\log\left(\frac{1}{\delta}\right)+\omega_2\right)+H_{\mathrm{localize}}^{(N)}\left(\log\left(\frac{1}{\delta\eta}\right)+\omega_3\right)\right).
		\end{align*}
		\item (Expectation bound) If we set $\delta_{\mathrm{explore}}=1/4$, then the expected budget $\mathbb{E}[\mathcal T]$ is bounded by
		\begin{align*}
			\mathbb{E}[\mathcal{T}]&\leq c \cdot\left( \omega_1 H_{\mathrm{detect}}^{(m)}\omega_2+H_{\mathrm{localize}}^{(N)}\left(\log\left(\frac{1}{\delta\eta}\right)+\omega_3\right)\right).
		\end{align*}
	\end{enumerate}
In these bounds, $c>0$ is a numerical constant, and $\omega_1,\omega_2,\omega_3$ are logarithmic terms defined as $\omega_1=\log(H_{\mathrm{detect}}^{(m)})$, $\omega_2=\max_{i=1,\dots,m} \left\{\log\left(\log\left(\mathcal{E}_i^{-2}\right)\vee e\right)+\log\left(1/s_i\right)\right\}$, and $\omega_3=\log\left(N\cdot\left(\log\left(H_{\mathrm{localize}}^{(N)}\right)\vee 1\right)\right)$. See~\eqref{def:H_detect},~\eqref{eq:H_loc} for definitions of $H_{\mathrm{detect}}^{(m)}$ and $H_{\mathrm{localize}}^{(N)}$.
\end{theorem}

\begin{algorithm}[ht]
	\caption{\texttt{LocalizeChangePoints}} \label{alg:lcp}
	\begin{algorithmic}[1]
		\Require $N$ number of change points, $\delta$ confidence parameter for correctness, $\eta$ precision parameter, $\delta_{\mathrm{explore}}$ confidence parameter for budget control
		\Ensure  $\mathcal{C}$ estimated change points
		\State $k\gets \lceil \log_2(2N)\rceil$ \Comment{Initialize doubling schedule}
		\label{lin:initk}
		\While{\texttt{True}}
		\State $\mathcal{I}^{(k)}\gets\texttt{DetectIntervals}\left(\frac{\delta_{\mathrm{explore}}}{4},2^k\right)$ \Comment{\emph{(i) Detect candidate intervals}}
		\label{lin:stepdetection}
		\If{$|\mathcal{I}^{(k)}|\geq N$}
		\State $\left(\mathcal{J}^{(k)},\mathcal{G}^{(k)}\right)\gets \texttt{EstimateJumps}\left(\mathcal{I}^{(k)},\frac{\delta_{\mathrm{explore}}}{4},2^k,N\right)$ \Comment{\emph{(ii) Estimate jump magnitudes}}
		\label{lin:stepmultiplejumps}
		\If{$|\mathcal{J}^{(k)}|\geq N$} \Comment{If $N$ intervals detected and estimated}
        \State Order $\mathcal{J}^{(k)}=(J_1^{(k)},\dots,J_N^{(k)} )$ according to estimate jump from $\mathcal{G}^{(k)}$
		\For{$v=1,\dots,N$}
		\State $c_{v}^{(k)}\gets \texttt{SHB}\left(J^{(k)}_{v}, T_v^{(k)},\eta\right)$ \Comment{\emph{(iii) Refine location to precision $\eta$; for $T_v^{(k)}$ see~\eqref{def:alpha_v}}}
		\label{lin:stepbinsearch}
        \State  $\mathrm{ok}_v^{(k)}\gets \texttt{VerifyCP}\left(l_{v}^{(k)},r_v^{(k)},\delta^{(k)},T_v^{(k)}\right)$ \Comment{\emph{(iv) Verify estimate; for $l_v,r_v$ see~\eqref{def:l_v_r_v}}}\label{lin:stepverify}
		\EndFor
		\If{$\forall v\in[N]$, $\mathrm{ok}_v^{(k)}$ is \texttt{True}}  \Comment{If all change points verified}
		\State \Return $\mathcal{C}\gets\left\{c_{1}^{(k)},\dots, c_{N}^{(k)}\right\}$\Comment{Stop and output estimates}\label{lin:returncp}
		\EndIf
		\EndIf
		\EndIf
		\State $k\gets k+1$, and $\delta^{(k)}\gets \frac{3\delta}{2\pi^2Nk^2}$\Comment{Double budget and retry}
		\label{lin:nextk}
		\EndWhile
	\end{algorithmic}
\end{algorithm}

\paragraph{Intuition.}
To interpret $H_{\mathrm{localize}}^{(N)} = \sum_{i=1}^N \Delta_{(i)}^{-2}$,
consider an oracle that knows the jumps $(\Delta_i)_{i=1}^m$ and provides $N$
disjoint intervals each containing one target change point. Running $N$ parallel
binary searches then costs $H_{\mathrm{localize}}^{(N)}\log(1/(\delta\eta))$ --
the pure localization cost. The main challenge is that these intervals are unknown;
acquiring them adaptively requires detecting that a region contains a change point. Our main message is that it comes at an additional cost of $\mathcal{E}_i^{-2}\log(1/\delta)$ per change point.
The total complexity thus splits into a detection cost $H_{\mathrm{detect}}^{(m)}$
and a localization cost $H_{\mathrm{localize}}^{(N)}\log(1/(\delta\eta))$.
Complete proofs are given in Section~\ref{appendix:UB}.

\paragraph{Sketch of proof.}
    By step (iv), each stage-$k$ acceptance via \texttt{VerifyCP} fails with probability at most $\delta^{(k)}$ (Lemma~\ref{lem:vcp}), so the union bound $\sum_k N\delta^{(k)}\le\delta$ implies $(\delta,\eta,N)$-correctness. \\
We now sketch the budget bound by following the four-step procedure at stage $k$. \\
\emph{(i) Detection.} By Lemma~\ref{lem:upperbounddetection}, if
$
T_k\gtrsim \omega_1 H_{\mathrm{detect}}^{(m)}\big(\log(1/\delta_{\mathrm{explore}})+\omega_2\big),
$
then with probability at least $1-\delta_{\mathrm{explore}}/4$, \texttt{DetectIntervals} returns $m$ disjoint intervals that isolate each of the $m$ change points. \\
\emph{(ii) Jump estimation.} On the same event, and with the same budget, Lemma~\ref{lem:multiplejumps} implies that \texttt{EstimateJumps} returns $N$ intervals and constant-factor accurate jump estimates. Therefore the allocation weights $\alpha_v^{(k)}$ are within constants of the oracle proportions $\Delta_{(v)}^{-2}/H_{\mathrm{localize}}^{(N)}$. \\
\emph{(iii) Refinement.} Given these allocations, Lemma~\ref{lem:shb} (from \citealp{lazzaro2025fixedbudget}) yields that the $v$-th change point is localized to accuracy $\eta$ once it receives budget of order $\Delta_{(v)}^{-2}\log(1/(\delta^{(k)}\eta))$. Summing over $v\in[N]$ gives a sufficient condition for the refinement step to succeed with probability at least $1-\delta^{(k)}$:
$T_k\gtrsim H_{\mathrm{localize}}^{(N)}\big(\log(1/(\delta^{(k)}\eta))+\omega_3\big)$. \\
\emph{(iv) Verification.} By Lemma~\ref{lem:vcp}, when $[l_v^{(k)}, r_v^{(k)}]$ contains the $v$-th change point, with jump magnitude $\Delta_{(v)}$, as long as $T_v^{(k)}\gtrsim \Delta_{(v)}^{-2}\log(1/(\delta^{(k)}\eta))$, \texttt{VerifyCP} returns \texttt{True} with high probability. Thus, once (i)--(iii) hold, the stage is successful with high probability, and the algorithm stops. \\
Combining (i)--(iv), a sufficient stage condition, implying both bounds is
 \begin{equation*}
T_k\gtrsim \omega_1 H_{\mathrm{detect}}^{(m)}\big(\log(1/\delta_{\mathrm{explore}})+\omega_2\big)
+H_{\mathrm{localize}}^{(N)}\big(\log(1/(\delta^{(k)}\eta))+\omega_3\big)\enspace. \end{equation*}
\section{Lower bounds}\label{sec:lower_bounds}
In this section, we derive information-theoretic lower bounds for the budget $\mathcal{T}$ of any $(\delta,\eta,m)$-correct algorithm  for the active change point detection problem in the case where $N=m$.

Consider an environment $\nu$ with $m\geq 1$ change points, for which we have to estimate every change point (i.e., $N=m$). As we consider different environments obtained by modification of $\nu$, we use $H_{\mathrm{detect}}^{(m,\nu)}$ (see~\eqref{def:H_detect}) and $H_{\mathrm{localize}}^{(m,\nu)}$ (see~\eqref{eq:H_loc}) to indicate their dependence on the environment $\nu$.

\begin{theorem}\label{thm:main_LB}
Let $\delta\in(0,1/4)$, $\eta\in(0,1/8)$, and consider $m=N$. Let $\pi$ be a $(\delta,\eta,m)$-correct algorithm. For any environment $\nu$ with $m=N$ change points, there exists an environment $\nu'$ such that $H_{\mathrm{detect}}^{(m,\nu)}\leq H_{\mathrm{detect}}^{(m,\nu')}\leq 4H_{\mathrm{detect}}^{(m,\nu)}$, $\frac{1}{2}H_{\mathrm{localize}}^{(m,\nu)} \leq H_{\mathrm{localize}}^{(m,\nu')}\leq H_{\mathrm{localize}}^{(m,\nu)}$, and such that
\begin{align*}
\mathbb{P}_{\pi,\nu'}\left(\mathcal{T}_{\pi}\geq\tfrac{1}{4}H_{\mathrm{detect}}^{(m,\nu)}\log\left(\frac{1}{8\delta}\right)+\tfrac{1}{2}H_{\mathrm{localize}}^{(m,\nu)}\log\left(\frac{1}{8\delta}\right)+\tfrac{1}{2}\sum_{i=1}^m \frac{1}{\Delta_{i}^{2}}\log_+\!\left(\frac{s_i}{16\eta}\right)\right)\geq \delta  \enspace.
\end{align*}
\end{theorem}

\begin{proof}[Sketch of proof]
The proof is information-theoretic. The high-probability statement combines two ingredients, proved separately in Lemmas~\ref{lem:LB_H_localize} and~\ref{lem:LB_H_detect}.

\emph{(i) Detection term.}
To obtain a $(1-\delta)$-quantile lower bound of order $H_{\mathrm{detect}}^{(m,\nu)}\log(1/\delta)$, we reduce the problem to signal detection.   This technique is, from our perspective, original and is, in particular, of interest for other active statistical problems. For intuition, consider an instance with two opposite jumps $\pm\Delta$ separated by spacing $s$, and compare it with a null environment with no change point. If an algorithm detects a change with probability at least $1-\delta$ using budget quantile $T$, it must sample sufficiently often inside the true plateau, whose location is unknown among about $1/s$ candidates shifts. This yields
$T\gtrsim \frac{1}{s\Delta^2}\log\!\left(\frac{1}{\delta}\right)$, which is exactly the energy-driven detection scaling. Our reduction extends this intuition to the general case. Our reduction requires $N=m$, so that removing one change point from the environment makes any $(\delta,\eta,N)$-correct algorithm fail to stop. \\
\emph{(ii) Localization at precision $\eta$.}
To obtain the term $\sum \Delta_i^{-2}\log(s_i/\eta)$, we construct a family of alternatives by shifting each change point $x_i^*$ over $\alpha_i\asymp s_i/8\eta$ candidate positions. Distinguishing these alternatives induces the factor $\log_+(s_i/\eta)$ in the quantile bound. We exploit symmetry techniques to reveal this multiple-testing cost. The term $H_{\mathrm{localize}}^{(m,\nu)}\log(1/\delta)$ follows from standard change-of-measure arguments~\citep{lattimore2019bandit,tsybakov2009introduction}, as in Theorem~5.2 in~\citet{pmlr-v267-lazzaro25a}. \\
By standard Markov-type arguments, the high-probability lower bound of Theorem~\ref{thm:main_LB} implies the expectation lower bound stated in the following corollary. See Appendix~\ref{appendix:LB} for the full proof.
\end{proof}

\begin{corollary}\label{cor:main_LB}
Let $\delta\in(0,1/16)$ and $\eta\in(0,1/8)$. Assume that all change points of $\nu$ are spaced by at least $2\eta$ ($\forall i=2,\dots,m-1,\ \vartheta_i >2\eta$). Then, there exists an environment $\nu'$ such that $H_{\mathrm{detect}}^{(m,\nu)}\leq H_{\mathrm{detect}}^{(m,\nu')}\leq 4H_{\mathrm{detect}}^{(m,\nu)}$, $\frac{1}{2}H_{\mathrm{localize}}^{(m,\nu)} \leq H_{\mathrm{localize}}^{(m,\nu')}\leq H_{\mathrm{localize}}^{(m,\nu)}$, and such that
\begin{align*}
\mathbb{E}_{\pi,\nu'}[\mathcal{T}_{\pi}]\geq c\left(H_{\mathrm{detect}}^{(m,\nu)}+H_{\mathrm{localize}}^{(m,\nu)}\log\left(\frac{1}{4\delta}\right)+ \sum_{i=1}^m \frac{1}{\Delta_{i}^{2}}\log_+\!\left(\frac{s_i}{16\eta}\right) \right) \enspace,
\end{align*}
where $c>0$ is a numerical constant.
\end{corollary}
\section{Numerical experiments}\label{sec:experiments}
\begin{figure}[ht]
    \centering
    \begin{subfigure}{0.49\textwidth}
        \centering
        \includegraphics[width=0.9\textwidth]{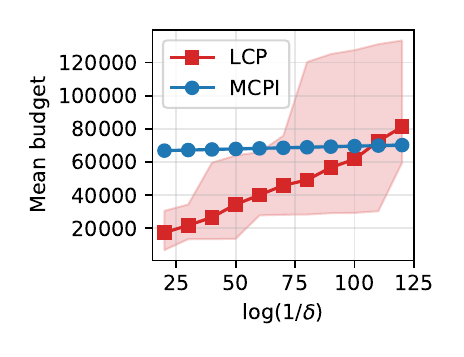}
        \caption{\textbf{Experiment 2}. Average budget and $(0.05,0.95)$ quantiles for localizing two CPs with varying $\delta$.}\label{fig:varying_delta}
    \end{subfigure}
    \hfill
    \begin{subfigure}{0.49\textwidth}
        \centering
        \includegraphics[width=0.9\textwidth]{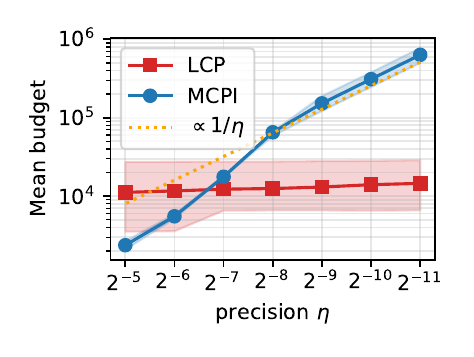}
        \caption{\textbf{Experiment 3}. Average budget and $(0.05,0.95)$ quantiles for localizing two CPs with varying $\eta$.}\label{fig:varying_eta}
    \end{subfigure}
    \caption{Numerical experiments}
\end{figure}
Our theoretical findings are supported by simple numerical experiments on simulated data, where we model the underlying noise to be standard normal. In the following two experiments,
we consider an environment with two change points,
$x_1^* \sim \mathcal{U}(0, 1/2)$ and $x_2^* = x_1^* + s$ with $s = 1/4$,
and jump sizes $\Delta_1 = -\Delta_2 = 1$. Each experiment is averaged
over $1000$ Monte Carlo runs.
We compare our method, \texttt{LCP}, to \texttt{MCPI}
from~\citet{pmlr-v267-lazzaro25a}. Recall that \texttt{MCPI} is designed
for discrete bandit change point problems, so we apply it to a
discretization of the continuous problem (see
Appendix~\ref{sec:continuousdiscrete}). \texttt{MCPI} follows a
Track-and-Stop approach with a forced exploration phase. We will see
that while the number of arms, here of order $1/\eta$, does not affect
the asymptotic sample complexity as $\delta \to 0$, it strongly
influences the algorithm's performance in practice.

In \textbf{Experiment 2}, we set $\eta=2^{-8}$. With these configurations, we run \texttt{LCP} (again we treat $\delta_{\mathrm{explore}}$ as tuning parameter and set $\delta_{\mathrm{explore}}=1$) and \texttt{MCPI} on a discretization for varying values of $\delta$, that is, $\log(1/\delta)\in\{20,30,\dots,120\}$. We plot the averaged budget in Figure~\ref{fig:varying_delta}. We can see that in this setting, the slope of the budget for \texttt{LCP} with respect to $\log(1/\delta)$ is larger than the very small slope of the budget required for \texttt{MCPI}. Still, it takes $\delta \approx e^{-110}$ until the average budget of our method is larger.

In \textbf{Experiment 3}, we choose $\delta =0.05$ and run \texttt{LCP} (again, with $\delta_{\mathrm{explore}}=1$) and \texttt{MCPI} for $\eta=2^{-i}$, $i=5,6,\dots,11$. We plot the results of our experiment in Figure~\ref{fig:varying_eta}. For \texttt{LCP}, there is only a small increase of the budget when $\eta$ decreases, compared to \texttt{MCPI}. For the latter, we can see that for small values of $\eta$, the budget increases nearly proportionally to $1/\eta$.

Both experiments illustrate that for moderate choices of $\delta$, asking for a small precision (or considering many arms in the discrete case) does not limit the performance of \texttt{LCP}, in contrast to \texttt{MCPI}.
The code to reproduce all experiments is available at \repo{}. We refer to Appendix~\ref{sec:further:experiments} for an evaluation of the runtime and further experiments. Our experiments also suggest that the large constant $c$ in Theorem~\ref{thm:cpl} is an artifact of the proof, as it is not reflected in practice.

\section{Discussion}\label{sec:discussion}
\paragraph{Comparison with \citet{pmlr-v267-lazzaro25a}.}
\citet{pmlr-v267-lazzaro25a} study the same problem, but their analysis targets the asymptotic regime
$\delta\to 0$ with all other parameters fixed.
In this regime the detection stage vanishes from the complexity, and the
leading term reduces to $H_{\mathrm{localize}}^{(N)}\log(1/\delta)$,
corresponding to $N$ independent localizations of a single change point.
This reduction misses a key phenomenon when $m\geq 2$: estimating several closely located jumps induces an additional detection cost, which
dominates when $\delta$ is bounded away from zero.
Our analysis isolates this term as $H_{\mathrm{detect}}^{(m)}$, shows that
it enters additively in expectation and does not depend on $\delta$, and
provides a matching lower bound (Theorem~\ref{thm:main_LB}).
We also sharpen the dependence on precision: replacing the uniform
$\eta$-grid initialization of \citet{pmlr-v267-lazzaro25a} -- which incurs
a detection cost of order $\eta^{-1}$ -- with multiscale detection followed
by binary-search refinement reduces the $\eta$ dependence to logarithmic,
so the localization component scales as $H_{\mathrm{localize}}^{(N)}\log(1/\eta)$.
More generally, the guarantees are non-asymptotic in both $(\eta,\delta)$ and adaptive to local spacing $s_i$, in contrast to approaches that require a global separation condition $\min_i\vartheta_i>\eta$.
\paragraph{Optimality.}
 When $N=m$, the upper and lower bounds match up to polylogarithmic
factors which do not depend on $\delta$ and $\eta$. Hence \texttt{LCP} is
rate-optimal across the full range of $(\delta,\eta)$, which characterizes the
sample complexity of active multiple change point identification up to these factors.
In the limit $\delta\to 0$, our expectation bound scales as
$H_{\mathrm{localize}}^{(N)}\log(1/\delta)$, confirming rate-optimality.
\texttt{MCPI}~\citep{pmlr-v267-lazzaro25a} attains the exact optimal constant
$\lim_{\delta\to 0}\mathbb{E}[\mathcal{T}]/\log(1/\delta) = 8H_{\mathrm{localize}}^{(N)}$, while our bound is order-optimal with a looser constant; closing this gap may
require a generalized likelihood ratio stopping rule, as in Track-and-Stop
methods~\citep{garivier2016optimal}.
\paragraph{Fixed budget setting.}
Our procedure targets the fixed-confidence setting, but one stage of the
doubling schedule in Algorithm~\ref{alg:lcp} could potentially be adapted
to fixed-budget. Doing so would require a parameter-free exploration phase,
whereas the current method relies on a tuning parameter $\delta_{\mathrm{explore}}$
to control failure probability.
Such a fixed-budget variant would also address the case $N>m$, where the
fixed-confidence formulation is ill-posed: any algorithm targeting $N$
distinct change points would loop indefinitely with probability at least
$1-\delta$.  A related alternative is to replace the input $N$ by a minimal energy
threshold $E_{\min}$, targeting all change points with energy at least
$E_{\min}$; we discuss this extension in Appendix~\ref{appendix:extensions}.
We leave both directions for future work.
\paragraph{Extensions.} The detection--localization decomposition at the core of our work is not specific to the one-dimensional piecewise-constant model: the phenomena it reveals -- the detection--localization trade-off and the role of energy -- motivate design principles that transfer to richer models. The detection step (two-sample tests on dyadic grids) and the localization step (binary search) are modular and can be replaced by richer procedures without altering the overall algorithmic structure, suggesting natural extensions to smoother signal classes, vector-valued outputs, and multi-dimensional input spaces. We discuss these directions in Appendix~\ref{appendix:extensions}.
\paragraph{Limitations.}
Our lower bounds are limited to the case $N=m$. When $N<m$, the upper bound is likely loose, and the optimal trade-off remains a highly non-trivial open question; we discuss this regime and a conjecture in Appendix~\ref{appendix:extensions}. Algorithmically, our procedure relies on a doubling schedule, which is not always desirable in practice. Adaptations are required to derive a fully-adaptive counterpart of the method.

\begin{ack}
The work of M. Graf is partially supported by the Deutsche Forschungsgemeinschaft (DFG) -- Project-ID 318763901 -- SFB1294 ``Data Assimilation'', Project A03, by the DFG within the Forschungsgruppe FOR5381 ``Mathematical Statistics in the Information Age -- Statistical Efficiency and Computational Tractability'', Project TP 02 (Project-ID 460867398), and by the Universit\'e franco-allemande (UFA) through the Coll\`ege doctoral franco-allemand CDFA-02-25 ``Statistisches Lernen f\"ur komplexe stochastische Prozesse''. The work of V. Thuot has partially been supported by ANR-21-CE23-0035 (ASCAI, ANR).
\end{ack}

\bibliography{biblio_change_point}

\newpage
\appendix

\section{List of notation}\label{sec:notation}
We summarize the notation used throughout the paper in Table~\ref{tab:notation}.

\begin{table}[ht]
  \centering
  \caption{Summary of notation.}\label{tab:notation}
  \renewcommand{\arraystretch}{1.3}
  \begin{tabular}{@{}lp{10cm}@{}}
    \toprule
    \textbf{Symbol} & \textbf{Description} \\
    \midrule
    $f$           & Unknown step function \\
    $m$                  & Total number of change points \\
    $x_i^*$  & Location of the $i$-th change point \\
    $x_t$                 & Location, chosen by a learner at time $t$ \\
    $\varepsilon_t$ & SubGaussian noise at time $t$\\
    $y_t$ & Feedback $y_t=f(x_t)+\varepsilon_t$ that the learner observes at time $t$\\
    $N$ & Number of change points the learner wants to localize\\
    $\eta$ & Precision parameter\\
    $\delta$  & Confidence parameter for correctness \\
    $\delta_{\mathrm{explore}}$ & Confidence parameter for exploration\\
    $\mathcal{T}$ & Sample complexity / budget\\
    $\Delta_i$ & Jump at the $i$-th change point, $\Delta_i=f(x_i^*)-f(x_{i-1}^*)$\\
    $\Delta_{(i)}$ & Jump ordered by its magnitude such that $|\Delta_{(1)}|\geq\dots\geq|\Delta_{(m)}|$\\
	$\vartheta_i$ & Spacing after the $i$-th change point, $\vartheta_i=x_{i+1}^*-x_{i}^*$ for $i=1,\dots,m-1$, with convention $\vartheta_0=\vartheta_m=1$ \\
    $s_i$ & Local spacing between the $i$-th change point and its neighboring change points, $s_i=\vartheta_{i-1}\wedge\vartheta_i$ for $i=1,\dots,m$,\\
    $\mathcal{E}_i^2$ & Energy of the $i$-th change point, $\mathcal{E}_i^2=s_i\Delta_i^2$\\
    $H_{\mathrm{detect}}^{(m)}$ & Complexity term for change point detection, $H_{\mathrm{detect}}^{(m)}=\displaystyle\max_{i=1,\dots,m}\mathcal{E}_i^{-2}$\\
	$H_{\mathrm{localize}}^{(N)}$& Complexity term for change point localization, $H_{\mathrm{localize}}^{(N)}=\sum_{i=1}^N1/\Delta_{(i)}^2$ \\
    $l(I)$, $r(I)$ & Boundaries of a compact interval $I=[l(I),r(I)]$ \\
    $\Delta(I)$ & Jump between interval boundaries, $\Delta(I)=f(r(I))-f(l(I))$
    \\
    \midrule

    \bottomrule
  \end{tabular}
\end{table}

\section{Extensions and perspectives}\label{appendix:extensions}

We collect here open problems within the current framework and directions for
extending it to richer settings.

\paragraph{Optimal complexity when $N<m$.}
When $N<m$, the upper bound is likely loose: $H_{\mathrm{detect}}^{(m)}$
accounts for all $m$ change points, so low-energy untargeted change points can
inflate the detection cost without affecting localization.
Consider $N=2$, $m=4$, with untargeted change points at
$x_1^*=\tfrac{1}{4}-\varepsilon$, $x_4^*=\tfrac{3}{4}+\varepsilon$
and targeted ones at $x_2^*=\tfrac{1}{4}$, $x_3^*=\tfrac{3}{4}$.
Setting $\Delta_1=\varepsilon=-\Delta_4$ and $\Delta_2=1=-\Delta_3$ gives
$\mathcal{E}_1^2=\mathcal{E}_4^2=\varepsilon^3$ and $\mathcal{E}_2^2=\mathcal{E}_3^2=\varepsilon$,
so $H_{\mathrm{detect}}^{(4)}\asymp\varepsilon^{-3}\to\infty$ while
$H_{\mathrm{localize}}^{(2)}=2$.
Yet \texttt{LCP} with $N=2$ should not pay this price: at coarse depths the jump
across $[0,\tfrac{1}{4}]$ is $\Delta_1+\Delta_2\approx 1$, so this interval is
detected at $O(1)$ cost and the perturbation $\varepsilon$ is never resolved.
The relevant detection cost should therefore involve the energy of
\emph{detectable} change points -- a notion that differs from $s_i$ when $N<m$,
and explains why $H_{\mathrm{detect}}^{(m)}$ overestimates the true cost.
Formalizing this and deriving matching lower bounds is left for future work.

\paragraph{Energy-based identification.}
The $(\delta,\eta,N)$-correctness framework requires $N\leq m$; if $N>m$,
no algorithm terminates with high probability under fixed confidence, since
the smallest energy $\mathcal{E}_i$ is unknown and the algorithm cannot
rule out the existence of a change point with arbitrarily low energy.
A natural alternative is to fix a minimal energy threshold $E_{\min}$
and ask the algorithm to identify \emph{all} change points with energy
$\mathcal{E}_i\geq E_{\min}$, within precision $\eta$ and with
confidence $1-\delta$; the total number of such change points is then
recovered as a by-product.
We expect that our methodology can be adapted to this setting.
From Lemma~\ref{lem:upperbounddetection}, a budget
$T\gtrsim E_{\min}^{-2}\log(1/\delta)$
suffices to detect all change points with energy at least $E_{\min}$
with probability at least $1-\delta$.
Combining this with the refinement and verification steps yields a
sufficient budget of
\begin{equation*}
T \;\gtrsim\;
\frac{1}{E_{\min}^2}\log\!\left(\frac{1}{\delta}\right)
+\sum_{i:\,\mathcal{E}_i\geq E_{\min}}
\frac{1}{\Delta_i^2}\log\!\left(\frac{s_i}{\eta}\right),
\end{equation*}
up to logarithmic factors.
This framework still requires knowledge of a lower bound on the energy of
the change points to be detected, mirroring the role of $N\leq m$ in the
original setting. We leave a full treatment for future work.

The algorithmic structure of \texttt{LCP} -- detecting rough regions that contain a change, then refining their location -- is not specific to the one-dimensional piecewise-constant model. We sketch here how it can be adapted to several richer settings.

\paragraph{Beyond piecewise-constant: significant changes.}
\citet{suk2022tracking} introduce the notion of \emph{significant arm
switches} to avoid the piecewise-constant assumption in non-stationary
bandits. Transposed to our spatial setting, this amounts to replacing the
definition of change points by that of \emph{significant changes}: points
$\tau\in[0,1]$ at which the local energy
\begin{equation*}
    \mathcal{E}_{\tau,s}
    = \left|\frac{\int_{\tau}^{\tau+s}f(x)\,\mathrm{d}x
               - \int_{\tau-s}^{\tau}f(x)\,\mathrm{d}x}{\sqrt{s}}\right|
\end{equation*}
is large for some scale $s>0$, together with an appropriate hierarchical
pruning to remove redundant detections~\citep{suk2022tracking}. This
formulation is general: it covers in particular piecewise-linear and
piecewise-polynomial functions~\citep{manegueu2021generalized}. We expect
that \texttt{LCP} can be adapted to this setting, using the
definitions and reduction tools of \citet{suk2022tracking} to relate it back
to the piecewise-constant change point scenario, possibly without requiring new theoretical arguments
beyond those developed here.

\paragraph{Smooth transitions.}
The piecewise-constant model assumes instantaneous jumps, which may be
idealized in practice. Our method can be adapted to settings where the
target locations are smooth transitions between levels, by replacing the
two-sample endpoint test with a threshold-crossing criterion, as in
\citet{cheshire2021problem}. Binary search remains applicable in this
context: it achieves precision $\eta$ in $O(\log(1/\eta))$ queries,
compared to $O(1/\eta)$ for a uniform sweep -- an exponential gain that
our algorithm extends to multiple transitions through the same
detection-then-localization decomposition.

\paragraph{Piecewise-linear and piecewise-polynomial functions.}
In a piecewise-linear model, the change points are discontinuities of $f'$
rather than of $f$. Our method can be adapted by estimating $f'$ from finite
differences: querying $f(x)$ and $f(x+\varepsilon)$ for small $\varepsilon$
yields derivative estimates that replace raw values in the two-sample tests
used by \texttt{DetectIntervals} and \texttt{VerifyCP}. The overall
algorithmic structure is unchanged, and we expect the main guarantees to
extend, with the jump magnitudes $\Delta_i$ replaced by the corresponding
derivative jumps. The same principle applies to piecewise-polynomial models
of any fixed degree, by targeting discontinuities of higher-order finite
differences.

\paragraph{Vector-valued outputs.}
Extending to vector-valued functions $f:[0,1]\to\mathbb{R}^d$ mainly
requires replacing the scalar two-sample test with an appropriate
multivariate procedure~\citep{wang2018high}, possibly sparse-adaptive when
$d$ is large~\citep{pilliat2023optimal}. The detection-localization
decomposition and the doubling schedule carry over without modification. In this case, remark that the computational cost in higher dimension is linear in that of a single high-dimensional two-sample test.
\paragraph{Multi-dimensional input space.}
Moving to $f:[0,1]^d\to\mathbb{R}$ with $d\geq 2$ changes the nature of the problem substantially: one now seeks a \emph{boundary} rather than isolated points, and the learning objective admits several non-equivalent formulations. This setting is still far from being fully understood even in the passive literature. Nevertheless, the core conceptual contribution of our work -- a multistage procedure that first identifies coarse regions containing a change and then refines their location -- may remain applicable. For instance, it can potentially be combined with the dyadic CART framework of \citet{madrid2021lattice}, where the change boundary is
a rectangular partition of $[0,1]^d$. We leave a formal treatment of this
direction for future work.

\section{Further numerical experiments and runtime analysis}\label{sec:further:experiments}
All experiments were run on a MacBook Air (Apple M1, 8 GB RAM),
using 4 parallel worker processes. The runtime is documented in table~\ref{tab:compute_resources}.
\begin{table}[ht]
    \centering
    \caption{Computational resources used for each experiment (runtimes in sec.)}
    \label{tab:compute_resources}
    \begin{tabular}{lcccc}
        \toprule
        & \textbf{\# Monte Carlo runs} & \textbf{Total runtime} & \textbf{Runtime \texttt{LCP}} & \textbf{Runtime \texttt{MCPI}} \\
        \midrule
        Experiment 1 & 1000  & 6667.64 & 433.76 & 26277.82 \\
        Experiment 2 & 1000& 1035.69 & 145.15 & 3987.08 \\
        Experiment 3 & 1000 & 1634.18 & 37.19 & 6487.09 \\
        Experiment 4 & 1000& 599.06 & 37.31 & 2354.73 \\
        Experiment 5 & 100 &  2619.28& 163.92 & 10281.14 \\
        \bottomrule
    \end{tabular}
\end{table}

\begin{figure}[ht]
  \centering
  \includegraphics[]{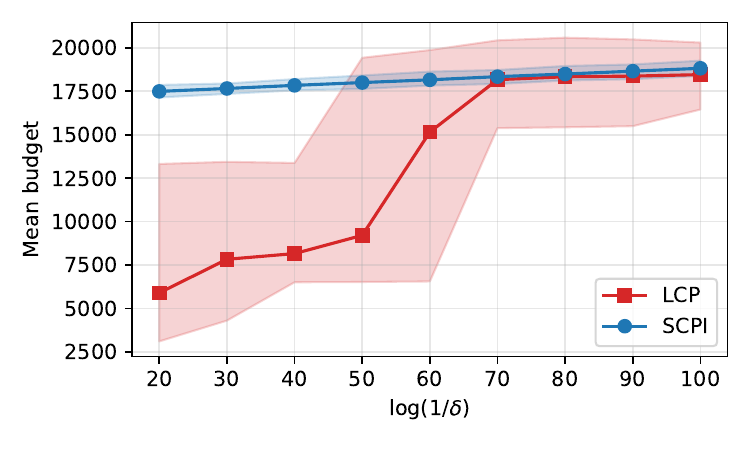}
  \caption{\textbf{Experiment 4}. Average budget for localizing one change point with varying $\delta$.}\label{fig:one_cp}
\end{figure}
We illustrate \textbf{Experiment 4} in Figure~\ref{fig:one_cp}. We consider one change point $x_1^*$ and $\Delta_1=1$. Again, we consider standard normal noise. In $1000$ Monte Carlo runs, we run \texttt{LCP} (with $\delta_{\mathrm{explore}}=1$) and \texttt{SCPI} with precision $\eta=2^{-7}$ and varying $\delta$. \texttt{SCPI} (\emph{Single Change Point Identification}) is an adaptation of \texttt{MCPI}, also provided in \citet{pmlr-v267-lazzaro25a} for identification of a single change point. Like in Experiment 2, we can see that while the budget of our method increases stronger for $\log(1/\delta)$ growing, we still achieve better performance for moderate choices of $\delta$.

\begin{figure}[ht]
  \centering
  \includegraphics[]{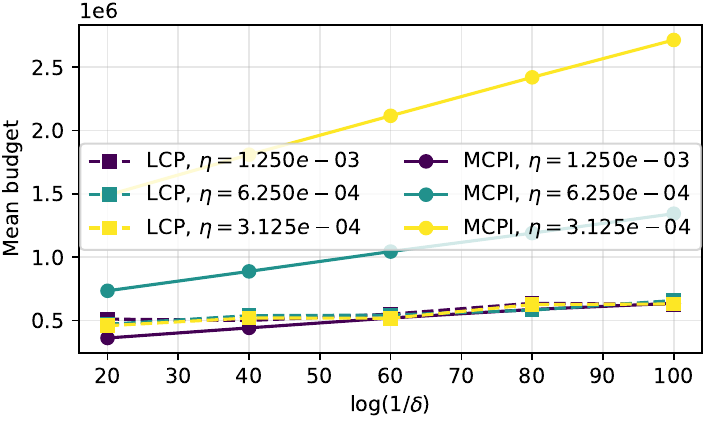}
  \caption{\textbf{Experiment 5}. Average budget for localizing $10$ change points for three values of $\eta$ and varying $\delta$.}\label{fig:many_cps}
\end{figure}
In \textbf{Experiment 5}, illustrated in Figure~\ref{fig:many_cps}, we consider a problem with $m=10$ change points and $\Delta_i=1$ for $i$ odd and $\Delta_i=-1$ for $i$ even. The noise is again supposed to be standard normal. We run \texttt{MCPI} and \texttt{LCP} (with $\delta_{\mathrm{explore}}=1$) with varying $\log(1/\delta)\in\{20,40,60,80,100\}$ and $\eta=0.0025\cdot2^{-i}$, $i=1,2,3$. For \texttt{LCP} we observe that only the choice of $\delta$ has a visible but mild effect on the average budget. For \texttt{MCPI}, increasing $\eta$ by a factor $4$ leads to a decrease of the budget by a factor $1/4$. Meanwhile, the slope seems to depend almost linearly on $\log(1/\delta)$.

\newpage
\section{Proofs of upper bounds}\label{appendix:UB}

\subsection{Multiple change point detection}

We now prove Lemma~\ref{lem:upperbounddetection}, which is stated in Section~\ref{sec:detection}.

\begin{proof}[Proof of Lemma~\ref{lem:upperbounddetection}]

    We refer to Section~\ref{sec:detection} and Algorithm~\ref{alg:di} for the introduction of the notation used in this proof.
	Fix one depth $d\in\{1,\dots,d_{\max}\}$ and one index $j\in\{0,\dots,n_d\}$. Let $y_j^{(d)}=f(x_j^{(d)})$ be the true function value at $x_j^{(d)}$. By assumption, the noise is $1$-subGaussian, so Hoeffding's inequality yields, for any $\varepsilon>0$,
	\begin{align*}
		\mathbb{P}\left(|\hat y_j^{(d)}-y_j^{(d)}|\geq \varepsilon\right)\leq 2\exp\left(-\frac{T_d\varepsilon^2}{2}\right)\enspace.
	\end{align*}

	From a union bound, first over all $n_d+1$ endpoints of depth $d$, and then over each depth $d=1,\dots,d_{\max}$, we obtain an event $\xi$ with $\mathbb{P}(\xi)\geq 1-\delta$ such that for all $d=1,\dots,d_{\max}$ and all $j=0,\dots,n_d$,
	\begin{align}\label{ineq:Hoeffdingfordetection}
		|\hat y_j^{(d)}-y_j^{(d)}|\leq \sqrt{\frac{2}{T_d}\log\left(\frac{2d_{\max}(n_d+1)}{\delta}\right)}=\beta_d/2.
	\end{align}

	First, on $\xi$, intervals without change points are not added to $\mathcal{I}$. Indeed, consider any depth $d$ and any interval $[x_{j-1}^{(d)},x_j^{(d)}]$ that does not contain a change point. Then $y_{j-1}^{(d)}=y_j^{(d)}$, and therefore, on $\xi$ (see~\eqref{ineq:Hoeffdingfordetection}),
	\begin{align*}
		|\hat y_j^{(d)}-\hat y_{j-1}^{(d)}|\leq |\hat y_j^{(d)}-y_j^{(d)}|+|y_{j-1}^{(d)}-\hat y_{j-1}^{(d)}|\leq \beta_d\enspace,
	\end{align*}
	and the interval is not added to $\mathcal{I}$ (Line~\ref{line:di:test}). This guarantees that every interval in $\mathcal{I}$ contains at least one change point, which proves the first part of the statement.

	Second, on $\xi$, change points are detected if the budget is large enough. Fix any change point $x_i^*$, with $i\in\{1,\dots,m\}$, jump size $\Delta_i$, and minimal spacing $s_i$ to the other change points. Recall that $d_{\max}=\lfloor \log_2(T/\log(1/\delta))\rfloor$, $n_d=2^d$, and define the depth specific to change point $i$
	\begin{align}
		d_i^*=\left\lceil \log_2\left(\frac{2}{s_i}\right)\right\rceil\leq d_{\max}\label{def:d^*_i}
	\end{align}
	The latter inequality ensures that depth $d_i^*$ is considered by the algorithm. It follows from
	\begin{align*}
		T\geq 2e\cdot \mathcal{E}_i^{-2}\cdot\log(1/\delta)\geq \frac{2e\cdot\log(1/\delta)}{s_i}\enspace,
	\end{align*}
	recalling that $\mathcal{E}_i^2=s_i\Delta_i^2\leq s_i$.

	Consider depth $d_i^*$ and the index $j$ such that $x_{j-1}^{(d_i^*)}<x_i^*\leq x_j^{(d_i^*)}$. By \eqref{def:d^*_i}, we have $2^{-d_i^*}\leq s_i/2$, so the definition of $s_i$ ensures that no other change point is contained in $[x_{j-1}^{(d_i^*)},x_j^{(d_i^*)}]$, and $y_j^{(d_i^*)}-y_{j-1}^{(d_i^*)}=\Delta_i$. Moreover, on $\xi$, using~\eqref{ineq:Hoeffdingfordetection} and the reverse triangle inequality,
	\begin{align*}
		|\hat y_j^{(d_i^*)}-\hat y_{j-1}^{(d_i^*)}|\geq |\Delta_i|-\beta_{d_i^*}\enspace.
	\end{align*}
	Now, if $T$ is large enough to ensure $\Delta_i\geq 2\beta_{d_i^*}$, then the test in Line~\ref{line:di:test} of Algorithm~\ref{alg:di} is positive, and $[x_{j-1}^{(d_i^*)},x_j^{(d_i^*)}]$ is added to $\mathcal{I}$ during the iteration at depth $d_i^*$. In particular, $x_i^*$ is contained in an interval of $\mathcal{I}$ with length at most $2^{-d_i^*}\leq s_i/2$.  Besides, at the following iterations at depth $d>d_i^*$, the algorithm maintains in $\mathcal{I}$ exactly one interval containing $x_i^*$, and whose length is at most $s_i/2$.

	Indeed, consider any depth $d>d_i^*$ and an interval $[x_{j-1}^{(d)},x_j^{(d)}]$ included in $[x_{j-1}^{(d_i^*)},x_j^{(d_i^*)}]$. If it does not contain a change point, it is not added to $\mathcal{I}$, as the event $\xi$ holds (Equation~\eqref{ineq:Hoeffdingfordetection}). Otherwise, $[x_{j-1}^{(d_i^*)},x_j^{(d_i^*)}]$ contains $x_i^*$ (as it is the unique change point in that interval). It may be added to $\mathcal{I}$, in which case the interval from $\mathcal{I}$ containing $x_i^*$ is removed. Therefore, at the end of Algorithm~\ref{alg:di}, $x_i^*$ is contained in exactly one interval of $\mathcal{I}$, with length at most $2^{-d_i^*}\leq s_i/2$.
	Moreover, since $f$ is right-continuous, an interval $[x_{j-1}^{(d)},x_j^{(d)}]$ with no change point in $(x_{j-1}^{(d)},x_j^{(d)}]$ satisfies $y_{j-1}^{(d)}=y_j^{(d)}$ and is therefore not added to $\mathcal{I}$ on $\xi$. Hence the interval $I\in\mathcal{I}$ containing $x_i^*$ satisfies $l(I)<x_i^*\leq r(I)$, where $x_i^*=r(I)$ is possible. Since $|I|\leq s_i/2$ and $s_i\leq\vartheta_{i-1}\wedge\vartheta_i$, we further get $l(I)\geq x_i^*-s_i/2>x_{i-1}^*$ and $r(I)< x_i^*+s_i/2<x_{i+1}^*$.

	It remains to determine a condition on $T$ that guarantees $\Delta_i\geq 2\beta_{d_i^*}$.

	Intuitively, by definition of $d_i^*$, we have $T_{d_i^*}\underset{\log}{\simeq} s_i\,T$ up to logarithmic factors, so $\Delta_i\geq 2\beta_{d_i^*}$ is roughly equivalent to $T\geq c\,\mathcal{E}_i^{-2}\log(1/\delta)$ up to logarithmic factors. The precise formulation is given in technical Lemma~\ref{lem:di:technical}, which concludes the proof.
\end{proof}

	\begin{lemma}\label{lem:di:technical}
		Let $0<\delta<1/4$, $\Delta_i\leq 1$. If
		\begin{align*}
			T&\geq 7.4\cdot 10^5 \cdot (\log(\mathcal{E}_i^{-2})\vee 1)\cdot\log \left(\frac{\log(\mathcal{E}_i^{-2})\vee 1}{s_i\cdot \delta}\right) \cdot \mathcal{E}_i^{-2} ~,
		\end{align*}
		then $T_{d_i^*}\geq 1$, and $\Delta_i \geq 2\beta_{d_i^*}$.
	\end{lemma}

\begin{proof}[Proof of Lemma~\ref{lem:di:technical}]
	Assume that $T\geq c\cdot (\log(\mathcal{E}_i^{-2})\vee 1)\cdot\log \left(\frac{\log(\mathcal{E}_i^{-2})\vee 1}{s_i\cdot \delta}\right) \cdot \mathcal{E}_i^{-2}$ for a constant $c$ we want to determine.

	First, we want to verify that $T_{d^*_i}\geq 1$. By definition of the budget allocated to depth $d^*_i$, we have
	\begin{align*}
		T_{d^*_i} = \left\lfloor\frac{T}{d_{\max}(n_{d^*_i}+1)}\right\rfloor \enspace .
	\end{align*}
	We have (since $\delta<1/4$)
	\begin{align*}
		d_{\max}\leq \log_2(T/\log(1/\delta))\leq\frac{\log(T)}{\log(2)} \\ \intertext{and}
		n_{d^*_i}+1 = 2^{\left\lceil\log_2\left(\frac{2}{s_i}\right)\right\rceil}+1\leq \frac{5}{s_i}\enspace ,\\
	\end{align*}
	so
	\begin{align*}
		\frac{d_{\max}(n_{d^*}+1)}{T}\leq\frac{5}{\log(2)}\cdot\frac{\log(T)}{s_i\cdot T}
		\enspace .
	\end{align*}
	Note that $s_i\leq 1/2$ (otherwise, we could simply pick any $T\geq 60 $, which could be relevant in the cases $i=1$ and $i=N$). Observe that $\mathcal{E}_i^2=s_i\Delta_i^2\leq s_i$, so that by assumption on $T$, we have
	\begin{align}
		T\geq c_1\cdot\frac{\log\left(\frac{1}{s_i}\right)}{s_i}\label{eq:ineqc1}.
	\end{align}
	Then we obtain from $\log(x)/x$ being decreasing for  $x\geq e$ that holds
	\begin{align*}
		\frac{5}{\log(2)}\cdot\frac{\log(T)}{s_i\cdot T}&\leq \frac{5}{\log(2)}\cdot\frac{\log(c_1)+\log\left(\frac{1}{s_i}\right)+\log\log\left(\frac{1}{s_i}\right)}{c_1\log\left(\frac{1}{s_i}\right)}\\
		&\leq \frac{5}{\log(2)^2}\cdot\frac{\log(c_1)+2}{c_1}\leq 1 \enspace ,
	\end{align*}
	where the last inequality holds for $c_1\geq 65$.

	Now, if $c_1\geq 65$, then $T_{d^*_i}\geq 1$, in which case $\beta_{d^*_i}$ is well-defined. Then, $\lfloor x\rfloor\geq x/2$ for $x\geq 1$, we can therefore obtain
	\begin{align*}
		\frac{32}{T_{d^*}}\leq 64\frac{d_{\max}(n_{d^*}+1)}{T}\leq \frac{320}{\log(2)}\cdot\frac{\log(T/\log(1/\delta))}{s_i\cdot T}
	\end{align*}
	The objective $\Delta_i \geq 2\beta_{d^*_i}$ is equivalent to the condition
	\begin{align*}
		4\beta_{d^*_i}^2=\frac{32}{T_{d^*_i}}\log\left(\frac{2d_{\max}(n_{d^*_i}+1)}{\delta}\right)\leq\Delta_i^2\enspace,
	\end{align*}
	which holds as long as
	\begin{align*}
		\frac{320}{\log(2)}\cdot\frac{\log(T/\log(1/\delta))\cdot \mathcal{E}_i^{-2}}{T}\cdot\log\left(\frac{10\log(T/\log(1/\delta))}{\log(2)\cdot s_i\cdot\delta}\right)\leq1 \enspace .
	\end{align*}
	Note that since $\delta<\frac{1}{4}$, we have
	\begin{align*}
		&\frac{320}{\log(2)}\cdot\frac{\log(T/\log(1/\delta))\cdot \mathcal{E}_i^{-2}}{T}\cdot\log\left(\frac{10\log(T/\log(1/\delta))}{\log(2)\cdot s_i\cdot\delta}\right)\\
		\leq &\frac{320}{\log(2)}\cdot\frac{\log(T/\log(1/\delta))\cdot\mathcal{E}_i^{-2}}{T}\\
		&\qquad \times \left[\left(\frac{5}{\log(2)^2}+1\right)\log(1/\delta)+\log\left(\frac{1}{s_i}\right)+\log\log(T/\log(1/\delta))\right]\enspace .
	\end{align*}
	Assume first, that $\mathcal{E}_{i}^{-2}\geq e$. We want to determine $c_2,c_3,c_4>0$ such that
	\begin{align}
		\frac{1600+320\log(2)^2}{\log(2)^3}\cdot\frac{\log(T/\log(1/\delta))\cdot\mathcal{E}_i^{-2}}{T}\cdot\log(1/\delta)\leq 1/3 \label{eq:ineqc2}\\
		\mathrm{for}\ T\geq c_2\cdot \mathcal{E}_i^{-2}\cdot \log(\mathcal{E}_i^{-2})\cdot \log(1/\delta)\ ;\notag\\\intertext{and}
		\frac{320}{\log(2)}\cdot\frac{\log(T/\log(1/\delta))\cdot\mathcal{E}_i^{-2}}{T}\cdot\log\left(\frac{1}{s_i}\right)\leq 1/3\label{eq:ineqc3}\\
		\mathrm{for}\ T\geq c_3\cdot \mathcal{E}_i^{-2}\cdot \log(\mathcal{E}_i^{-2})\cdot \log\left(\frac{1}{s_i}\right)\cdot\log\left(e\vee\log\left(\frac{1}{s_i}\right) \right)\ ;\notag\\\intertext{and}
		\frac{320}{\log(2)}\cdot\frac{\log(T/\log(1/\delta))\cdot\mathcal{E}_i^{-2}}{T}\cdot\log\log(T/\log(1/\delta))\leq 1/3\label{eq:ineqc4}\\
		\mathrm{for}\ T\geq c_4\cdot \mathcal{E}_i^{-2}\cdot \log(\mathcal{E}_i^{-2})\cdot \log(\log(\mathcal{E}_i^{-2})\vee e) \enspace .\notag
	\end{align}

	By the monotonicity of $\log(x)/x$ for $x\geq e$, bounding $T$ in inequality~\eqref{eq:ineqc2} yields
	\begin{align*}
		\frac{\log(T/\log(1/\delta))\cdot\mathcal{E}_i^{-2}}{T}\cdot\log(1/\delta)&\leq \frac{\log(c_2)+\log(\mathcal{E}_i^{-2})+\log\log(\mathcal{E}_i^{-2})}{c_2\cdot\log(\mathcal{E}_i^{-2})}\\
		&\leq \frac{\log(c_2)+2}{c_2}\leq \frac{\log(2)^3}{4800+960\log(2)^2}
	\end{align*}
	for $c_2\geq 2.3\cdot 10^5$.

	A similar approach for inequality~\eqref{eq:ineqc3} yields
	\begin{align*}
		\frac{\log(T/\log(1/\delta))\cdot\mathcal{E}_i^{-2}}{T}\cdot\log\left(\frac{1}{s_i}\right)\\
		\leq \frac{\log(c_3)+\log(\mathcal{E}_i^{-2})+\log\log(\mathcal{E}_i^{-2})+\log\log\left(\frac{1}{s_i}\right)+\log\log\left(e\vee\log\left(\frac{1}{s_i}\right)\right)}{c_3\cdot \log(\mathcal{E}_i^{-2})\cdot \log\left(e\vee\log\left(\frac{1}{s_i}\right) \right)}\\
		\leq \frac{\log(c_3)+4}{c_3}\leq\frac{\log(2)}{960}
	\end{align*}
	for $c_3\geq 2\cdot 10^4$.

	In the same way, we obtain for inequality~\eqref{eq:ineqc4}, using
	\begin{align}
		\log\log(xy)\leq \log2 (\log x\vee y)\leq \log\log x +2\log y \quad \textrm{for }x>e,\ y\geq2 \label{eq:loglog}
	\end{align}
	 that for $c_4\geq 2$
	\begin{align*}
		\frac{\log(T/\log(1/\delta))\cdot\mathcal{E}_i^{-2}}{T}\cdot\log\log(T/\log(1/\delta))\\
		\leq \frac{\left[\log(c_4)+3\log(\mathcal{E}_i^{-2})\right]\cdot\left[5\log\log(\mathcal{E}_i^{-2})+2\log(c_4)\right]}{c_4\cdot  \log(\mathcal{E}_i^{-2})\cdot \log(\log(\mathcal{E}_i^{-2})\vee e)}\\
		\leq\frac{2\log(c_4)^2+11\log(c_4)+15}{c_4}\leq\frac{\log(2)}{960}
	\end{align*}
	for $c_4\geq7.4\cdot10^5$.

	In the case $\mathcal{E}_i^{-2}\leq e$, we want to prove
		\begin{align}
		\frac{(1600+320\log(2)^2)e}{\log(2)^3}\cdot\frac{\log(T/\log(1/\delta))}{T}\cdot\log(1/\delta)\leq 1/3 \label{eq:ineqc2b}\\
		\mathrm{for}\ T\geq c'_2\cdot \log(1/\delta)\ ;\notag\\\intertext{and}
		\frac{320e}{\log(2)}\cdot\frac{\log(T/\log(1/\delta))}{T}\cdot\log\left(\frac{1}{s_i}\right)\leq 1/3\label{eq:ineqc3b}\\
		\mathrm{for}\ T\geq c'_3\cdot  \log\left(\frac{1}{s_i}\right)\cdot\log\left(e\vee\log\left(\frac{1}{s_i}\right) \right)\ ;\notag\\\intertext{and}
		\frac{320e}{\log(2)}\cdot\frac{\log(T/\log(1/\delta))}{T}\cdot\log\log(T/\log(1/\delta))\leq 1/3\label{eq:ineqc4b}\\
		\mathrm{for}\ T\geq c'_4 \enspace .\notag
	\end{align}

	Indeed, for inequality~\eqref{eq:ineqc2b} we have
	\begin{align*}
		\frac{\log(T/\log(1/\delta))}{T}\cdot\log(1/\delta)\leq\frac{\log(c_2')}{c_2'}\leq \frac{\log(2)^3}{3e\cdot(1600+320\log(2)^2)}
	\end{align*}
	for $c_2'\geq 5.7\cdot10^5$.

	For inequality~\eqref{eq:ineqc3b} we obtain
	\begin{align*}
		\frac{\log(T/\log(1/\delta))}{T}\cdot\left(\frac{1}{s_i}\right)\leq\frac{\log(c_3')+2}{c_3'}\leq\frac{\log(2)}{960e}
	\end{align*}
	for $c_3'\geq4.9\cdot10^4$.

	We finally observe that inequality~\eqref{eq:ineqc4b} is true for $c_4'\geq 1.1\cdot10^5$, since $\delta<1/4$.

	All together, we have shown that if $T\geq c\cdot (\log(\mathcal{E}_i^{-2})\vee 1)\cdot\log \left(\frac{\log(\mathcal{E}_i^{-2})\vee 1}{s_i\cdot \delta}\right) \cdot \mathcal{E}_i^{-2}$ for $c\geq \max\{c_1,c_2,c_3,c_4,c_2',c_3',c_4'\}$, then $T_{d^*_i}\geq 1$ and $\Delta_i \geq 2\beta_{d^*_i}$.
	\end{proof}

\subsection{Estimation of the jumps}\label{app:estimate_jumps}

Algorithm~\ref{alg:emj} takes as input a set $\mathcal{I}$ of candidate intervals, a confidence level $\delta$, a budget $T$, and a target number $N$ of jump estimates. It outputs a subset $\mathcal{J} \subset \mathcal{I}$ of accepted intervals, and a set $\mathcal{G}$ of estimated jump magnitudes for the accepted intervals.

\paragraph{Description of the procedure.}
The procedure proceeds in rounds. At round $k$, each active interval $I\in\mathcal{I}\setminus\mathcal{J}$ receives $2^{k-1}$ samples at each endpoint $x_l(I)$ and $x_r(I)$ (Lines~\ref{lin:emj:for-active}--\ref{lin:emj:sample}), allowing to compute empirical means $\hat y_l^{(k)}(I)$ and $\hat y_r^{(k)}(I)$. The empirical jump estimate is
\[
\hat\Delta^{(k)}(I)=\big|\hat y_r^{(k)}(I)-\hat y_l^{(k)}(I)\big|
\]
(Line~\ref{lin:emj:estimate}).

An interval is accepted when this estimate exceeds a confidence threshold
\[
\sqrt{{2^{-(k-5)}}\log\!\left(\frac{\pi^2\,|\mathcal{I}|\,k^2}{3\delta}\right)}
\]
(Line~\ref{lin:stopmultiplejumps}). When $I$ is accepted at round $k$, its estimate $\hat\Delta^{(k)}(I)$  is added to $\mathcal{G}$ and the interval is added to $\mathcal{J}$ (Line~\ref{lin:emj:accept}), so it is no longer sampled in later rounds. The running budget counter $\tau$ is updated after each batch of samples (Line~\ref{lin:emj:update-budget}), and the while-loop (Line~\ref{lin:emj:while}) stops once either more than $N$ intervals are accepted ($|\mathcal{G}|\geq N$) or the next round would exceed the total budget $T$ ($\tau+|\mathcal{I}\setminus\mathcal{J}|\cdot 2^k > T$).

This classic design yields an adaptive multiple estimation strategy: large jumps are typically accepted early, while smaller jumps receive additional rounds of sampling, under a global budget constraint. This procedure is quite standard, and is very similar to existing elimination algorithms for $\texttt{TopM}$ identification in multi-armed bandits~\citep[see, e.g.,][]{bubeck2013multiple,kalyanakrishnan2012pac}.

\begin{algorithm}
	\caption{\texttt{EstimateJumps}}\label{alg:emj}
	\begin{algorithmic}[1]
		\Require $\mathcal{I}$ set of intervals, $\delta$ confidence parameter, $T$ budget, $N$ number of desired estimates
		\Ensure $\mathcal{J}$ subset of $\mathcal{I}$, $\mathcal{G}$ estimated jumps for intervals in $\mathcal{J}$
		\State $\hat\Delta(I)\gets 0$ for $I\in\mathcal{I}$, $k\gets 1$, $\tau\gets0$, $\mathcal{G}\gets \emptyset$, $\mathcal{J}\gets\emptyset$
		\label{lin:emj:init}
		\While{$|\mathcal{G}|<N$ and $\tau+|\mathcal{I}\setminus\mathcal{J}|\cdot 2^k\leq T$} \Comment{Sample until $N$ estimates or budget exhausted}
		\label{lin:emj:while}
		\For{$[x_l(I),x_r(I)]=I\in\mathcal{I}\setminus\mathcal{J}$} \Comment{Sample active intervals}
		\label{lin:emj:for-active}
		\State sample $2^{k-1}$ times from $x_l(I)$ and $x_r(I)$, store means as $\hat y_l^{(k)}(I)$ and $\hat y_r^{(k)}(I)$
		\label{lin:emj:sample}
		\State $\hat\Delta^{(k)}(I)\gets |\hat y_r^{(k)}(I)-\hat y^{(k)}_l(I)| $\Comment{Estimate jump at endpoints}\label{lin:emj:estimate}
		\If{$\hat \Delta^{(k)}(I)\geq\sqrt{2^{-(k-5)}\log\left(\frac{\pi^2\cdot|\mathcal{I}|\cdot k^2}{3\delta}\right)}$} \Comment{Accept if estimate exceeds threshold}
		\label{lin:stopmultiplejumps}
		\State $\mathcal{G}\gets \mathcal{G}\cup\{\hat\Delta^{(k)}(I)\}$, $\mathcal{J}\gets \mathcal{J}\cup\{I\}$ \Comment{Add to accepted set}
		\label{lin:emj:accept}
		\EndIf
		\State $\tau\gets \tau+2^k$ \Comment{Update budget}
		\label{lin:emj:update-budget}
		\EndFor
		\State $k\gets k+1$ \Comment{Double samples for next round}
		\label{lin:emj:next-round}
		\EndWhile
	\end{algorithmic}
\end{algorithm}

\begin{lemma}\label{lem:multiplejumps}
	Let $M=|\mathcal{I}|$ and number the intervals $I\in \mathcal{I}$ in a way, such that if $I=[l(I),r(I)]$ and $\Delta(I)=f(r(I))-f(l(I))$, then $|\Delta(I_{1})|\geq|\Delta(I_{2})|\geq\dots|\Delta(I_{M})|$. Let $\delta<1/2$ and consider the output of Algorithm~\ref{alg:emj} $\mathcal{J}=\{J_{1},\dots,J_{N'}\}$ and $\mathcal{G}=\left\{\hat\Delta\left(J_1\right),\dots,\hat\Delta\left(J_{N'}\right)\right\}$ such that $\hat\Delta\left(J_1\right)\geq \dots\geq \hat\Delta\left(J_{N'}\right)$.

	Then there exists an event of probability at least $1-\delta$ such that, for every $v=1,\dots,N'$,
	\begin{align*}
	0<\frac{1}{2}\hat\Delta(J_v)\leq |\Delta(I_{v})|\leq \frac{3}{2}\hat\Delta(J_v)\enspace .
	\end{align*}
		and
		\begin{align*}
		0<\frac{1}{2}\cdot \hat{\Delta}\left(J_v\right)\leq |\Delta(J_v)|\leq\frac{3}{2}\cdot\hat{\Delta}\left(J_v\right)\enspace .
	\end{align*}
	Additionally, if
	\begin{align*}
		T\geq& c \cdot\Bigg[ \sum_{v=1}^N\frac{1}{\Delta(I_{v})^2}\cdot \log\left(\frac{M\cdot(\log(1/\Delta(I_{v})^2)\vee 1)}{\delta}\right)\\
		&+(M-N)\cdot\frac{1}{\Delta(I_{N})^2}\cdot \log\left(\frac{M\cdot(\log(1/\Delta(I_{N})^2)\vee 1)}{\delta}\right)\Bigg]\enspace ,
	\end{align*}
	for $c\geq 223$, then on the same event $N'\geq N$.
\end{lemma}

\begin{proof}[Proof of Lemma~\ref{lem:multiplejumps}]
Let $k\geq 1$ be some loop iteration of the algorithm. We define the quantity $\beta_k=\sqrt{2^{-(k-3)}\log\left(\frac{\pi^2\cdot M\cdot k^2}{3\delta}\right)}$ for twice the threshold in Line~\ref{lin:stopmultiplejumps}.

For a given interval $I$, the empirical means $\hat y_l^{(k)}(I)$ and $\hat y_r^{(k)}(I)$ are computed from $2^{k-1}$ independent samples from $f(l(I))$ and $f(r(I))$, respectively (Line~\ref{lin:emj:estimate}). By Hoeffding's inequality, using the subGaussian noise assumption, for any fixed interval $I$ and iteration $k$, we have
\begin{align*}
\mathbb{P}\left(\left|\hat\Delta^{(k)}(I)-\Delta(I)\right|>\beta_k\right)\leq 2\exp\left(-\frac{2^{k-2}\cdot\beta_k^2}{2}\right)=\frac{6\delta}{\pi^2\cdot M\cdot k^2}\enspace .
\end{align*}

Since $\sum_{k\geq 1}\frac{1}{k^2}=\frac{\pi^2}{6}$, a union bound yields on an event $\xi$ with $\mathbb{P}(\xi)\geq 1-\delta$, on which
	\begin{align}
		\left|\hat\Delta^{(k)}(I)-\Delta(I)\right|\leq \beta_k\enspace \label{eq:HoeffdingMultipleJumps}\enspace
	\end{align}
	for all iterations $k$ and all intervals $I$ considered still active at iteration $k$.

	First, we want to show that on $\xi$, for every accepted interval $J_v\in\mathcal{J}$, we have $\frac{1}{2}\hat\Delta(J_v)\leq |\Delta(J_v)|\leq \frac{3}{2}\hat\Delta(J_v)$. Note that if $J_v\in\mathcal{J}$, then there exists an iteration $k^*(J_v)$ where the algorithm stops considering the interval and accepts it. In particular, at iteration $k^*(J_v)$ the condition in Line~\ref{lin:stopmultiplejumps} is satisfied, so $\hat\Delta(J_v)=\hat\Delta^{(k^*(J_v))}(J_v)$.

	By inequality~\eqref{eq:HoeffdingMultipleJumps}, we know from Line~\ref{lin:stopmultiplejumps}
	\begin{align}
		\hat\Delta(J_v)\geq 2\beta_{k^*(J_v)} \notag\\
		\Rightarrow \left|\hat\Delta(J_v)-\left|\Delta(J_v)\right|\right|\leq \frac{1}{2}\hat\Delta(J_v)\enspace . \label{Eq:sameinterval}
	\end{align}

	This proves the second inequality of Lemma~\ref{lem:multiplejumps} and if $|\Delta(J_v)|=|\Delta(I_{v})|$, this coincides with the first inequality.

	So assume that $|\Delta(J_v)|<|\Delta(I_{v})|$. Then by inequality~\eqref{Eq:sameinterval}, we know that
	\begin{align*}
		\frac{1}{2}\hat\Delta(J_v)\leq|\Delta(J_v)|<|\Delta(I_{v})|\enspace .
	\end{align*}
	On the other hand, the condition $|\Delta(J_v)|<|\Delta(I_{v})|$ means that there exists an $w\leq v$ such that either $I_{w}\notin \mathcal{J}$ or $\hat\Delta(I_{w})\leq\hat \Delta(J_v)$. If $I_{w}\notin\mathcal{J}$ this means that the condition in Line~\ref{lin:stopmultiplejumps} is false for any iteration, in particular at iteration $k^*(J_v)$. By the concentration inequality~\eqref{eq:HoeffdingMultipleJumps}, we obtain
	\begin{align*}
		|\Delta(I_{v})|\leq|\Delta(I_{w})|\leq |\hat \Delta(I_{w})^{(k^*(J_v))}- \Delta(I_{w})|+\hat \Delta(I_{w})^{(k^*(J_v))}<3\beta_{k^*(J_v)}\leq\frac{3}{2}\hat{\Delta}(J_v)\enspace .
	\end{align*}
	If $I_{w}\in \mathcal{J}$ but $\hat\Delta(I_{w})\leq\hat\Delta(J_v)$, then there exists an iteration $k^*(I_{w})$ where Algorithm~\ref{alg:emj} stops considering the interval and the condition
	\begin{align*}
		\hat\Delta(I_{w})\geq 2\beta_{k^*(I_{w})}
	\end{align*}
	holds. If $k^*(I_{w})>k^*(J_v)$, we can use the same argumentation as for the case $I_{w}\notin \mathcal{J}$, so assume without loss of generality that $k^*(I_{w})\leq k^*(J_v)$. Then again by inequality~\eqref{eq:HoeffdingMultipleJumps} we know that
	\begin{align*}
		|\Delta(I_{v})|&\leq|\Delta(I_{w})|\leq |\Delta(I_{w})-\hat\Delta(I_{w})|+\hat\Delta(I_{w})\\
		&\leq \beta_{k^*(I_{w})}+\hat{\Delta}(I_{w})\leq\frac{3}{2}\hat{\Delta}(I_{w})\leq \frac{3}{2}\hat\Delta(J_v)\enspace ,
	\end{align*}
	which proves the claimed inequality.

  We can prove the claim in the case $|\Delta(J_v)|>|\Delta(I_{v})|$ with the same arguments. Note that in this case, there is an $w\geq v$ with $\hat\Delta(I_{w})\geq \hat{\Delta}(J_v)$.

	Now, let us prove the second part of the Lemma. Consider some interval $I\in\mathcal{I}$.  Assume that $
		3\beta_k\leq |\Delta(I)|$.
	Together with inequality~\eqref{eq:HoeffdingMultipleJumps} this implies $\hat{\Delta}(I)\geq 2\beta_k$. This means, that if we reach the smallest $k$, such that
	\begin{align*}
		\frac{32}{9\cdot2^k}\log\left(\frac{\pi^2\cdot M\cdot k^2}{3\delta}\right)\leq \Delta(I)^2\enspace ,
	\end{align*}
	before the algorithm terminates, we will have $I\in\mathcal{J}$ and $I$ will not be considered in later sampling steps. We have assumed, that $\delta <1/2$, so we can rewrite
	\begin{align*}
		\log\left(\frac{\pi^2\cdot M\cdot k^2}{3\delta}\right)\leq 2\log_2(k)+3\log\left(\frac{M}{\delta}\right)\enspace .
	\end{align*}
	If we choose
	\begin{align*}
		k =\left\lceil \log_2\left( c\cdot\frac{\log\left(\frac{M\cdot(\log(1/\Delta(I)^2)\vee 1)}{\delta}\right)}{\Delta(I)^2}\right)\right\rceil
	\end{align*}
	for some $c>0$ large enough, we obtain with inequality~\eqref{eq:loglog} and $\log_2(\lceil\log_2(x)\rceil)\leq 1+ \frac{1}{\log(2)}\log(\log(x)\vee1)$ that
	\begin{align*}
		&\frac{32}{9\cdot\Delta(I)^2}\cdot\frac{\log\left(\frac{\pi^2\cdot M\cdot k^2}{3\delta}\right)}{2^k}\\
        &\leq\frac{32}{9}\cdot\frac{4\log\left(\frac{M}{\delta}\right)+\frac{2}{\log(2)}\cdot\log\left(\log\left(c\cdot\frac{\log\left(\frac{M\cdot(\log(1/\Delta(I)^2)\vee 1)}{\delta}\right)}{\Delta(I)^2}\right)\vee 1\right)}{c\cdot \left(\log\left(\frac{M}{\delta}\right)+\log\left(\log(1/\Delta(I)^2)\vee 1\right)\right)} \\
		&\leq \cdot\frac{32(4\log(2)+4)\log\left(\frac{M}{\delta}\right)+2\log(\log(1/\Delta(I)^2)\vee 1)+4\log(c)}{9\log(2)\cdot c\cdot \left(\log\left(\frac{M}{\delta}\right)+\log\left(\log(1/\Delta(I)^2)\vee 1\right)\right)}\\
		&\leq \frac{128(\log(2)+1)}{9\log(2)}\frac{1+\log(c)}{c}\leq 1
	\end{align*}
	for $c\geq 223$.

	So for sampling the arms corresponding to interval $I$, we use a budget of at most
	\begin{align*}
		4c\cdot\frac{\log\left(\frac{M\cdot(\log(1/\Delta(I)^2)\vee 1)}{\delta}\right)}{\Delta(I)^2}
	\end{align*}
	and spend the remaining budget on the intervals in $\mathcal{I}\setminus\mathcal{J}$. The algorithm stops, when we added $N$ arms to $\mathcal{J}$, which yields the claimed bound of the total budget.
	\end{proof}

    \subsection{Single change point localization}

	Once we have detected an interval $I$ containing a change point, we want to localize it, at a precision of $\eta$. We adapt for that  the \texttt{Sequential halving with backtracking} (\texttt{SHB}) algorithm introduced in \citet{lazzaro2025fixedbudget} as their Algorithm~2, and recall their optimal guarantees for single change point localization, with fixed budget.

	\paragraph{Description of the procedure.}

	The algorithm takes as input a compact interval $I=[l(I),r(I)]$, a budget $T$, and a precision parameter $\eta>0$. It outputs a potential change point $c$.

	The algorithm maintains a tuple of five point arms $(l(I), l_d, c_d, r_d, r(I))$, where the boundary arms $l(I)$ and $r(I)$ are fixed throughout. At each round $d$, it collects $\tau$ samples from each arm and compares the resulting empirical means to decide whether to zoom into the left sub-interval $(l_d, c_d)$, the right sub-interval $(c_d, r_d)$. It may also \emph{backtrack} to the parent window when the evidence suggests the change point lies outside the current candidate region~\citep{lazzaro2025fixedbudget}. The procedure runs for $d_{\max} = \lceil 6\log(|I|/\eta) \rceil$ rounds, distributing the budget evenly across rounds and arms.

	The only (and minor) difference with the original \texttt{SHB} algorithm is that we stop the procedure early if $|I|\leq 2\eta$, and return the midpoint of the interval as the change point estimate. This is because in this case, we are already guaranteed to be within $\eta$ of the true change point, so there is no need to further zoom in. Moreover, we adapt the procedure to an interval of size $|I|$ instead of $[0,1]$, which only requires a rescaling of the arms and the precision parameter $\eta$.

    \begin{algorithm}
\caption{\texttt{SHB (Sequential Halving with Backtracking)}}\label{alg:shb}
		\begin{algorithmic}[1]
			\Require $I=[l(I),r(I)]$ compact interval, $T$ budget, $\eta>0$ precision parameter
            \Ensure potential change point $c$
            \State $c_1\gets \frac{l(I)+r(I)}{2}$, $l_1\gets l(I)$, $r_1\gets r(I)$
            \If{$|I|\leq 2\eta$}
            \State \Return $c_1$ \Comment{Return interval midpoint if target precision already achieved}
            \Else
            \State $d_{\max}\gets\lceil6\log(|I|/\eta)\rceil$, $\tau\gets \lfloor T/5d_{\max}\rfloor $, $\mathcal{A}^1\gets (l(I),l_1,c_1,r_1,r(I))$
            \For{$d=1,\dots,d_{\max}$}
            \State sample $\tau$ times from arms in $\mathcal{A}^d$, store means as $\hat y_{l(I)}^{(d)}$, $\hat y_{l_d}^{(d)}$, $\hat y_{c_d}^{(d)}$, $\hat y_{r_d}^{(d)}$, $\hat y_{r(I)}^{(d)}$
            \If{{\tiny $\left|\frac{\hat y_{l(I)}^{(d)}+\hat y_{l_d}^{(d)}}{2}-\frac{\hat y_{r_d}^{(d)}+\hat y_{r(I)}^{(d)}}{2}\right|<\frac{3}{4} \max\left(\left|\frac{\hat y_{l(I)}^{(d)}+\hat y_{l_d}^{(d)}+\hat y_{r_d}^{(d)}}{3}-\hat y_{r(I)}^{(d)}\right|,\left|\hat y_{l(I)}^{(d)}-\frac{\hat y_{l_d}^{(d)}+\hat y_{r_d}^{(d)}+\hat y_{r(I)}^{(d)}}{3}\right|\right)$}}\Comment{If Jump seems outside middle region}
            \State $\mathcal{A}^{d+1}\gets P(\mathcal{A}^d)$ \Comment{Backtrack: retreat to parent region}
            \ElsIf{$|\hat y_{l_d}^{(d)}-\hat y_{c_d}^{(d)}|\leq |\hat y_{c_d}^{(d)}-\hat y_{r_d}^{(d)}|$} \Comment{Jump seems closer to right}
            \State $\mathcal{A}^{d+1}\gets R(\mathcal{A}^d)=(l(I),c_d,(c_d+r_d)/2,r_d,r(I))$ \Comment{Zoom right}
            \Else \Comment{Jump seems closer to left}
            \State $\mathcal{A}^{d+1}\gets L(\mathcal{A}^d)=(l(I),l_d,(l_d+c_d)/2,c_d,r(I))$ \Comment{Zoom left}
            \EndIf
            \EndFor
            \State \Return $c_{d+1}$ \Comment{Return final midpoint}
            \EndIf
		\end{algorithmic}
	\end{algorithm}

   \begin{lemma}[\citealp{lazzaro2025fixedbudget}]\label{lem:shb}
        Assume that $x_{i-1}^*\leq l(I)<x_i^* \leq r(I)<x_{i+1}^*$. If $|I|\leq2\eta$, Algorithm~\ref{alg:shb} will return $c=\frac{l(I)+r(I)}{2}$ with $|x_i^*-c|\leq \eta$.

        Else, for a budget
        \begin{align*}
            T\geq \frac{600}{\Delta_i^2}\left(\log\left(\frac{1}{\delta}\right)+13\log\left(\frac{|I|}{4\eta}\right)\right)
        \end{align*}
        Algorithm~\ref{alg:shb} will return $c$ with $|x_i^*-c|\leq\eta$ with probability at least $1-\delta$.
    \end{lemma}

\begin{proof}
The first part is immediate: if $|I|\leq 2\eta$, the midpoint
$c=\frac{l(I)+r(I)}{2}$ satisfies $|c-x_i^*|\leq|I|/2\leq\eta$ for any
$x_i^*\in[l(I),r(I)]$.
For the second part, we apply Theorem~1 from \citet{lazzaro2025fixedbudget}
directly.\footnote{  \citet{lazzaro2025fixedbudget} uses a left-endpoint convention
  ( $f$ left-continuous, change point admissible at the left endpoint of the interval}).
  Our model is right-continuous, so Lemma~\ref{lem:shb} is its
  mirror image; the theorem applies after a left-right reflection.
\end{proof}

    \subsection{Verification of the localized change points}

\paragraph{Description of the procedure.}
\texttt{VerifyCP} is given in Algorithm~\ref{alg:vcp}, and is a very standard two-sample test.
It takes as input a candidate change point location $x$, an interval $I$ containing $x$, a confidence parameter $\delta$, a precision parameter $\eta$, and a budget $T$. It performs a simple two-sample test to check for the presence of a change point in the neighborhood of $x$, given by the interval $[x_-, x_+]$ where $x_-=(x-\eta)\vee \min I$ and $x_+=(x+\eta)\wedge \max I$, which corresponds to an $\eta$-neighborhood around $x$ included in $I$. As we see in Lemma~\ref{lem:vcp}, if there is no actual jump ($f(x_+)=f(x_-)$), the algorithm correctly returns \texttt{False} with high probability $1-\delta$. Conversely, if there is a jump of size $\Delta=|f(x_+)-f(x_-)|$, and the budget $T$ is sufficiently large (scaling as $1/\Delta^2$), the algorithm will correctly return \texttt{True} with high probability.

\begin{algorithm}
	\caption{\texttt{VerifyCP}}\label{alg:vcp}
	\begin{algorithmic}
    \Require $x_-$ left point, $x_+$ right point, $\delta$ confidence parameter, $T$ budget
	\Ensure \texttt{detection} (boolean)
	\State sample $\lfloor T/2\rfloor$ times from $x_-$ and $x_+$, store averages as $\hat y_-$ and $\hat y_+$
	\If{$T\geq 2$ and $|\hat y_+-\hat y_-|>\sqrt{\frac{16}{T}\log(2/\delta)}$}
	\State $\texttt{detection}\gets\texttt{True}$
	\Else
	\State $\texttt{detection}\gets\texttt{False}$
	\EndIf
	\end{algorithmic}
\end{algorithm}

\begin{lemma}\label{lem:vcp}
    Let $\delta<1/2$ and consider the output of Algorithm~\ref{alg:vcp} with input $x_-$, $x_+$, $\delta$,  and $T$. Let $\Delta :=f(x_+)-f(x_-)$.
	\begin{enumerate}
		\item If $\Delta=0$, with probability at least $1-\delta$ Algorithm~\ref{alg:vcp} returns \texttt{False}.
		\item If $\Delta\neq 0$ and $T\geq 64\cdot\frac{\log\left(\frac{2}{\delta}\right)}{\Delta^2}$, Algorithm~\ref{alg:vcp} returns \texttt{True} with probability at least $1-\delta$.
	\end{enumerate}
\end{lemma}

\begin{proof}[Proof of Lemma~\ref{lem:vcp}]

By Hoeffding's inequality, we know that for
\begin{align*}
    \beta = \sqrt{\frac{16}{T}\log(2/\delta)}
\end{align*}
holds
\begin{align*}
    \mathbb{P}\left(\left|(\hat y_+-\hat y_-)-(y_+-y_-)\right|\geq \beta\right)\leq \delta\enspace .
\end{align*}
So if $y_+=y_-$, then $|\hat y_+-\hat y_-|<\beta$ and Algorithm~\ref{alg:vcp} returns \texttt{False}, with probability at least $1-\delta$. If on the other hand $\Delta=|y_+-y_-|\geq 2\beta $ this guarantees, that $|\hat y_+-\hat y_-|>\beta$ and Algorithm~\ref{alg:vcp} will return \texttt{True}. This yields the condition
\begin{align*}
    \Delta \geq 2\cdot\sqrt{\frac{16}{T}\log(2/\delta)} \quad \Leftrightarrow \quad T\geq 64\cdot \frac{\log(2/\delta)}{\Delta^2}\enspace .& \qedhere
\end{align*}

\end{proof}
\subsection{Proof of the main theorem}
\begin{proof}[Proof of Theorem~\ref{thm:cpl}]

In this proof, we gather the guarantees on our four subroutines through the following Lemmas:
(i) Lemma~\ref{lem:upperbounddetection} for \texttt{DetectIntervals},
(ii) Lemma~\ref{lem:multiplejumps} for \texttt{EstimateJumps},
(iii) Lemma~\ref{lem:shb} for \texttt{SHB},
(iv) Lemma~\ref{lem:vcp} for \texttt{VerifyCP}.

First, we justify the correctness guarantee in point 1 of the theorem. Then, we fix a stage index $k$ in the doubling schedule and analyze the four steps of the procedure at this stage, deriving conditions on $k$ under which they provide an accurate estimate of $N$ change points and the stopping condition is reached. Finally, we derive a high-probability bound on the required budget when $\delta_{\mathrm{explore}}=\delta$, and an expectation bound when $\delta_{\mathrm{explore}}=1/4$.

Consider $\delta,\eta$ and $N$ as in the statement of the theorem, and any environment $\nu$ with $m\geq N$ change points. We run Algorithm~\ref{alg:lcp} with input parameters $(N,\delta,\eta,\delta_{\mathrm{explore}})$, with $\delta_{\mathrm{explore}}\in\{\delta,1/4\}$.

\paragraph{Correction.}
The correctness guarantee relies only on the verification step and is independent of the choice of $\delta_{\mathrm{explore}}$. Let us verify that if Algorithm~\ref{alg:lcp} stops at some stage $k$, the set $\mathcal{C}^{(k)}$ returned by the algorithm contains $N$ points $c_1^{(k)}<\dots<c_N^{(k)}$ with $|x_{l_v}^*-c_v^{(k)}|\leq \eta$ for $N$ distinct change points $1\leq l_1<l_2<\dots<l_N\leq m$.

If not, then there exists a stage index $k$, and some estimated change point $c_v^{(k)}\in J_v^{(k)}$, with $J_v^{(k)}\in \mathcal{J}^{(k)}$, such that  $[c_v^{(k)}-\eta,c_v^{(k)}+\eta]\cap J_v^{(k)}$ does not contain any change point, but $\mathrm{ok}_v^{(k)}=\texttt{VerifyCP}\left(l_{v}^{(k)},r^{(k)}_{v},\delta^{(k)},\left\lfloor\alpha_{v}^{(k)}\cdot 2^k\right\rfloor\vee 1\right)$ is True.

Consider such a $c_v^{(k)}$. Then, in $\texttt{VerifyCP}$, for $l_v^{(k)}=(c_v^{(k)}-\eta)\vee \min(J_v^{(k)})$ and $r_v^{(k)}=(c_v^{(k)}+\eta)\wedge \max(J_v^{(k)})$, we have $f(l_v^{(k)})=f(r_v^{(k)})$. By Lemma~\ref{lem:vcp}, there exists an event $\xi_{\mathrm{verify},v}^{(k)}$ with probability at least $1-\delta^{(k)}$ such that \texttt{VerifyCP} in Line~\ref{lin:stepverify} returns \texttt{False} for $c_v^{(k)}$. By a union bound, the probability of returning an incorrect set $C^{(k)}$ in step $k$ is therefore bounded by $N\delta^{(k)}=\frac{\delta}{4}\cdot \frac{6}{\pi^2(k)^2}$. Now, since
	 \begin{align*}
	 	\sum_{k'=\lceil\log_2(N)\rceil}^{+\infty}\frac{\delta}{4}\cdot \frac{6}{\pi^2(k')^2}\leq\frac{1}{4}\delta \enspace ,
	 \end{align*}
	another union bound tells us that, with probability at least $1-\delta$, the set $\mathcal{C}$ returned by Algorithm~\ref{alg:lcp} contains estimates of $N$ distinct change points at uniform precision $\eta$. This proves the correctness guarantee in point 1 of Theorem~\ref{thm:cpl}.

\paragraph{Construction of a good event for a fixed stage index $k$.}
Fix for now a stage index $k$ in the doubling schedule (the index of the global while loop), at which the budget used is $T_{k}=4\times 2^{k}$. We analyze the four steps of the procedure at this stage, and we derive conditions on $k$ under which they provide an accurate estimate of $N$ change points and the stopping condition is reached with probability at least $1-\delta_{\mathrm{explore}}$.

\textit{(i) Detection of changes in intervals of interest.}

For a fixed stage index $k$, consider the set $\mathcal{I}^{(k)}$ of intervals returned by \texttt{DetectIntervals} in Line~\ref{lin:stepdetection} of Algorithm~\ref{alg:lcp}, with confidence parameter $\delta_{\mathrm{explore}}/4$, and budget $T_k/4=2^k$.

Consider $\omega_i^a(\delta_{\mathrm{explore}}/4)$, and $c_1=7.4\cdot 10^5$ numerical constant from Lemma~\ref{lem:upperbounddetection}.
By Lemma~\ref{lem:upperbounddetection}, under the condition
\begin{equation}\label{eq:conditiondetection}
	2^k\geq c_1\cdot\max_{i=1,\dots,N}\omega_i^a\left(\frac{\delta_{\mathrm{explore}}}{4}\right)\cdot \mathcal{E}_i^{-2} \enspace,
\end{equation}
there is an event $\xi_{\mathrm{detect}}^{(k)}$ with $\mathbb{P}\left(\xi_{\mathrm{detect}}^{(k)}\right)\geq1-\delta_{\mathrm{explore}}/4$, on which the set $\mathcal{I}^{(k)}$ consists of $m$ disjoint intervals, each containing exactly one change point. Denote by $I_i^{(k)}$ the interval in $\mathcal{I}^{(k)}$ with $x_i^*\in I_i^{(k)}$; it follows from the lemma that $|I_i^{(k)}|\leq \frac{s_i}{2}$.

Note that since $\delta_{\mathrm{explore}} <1/4$ we have
	\begin{align*}
		\max_{i=1,\dots,N} \omega_i^a\left(\frac{\delta_{\mathrm{explore}}}{4}\right)\cdot \mathcal{E}_i^{-2}\leq 2 \left( \omega_1 H_{\mathrm{detect}}^{(m)}\left(\log\left(\frac{1}{\delta_{\mathrm{explore}}}\right)+\omega_2\right)\right) \enspace,
	\end{align*}
so that the condition~\eqref{eq:conditiondetection} is satisfied as soon as
\begin{equation}\label{eq:conditiondetection_sufficient}
2^k\geq 2c_1 \cdot \omega_1 H_{\mathrm{detect}}^{(m)}\left(\log\left(\frac{1}{\delta_{\mathrm{explore}}}\right)+\omega_2\right) \enspace.
\end{equation}

Assume for what follows that the event $\xi_{\mathrm{detect}}^{(k)}$ holds, and that condition \eqref{eq:conditiondetection} is satisfied.

\textit{(ii) Estimation of the jumps.}

Consider the subset $\mathcal{J}^{(k)}\subset \mathcal{I}^{(k)}$ of intervals returned by \texttt{EstimateJumps} in Line~\ref{lin:stepmultiplejumps} of Algorithm~\ref{alg:lcp}, with confidence parameter $\delta_{\mathrm{explore}}/4$ and budget $T_k/4=2^k$, and number of arms $N$. Let $\mathcal{G}^{(k)}$ denote the set of estimated jump magnitudes on the intervals in $\mathcal{J}^{(k)}=\{J_v^{(k)}\}_{v=1}^{N'}$, ordered decreasingly, so that $\hat \Delta(J_1^{(k)})\geq \dots \geq \hat \Delta(J_{N'}^{(k)})$, with $N'=|\mathcal{J}^{(k)}|$.
On $\xi_{\mathrm{detect}}^{(k)}$, the intervals in $\mathcal{I}^{(k)}$ are disjoint and each contains exactly one change point. In particular, the true jumps on these intervals are exactly the jump magnitudes of the $m$ change points of $f$.\footnote{Recall that we defined $\Delta(I)=f(l)-f(r)$ for the jump across an interval $I=[l,r]$.} Therefore, we can apply Lemma~\ref{lem:multiplejumps} to the set of intervals $\mathcal{I}^{(k)}$ returned by \texttt{DetectIntervals} and to the corresponding jumps $\Delta_1,\dots,\Delta_m$.

By Lemma~\ref{lem:multiplejumps}, there exists an event $\xi_{\mathrm{estimate}}^{(k)}$ with $\mathbb{P}(\xi_{\mathrm{estimate}}^{(k)})\geq 1-\delta_{\mathrm{explore}}/4$, so that, with $c_2$ a numerical constant from Lemma~\ref{lem:multiplejumps}, under the condition
	\begin{align}
		2^k\geq c_2\cdot\Bigg[ &\sum_{i=1}^N\frac{1}{\Delta_{(i)}^2}\cdot\log\left(\frac{4m\cdot(\log(1/\Delta_{(i)}^2)\vee 1)}{\delta_{\mathrm{explore}}}\right) \nonumber\\
        &+(m-N)\cdot\frac{1}{\Delta_{(N)}^2}\cdot\log\left(\frac{4m\cdot(\log(1/\Delta^2(I_{(N)}))\vee 1)}{\delta_{\mathrm{explore}}}\right)\Bigg] \label{eq:conditionestimate}
	\end{align}
	\texttt{EstimateJumps} returns a set $\mathcal{G}^{(k)}=\{\hat\Delta(J_v^{(k)}),\ v=1,\dots,N'\}$, with $N'\geq N$, with $\hat\Delta(J_1^{(k)})\geq \hat{\Delta}(J_2^{(k)})\geq \cdots \geq \hat\Delta(J_{N'}^{(k)})$, and such that for $v=1,\dots,N$,
	\begin{align}
	\frac{1}{2}\hat\Delta(J_v^{(k)})\leq|\Delta(J_v^{(k)})|\leq\frac{3}{2}\hat\Delta(J_v^{(k)}) \label{eq:estimate_jump_interval} \\
	\frac{1}{2}\hat\Delta(J_v^{(k)})\leq|\Delta_{(v)}|\leq\frac{3}{2}\hat\Delta(J_v^{(k)}) \label{eq:estimate_jump_order} \enspace.
	\end{align}
	  These inequalities will allow us to verify that the proportion $\alpha_v^{(k)}$ of the budget allocated to interval $J_v^{(k)}$ is proportional to $\Delta_{(v)}^{-2}$, which we interpret as the optimal allocation for localizing the change point in $J_v^{(k)}$.

	Let us now examine condition~\eqref{eq:conditionestimate}.  Naturally, estimating the jumps once they have been detected is less costly than the detection step itself.
	First, the energy $\mathcal{E}_i^2$ is the product of $\Delta_i^2$ and a spacing term, so $\Delta_i^{-2}\leq \mathcal{E}_i^{-2}$ for all $i$. This implies that all logarithmic terms in the condition~\eqref{eq:conditionestimate} are upper bounded by $2(\omega_2+\log(1/(\delta_{\mathrm{explore}})))$, with $\omega_2$ defined in the statement of the theorem.
	Secondly, one has
	\begin{align*}
	\sum_{i=1}^N\frac{1}{\Delta_{(i)}^2}+(m-N)\cdot\frac{1}{\Delta_{(N)}^2} & \leq \sum_{i=1}^m\frac{1}{\Delta_{(i)}^2}
	 \leq \sum_{i=1}^m\frac{1}{s_i\Delta_{(i)}^2}s_i\leq \left(\max_{j=1,\dots,m} {\mathcal{E}_j^{-2}}\right)\sum_{i=1}^m s_i  \leq 2 H_{\mathrm{detect}}^{(m)}
\enspace,
	\end{align*}
	where the last inequality holds since $\sum_{i=1}^m s_i\leq 2$, and $H_{\mathrm{detect}}^{(m)}=\displaystyle\max_{i=1,\dots,m} \mathcal{E}_i^{-2}$.

	All together, condition~\eqref{eq:conditionestimate} is implied by
	\begin{equation}\label{eq:conditionestimate_sufficient}
	2^k\geq 2c_2\cdot H_{\mathrm{detect}}^{(m)}\cdot (\omega_2+\log(1/\delta_{\mathrm{explore}}))
\end{equation}

	Now, we assume that both events $\xi_{\mathrm{detect}}^{(k)}$ and $\xi_{\mathrm{estimate}}^{(k)}$ hold, and that conditions~\eqref{eq:conditiondetection} and~\eqref{eq:conditionestimate} are satisfied. We are now in a position to analyze the last two steps of the procedure.

	\textit{(iii) Localization of the change points.}

	Let $v\in\{1,\dots,N\}$, and consider the interval $J_v^{(k)}$ returned by \texttt{EstimateJumps}, with the $v$-th largest estimated jump $\hat\Delta(J_v^{(k)})$. From the guarantees on \texttt{DetectIntervals}, we know that $J_v^{(k)}$ contains exactly one change point, with jump $\Delta(J_v^{(k)})$, and that the if this unique change point is $x_{j_v}^*$, then $|\Delta(J_v^{(k)})| = |\Delta_{j_v}|$, and $|J_v^{(k)}| \leq s_{j_v}$. Moreover, by Lemma~\ref{lem:upperbounddetection} and its proof, $x_{j_v-1}^*<l(J_v^{(k)})<x_{j_v}^*\leq r(J_v^{(k)})<x_{j_v+1}^*$, so that the assumption of Lemma~\ref{lem:shb} is satisfied.

	Consider the proportion $\alpha_v^{(k)}$ of the budget allocated to interval $J_v^{(k)}$ for localization, defined in Line~\ref{lin:stepbinsearch} of Algorithm~\ref{alg:lcp}. Let $c_v^{(k)}$ denote the output of the SHB algorithm (Algorithm~\ref{alg:shb}) applied to $J_v^{(k)}$ with budget $\lfloor\alpha_v^{(k)}\cdot 2^k \rfloor \vee 1$, where $\alpha_v^{(k)}$ is defined in~\eqref{def:alpha_v}.

	If $|J_v^{(k)}|\leq 2\eta$, Algorithm~\ref{alg:shb} takes the midpoint of the interval as the change point estimate $c_v^{(k)}$, and we have the guarantee $|c_v^{(k)}-x_{j_v}^*|\leq\eta$. Otherwise, by Lemma~\ref{lem:shb} from~\citet{lazzaro2025fixedbudget}, there exists an event $\xi^{(k)}_{\mathrm{loc},v}$ of probability at least $1-\delta_{\mathrm{explore}}/4N$ such that, under the condition
	 \begin{align}
	 	\lfloor\alpha_v^{(k)}\cdot 2^k \rfloor \vee 1\geq \frac{7200}{\Delta_{j_v}^2}\log\left(\frac{N}{\eta\cdot\delta_{\mathrm{explore}}}\right) \label{eq:conditionloc}\enspace ,
	 \end{align}
	 an estimate $c_v^{(k)}\in J_v^{(k)}$ with $|c_v^{(k)}-x_{j_v}^*|\leq \eta$.

	 Let us now find a condition on $k$ that implies the condition~\eqref{eq:conditionloc} for any $v=1,\dots,N$.
	 From Equations~\eqref{eq:estimate_jump_interval} and~\eqref{eq:estimate_jump_order}, we have $|\Delta(J_v^{(k)})|\geq \frac{1}{2}\hat\Delta(J_v^{(k)})\geq \frac{1}{2}\frac{2}{3} \Delta_{(v)} = \frac{1}{3} \Delta_{(v)}$, where $\Delta_{(v)}$ is the $v$-th largest true jump. In particular, this implies that
	 \[
	 H_{\mathrm{localize}}^{(N)}=\sum_{w=1}^N \Delta_{(w)}^{-2}\geq \frac{1}{9}\sum_{w=1}^N \Delta(J_w^{(k)})^{-2} \enspace.
	 \]

	 Moreover, from inequality~\eqref{eq:estimate_jump_interval}, we also obtain for any $v=1,\dots,N$,
	 \begin{align*}
	 	\alpha_{v}^{(k)}=\frac{\left(\hat\Delta(J_{v})\right)^{-2}}{\sum_{w=1}^N\left(\hat\Delta(J_{w})\right)^{-2}}
	 	\geq \frac{\left(2\Delta_{j_v}\right)^{-2}}{\sum_{w=1}^N \left(\frac{2}{3}\Delta_{j_w}\right)^{-2}}
	 	=\frac{1}{9}\frac{\Delta_{j_v}^{-2}}{\sum_{w=1}^N \Delta_{j_w}^{-2}}\enspace .
	 \end{align*}

	Now, under the assumption
	 \begin{align*}
	 	2^k\geq 1263600\log\left(\frac{N}{\eta\cdot\delta_{\mathrm{explore}}}\right) \sum_{w=1}^{N}\frac{1}{\Delta_{(w)}^2}\geq 140400\log\left(\frac{N}{\eta\cdot\delta_{\mathrm{explore}}}\right) \sum_{w=1}^{N}\frac{1}{\Delta_{j_w}^2}
	 \end{align*}
	 we obtain
	 \begin{align*}
	 	\left\lfloor\alpha_{i_v}^{(k)}\cdot2^k\right\rfloor \geq \frac{7200}{\Delta_{j_v}^2}\log\left(\frac{N}{\eta\cdot\delta_{\mathrm{explore}}}\right)\enspace .
	 \end{align*}
	 By a union bound, we obtain an event $\xi_{\mathrm{localize}}^{(k)}=\cap_{v=1}^N \xi_{\mathrm{loc},v}^{(k)}$ with $\mathbb{P}(\xi_{\mathrm{localize}}^{(k)})\geq 1-\delta/4$, on which $|c_v^{(k)}-x_{j_v}^*|\leq \eta$. This yields following sufficient condition for localization:
	 \begin{align}\label{eq:conditionloc_sufficient}
		2^k\geq 1263600\cdot H_{\mathrm{localize}}^{(N)}\cdot\left(\log\left(\frac{1}{\delta_{\mathrm{explore}}\eta}\right)+\omega_3\right)\enspace ,
	 \end{align}
	 using $\omega_3=\log\left(N\cdot\log(H_{\mathrm{localize}}^{(N)})\right)$ as defined in the theorem.

	For now, assume that we are on the event $\xi_{\mathrm{detect}}^{(k)}\cap\xi_{\mathrm{estimate}}^{(k)}\cap\xi_{\mathrm{localize}}^{(k)}$.

	\textit{(iv) Verification of the localized change points.}

	Finally, we seek a condition on $k$ such that, if the previous steps go well, the verification step also succeeds and the algorithm stops. This will hold under the following condition:
	 \begin{align*}
	 	2^k=c'\cdot\log\left(\frac{N\cdot\left(\log\left(\sum_{w=1}^N\frac{1}{\Delta_{(w)}^2}\right)\vee 1\right)}{\delta}\right)\cdot \sum_{w=1}^N\frac{1}{\Delta_{(w)}^2},
	 \end{align*}
	 where $c'>0$ is computed later. If we can show that
	 \begin{align*}
	 	\alpha_{i_v}^{(k)}\cdot 2^k\geq 128\cdot\frac{\log\left(\frac{4\pi^2\cdot N\cdot k^2}{3\cdot\delta}\right)}{\Delta_{j_v}^2}\enspace ,\quad i=1,\dots,N\enspace ,
	 \end{align*}
	 then
	 \begin{align*}
	 	\left\lfloor \alpha_{i_v}^{(k)}\cdot 2^k\right\rfloor \geq 64\cdot\frac{\log\left(\frac{4\pi^2\cdot N\cdot k^2}{3\cdot\delta}\right)}{\Delta_{j_v}^2}\enspace ,\quad v=1,\dots,N\enspace ,
	 \end{align*}
	 and Lemma~\ref{lem:vcp} together with a union bound show that, on the event $\xi_{\mathrm{verify}}^{(k)}$ with $\mathbb{P}(\xi^{(k)}_{\mathrm{verify}})\geq1-\frac{3\delta}{2\pi^2k^2}$ (already introduced in the correctness proof), \texttt{VerifyCP} in Line~\ref{lin:stepverify} returns \texttt{True} for all $i=1,\dots,N$.

	 Note that we have by the inequalities~\eqref{eq:estimate_jump_interval} and \eqref{eq:estimate_jump_order} that
	 \begin{align*}
	 	&\alpha_{i_v}^{(k)}\cdot c'\cdot\log\left(\frac{N\cdot\left(\log\left(\sum_{w=1}^N\frac{1}{\Delta_{(w)}^2}\right)\vee 1\right)}{\delta}\right)\cdot\sum_{w=1}^N\frac{1}{\Delta_{(w)}^2}\\
	 	&\geq\frac{c'}{9\Delta_{j_v}^2}\cdot\log\left(\frac{N\cdot\left(\log\left(\frac{1}{9}\sum_{i=1}^N\frac{1}{\Delta_{j_w}^2}\right)\vee 1\right)}{\delta}\right)\\
        &\geq\frac{c'}{18\Delta_{j_v}^2}\cdot\log\left(\frac{N\cdot\left(\log\left(\sum_{i=1}^N\frac{1}{\Delta_{j_w}^2}\right)\vee 1\right)}{\delta}\right)\enspace .
	 \end{align*}
	 The last inequality follows from the fact that for $\delta<1/4$ and $x>0$, we can bound
     \begin{align*}
         \log\left(\frac{N\cdot(\log(\frac{1}{9}x)\vee 1))}{\delta}\right)
         \geq\frac{1}{2}\log\left(\frac{N\cdot(\log(x)\vee 1)}{\delta}\right)\enspace ,
     \end{align*}
     since
     \begin{align*}
         \log\left(\frac{N\cdot(\log(\frac{1}{9}x)\vee 1))}{\delta}\right)
         -\frac{1}{2}\log\left(\frac{N\cdot(\log(x)\vee 1)}{\delta}\right)\\
         =\frac{1}{2}\log\left(\frac{N}{\delta}\right)+\log\left(\log\left(\frac{1}{9}x\right)\vee 1\right)-\frac{1}{2}\log\left(\log\left(x\right)\vee 1\right)\\
         \geq \log(2)+\log\left(\left(\log\left(x\right)-\log(9)\right)\vee 1\right)-\frac{1}{2}\log\left(\log\left(x\right)\vee 1\right)\enspace\eqqcolon g(x) .
     \end{align*}
     The function $g$ is continuous $[0,\infty)$ with $g(x)=\log(2)$ on $[0,e]$, $g(x)$ decreasing on $(e,9e)$ with $g(9e)=\log(2)-\frac{1}{2}\log(1+\log(9))\geq 0$ and increasing on $(9e,\infty)$, so we have indeed $g(x)\geq 0$.

     Therefore, we obtain
	 \begin{align*}
	 	\frac{\log\left(\frac{4\pi^2\cdot N\cdot k^2}{3\cdot\delta}\right)}{\alpha_{i_v}^{(k)}\cdot2^k\cdot\Delta_{j_v}^2}&\leq \frac{182 \log\left(\frac{N}{\delta}\right)+102\log\left(\log\left(\sum_{w=1}^N\frac{1}{\Delta_{(w)}^2}\right)\vee 1\right)+72\log(c')}{c'\cdot\left[\log\left(\log\left(\sum_{w=1}^N\frac{1}{\Delta_{j_w}^2}\right)\vee 1\right)+\log\left(\frac{N}{\delta}\right)\right]}\\
	 	&\leq \frac{182+72\log(c_4)}{c_4}\leq \frac{1}{128}
	 \end{align*}
	 when $c_4>1.4\cdot10^5$. For the second inequality, we used Equation~\eqref{eq:loglog} to obtain
	 \begin{align*}
	 	\log\log(2^k)&\leq\log\left(\log\left(\sum_{w=1}^N\frac{1}{\Delta_{(w)}^2}\right)\vee 1\right)\\&\qquad+2\log\left(c'\cdot\log\left(\frac{N\cdot\left(\log\left(\sum_{w=1}^N\frac{1}{\Delta_{(w)}^2}\right)\vee 1\right)}{\delta}\right)\right)\\
	 	&\leq 2\log\left(\frac{N}{\delta}\right)+3\log\left(\log\left(\sum_{w=1}^N\frac{1}{\Delta_{(w)}^2}\right)\vee 1\right)+2\log(c')\enspace .
	 \end{align*}

	 All together, using $\omega_3= \log\left({N\cdot\left(\log\left(H_{\mathrm{localize}}^{(N)}\right)\vee 1\right)}\right)$, we have shown that under the condition
	 \begin{align}
	 	2^k\geq 1.4\cdot10^5\cdot\left(\omega_3+\log(1/\delta)\right)\cdot H_{\mathrm{localize}}^{(N)}\label{eq:condition_verify_sufficient}\enspace ,
	 \end{align}
	 then on the event $\xi_{\mathrm{verify}}^{(k)}$, \texttt{VerifyCP} in Line~\ref{lin:stepverify} returns \texttt{True} for all $v=1,\dots,N$.

\paragraph{Bound in high probability.}
Now, we are ready to derive a high probability bound on the required budget when $\delta_{\mathrm{explore}}=\delta$.
We have a constant $c$ such that for any $k\geq \lceil\log_2(N)\rceil$, if $k$ satisfies
	\begin{align}
	 	2^k&\geq c\cdot\left(\omega_1 H_{\mathrm{detect}}^{(m)}\left(\log\left(\frac{1}{\delta}\right)+\omega_2\right)+H_{\mathrm{localize}}^{(N)}\left(\log\left(\frac{1}{\delta\eta}\right)+\omega_3\right)\right)\label{eq:condition_quantile}
		\enspace ,
	 \end{align}
	 then all conditions~\eqref{eq:conditiondetection_sufficient},\eqref{eq:conditionestimate_sufficient},\eqref{eq:conditionloc_sufficient},\eqref{eq:condition_verify_sufficient} are also verified (with $\delta_{\mathrm{explore}}=\delta$). From the previous paragraph, under this budget, on the intersection $\xi^{(k)}:=\xi_{\mathrm{detect}}^{(k)}\cap\xi_{\mathrm{estimate}}^{(k)}\cap\xi_{\mathrm{localize}}^{(k)}\cap\xi_{\mathrm{verify}}^{(k)}$, the stopping condition at the $k$-th stage of Algorithm~\ref{alg:lcp} is reached.
	 It is in particular sufficient to choose $k_*$ as the smallest integer such that this bound~\eqref{eq:condition_quantile} holds. Now, on the event $\xi=\xi_{\mathrm{detect}}^{(k_*)}\cap\xi_{\mathrm{estimate}}^{(k_*)}\cap\xi_{\mathrm{localize}}^{(k_*)}\cap(\cap_{k'=1}^{k_*}\xi_{\mathrm{verify}}^{(k')})$, then the algorithm stops at stage $k_*$ (or before), and the output is correct.

	Finally, one has
	$\mathbb{P}(\xi^{(k_*)})\geq 1-\frac{3}{4}\delta-\sum_{k'=1}^{k_*}\frac{6}{\pi^2k'^2}\cdot\frac{\delta}{4}\geq 1-\delta$.
	 Moreover, on $\xi$, the total budget used by the algorithm can be bounded as $\sum_{k'=\lceil\log_2(N)\rceil}^{k_*}4\cdot 2^{k'}<8\cdot 2^{k_*}$, which gives us the claimed bound on the total budget. This concludes the proof of the high probability bound in point 2 of Theorem~\ref{thm:cpl}.

	 \paragraph{Bound in expectation.}

	 Finally, we can also derive a bound on the expected budget. Fix $\delta_{\mathrm{explore}}=1/4$.

	 From the analysis of the four steps of the procedure, there exists a universal constant $c$ such that for any $k\geq \lceil\log_2(N)\rceil$, if $k$ satisfies
	 \begin{align}
		2^k&\geq c\cdot\left(\omega_1 H_{\mathrm{detect}}^{(m)}\left(\log(4)+\omega_2\right)+H_{\mathrm{localize}}^{(N)}\left(\log(4/\delta\eta)+\omega_3\right)\right)\label{eq:condition_expectation}\enspace ,
	\end{align}
	then all conditions~\eqref{eq:conditiondetection_sufficient},\eqref{eq:conditionestimate_sufficient},\eqref{eq:conditionloc_sufficient},\eqref{eq:condition_verify_sufficient} are verified. Under this budget constraint, on the intersection $\xi^{(k)}:=\xi_{\mathrm{detect}}^{(k)}\cap\xi_{\mathrm{estimate}}^{(k)}\cap\xi_{\mathrm{localize}}^{(k)}\cap\xi_{\mathrm{verify}}^{(k)}$, the stopping condition at the $k$-th stage of Algorithm~\ref{alg:lcp} is reached. Moreover, $\mathbb{P}(\xi^{(k)})\geq 1-\frac{3}{4}\cdot\frac{1}{4}-\sum_{k'=\lceil\log_2(N)\rceil}^{k}\frac{6}{\pi^2k'^2}\cdot\frac{\delta}{4}\geq \frac{3}{4}$, as $\delta\leq 1/4$ and $\delta_{\mathrm{explore}}=1/4$.

	Now, consider $k_0$, the smallest integer such that condition~\eqref{eq:condition_expectation} is satisfied. Then, on the event $\xi^{(k_0)}$, for any $k\geq k_0$, the stopping condition at the $k$-th stage of Algorithm~\ref{alg:lcp} is reached, and the output is correct with probability at least $\frac{3}{4}$. Therefore, if we denote by $K$ the random variable corresponding to the stage at which the algorithm stops, then $(K-k_0)\vee 0$ is stochastically dominated by a geometric random variable with parameter $3/4$, and $\mathbb{E}[2^{(K-k_0)\vee 0}]\leq \sum_{l=0}^{\infty}2^l(1/4)^l(3/4)=\frac{3}{2}$.

	Moreover, the total budget used by the algorithm can be bounded as
	\[
	\mathcal{T}\leq \sum_{k'=\lceil\log_2(N)\rceil}^{K}4\cdot 2^{k'}<8\cdot 2^{k_0}\cdot 2^{(K-k_0)\vee 0}\enspace,
	\]
	and
	\begin{align*}
	\mathbb{E}[\mathcal{T}] & \leq 8\cdot 2^{k_0}\cdot \mathbb{E}[2^{(K-k_0)\vee 0}] \leq 12\cdot 2^{k_0} \\
	2^{k_0} & \leq 2c\cdot\left(\omega_1 H_{\mathrm{detect}}^{(m)}\left(\log(4)+\omega_2\right)+H_{\mathrm{localize}}^{(N)}\left(\log(4/\delta\eta)+\omega_3\right)\right)\enspace .
	\end{align*}
	This concludes the proof of the bound in expectation in point 3 of Theorem~\ref{thm:cpl}.
\end{proof}

\section{Proofs of lower bounds}\label{appendix:LB}
\subsection{Proof of Theorem~\ref{thm:main_LB} and Corollary~\ref{cor:main_LB}}

Before proving Theorem~\ref{thm:main_LB}, we first state two lemmas, which are the main building blocks of the proof. The first is a lower bound showing the optimality of the localization complexity $H_{\mathrm{localize}}^{(m)}$, and the second is a lower bound scaling with $H_{\mathrm{detect}}^{(m)}$.

Fix $\delta\in(0,1/4)$ and $\eta\in(0,1/8)$.
For both lemmas, we consider any $(\delta,\eta,m)$-correct algorithm $\pi$ for the active change point detection problem, and any valid environment $\nu$ with $m\geq 1$ change points. We consider the case where all change points have to be localized, so that $N=m$. In all this section, we may omit the superscript $(m)$ in $H_{\mathrm{detect}}^{(m)}$ and $H_{\mathrm{localize}}^{(m)}$ for readability, and simply write $H_{\mathrm{detect}}$ and $H_{\mathrm{localize}}$. We may precise the dependence of these quantities on the environment $\nu$ when needed, and write $H_{\mathrm{detect}}^{\nu}$ and $H_{\mathrm{localize}}^{\nu}$.

\begin{lemma}\label{lem:LB_H_localize}
There exists an environment $\nu'$ such that (1) $\nu'$ and $\nu$ have the same jump magnitudes $\Delta_1,\dots,\Delta_m$  -- so that $H_{\mathrm{localize}}^{(m,\nu')}=H_{\mathrm{localize}}^{(m,\nu)}$; (2) $\forall i\in\{1,\dots,m\}$, $\frac{1}{4}s_i\leq s'_i\leq \frac{3}{4}s_i$  -- so that $\frac{4}{3} H_{\mathrm{detect}}^{(m,\nu)} \leq H_{\mathrm{detect}}^{(m,\nu')}\leq 4H_{\mathrm{detect}}^{(m,\nu)}$;
and
\begin{align*}
\mathbb{P}_{\pi,\nu'}\left(\mathcal{T}_{\pi}\geq H_{\mathrm{localize}}^{(m,\nu)}\log\left(\frac{1}{8\delta}\right)+\sum_{i=1}^m \frac{1}{\Delta_{i}^{2}}\log_+\!\left(\frac{s_i}{16\eta}\right)\right)\geq \delta  \enspace.
\end{align*}
\end{lemma}

\begin{lemma}\label{lem:LB_H_detect}
Assume that $\nu$ contains at least $2$ change points, $N=m\geq 2$. Then, there exists an environment $\nu'$ such that $H_{\mathrm{detect}}^{(m,\nu)}\leq H_{\mathrm{detect}}^{(m,\nu')}\leq 4H_{\mathrm{detect}}^{(m,\nu)}$, $\frac{1}{2}H_{\mathrm{localize}}^{(m,\nu)} \leq H_{\mathrm{localize}}^{(m,\nu')}\leq H_{\mathrm{localize}}^{(m,\nu)}$, and
\begin{align*}
\mathbb{P}_{\pi,\nu'}\left(\mathcal{T}_{\pi}\geq \frac{1}{2}H_{\mathrm{detect}}^{(m,\nu)}\log\left(\frac{1}{6\delta}\right) \right)\geq \delta  \enspace .
\end{align*}
\end{lemma}

From these two lemmas, we can now deduce Theorem~\ref{thm:main_LB}.

\begin{proof}[Proof of Theorem~\ref{thm:main_LB}]
The proof is a direct consequence of Lemmas~\ref{lem:LB_H_localize} and~\ref{lem:LB_H_detect}.

First, assume that $N=m=1$. Recall that in this case there is a single change point with energy $\mathcal{E}_1=\Delta_1^2$ and sparsity $s_1=1$, so $H_{\mathrm{detect}}=H_{\mathrm{localize}}$. We apply Lemma~\ref{lem:LB_H_localize} to obtain an environment $\nu'$ such that $H_{\mathrm{localize}}^{(m,\nu)}=H_{\mathrm{localize}}^{(m,\nu')}$ (and similarly for $H_{\mathrm{detect}}$) and such that
\begin{align*}
\mathbb{P}_{\pi,\nu'}\left(\mathcal{T}_{\pi}\geq \frac{1}{\Delta_1^2}\log\left(\frac{1}{8\delta}\right)+ \frac{1}{\Delta_{1}^{2}}\log_+\!\left(\frac{1}{16\eta}\right)\right)\geq \delta  \enspace,
\end{align*}
and the claimed bound follows from the fact that $H_{\mathrm{detect}}=H_{\mathrm{localize}}$ in this case.

Second, assume that $N=m\geq 2$. We apply Lemma~\ref{lem:LB_H_detect}, and Lemma~\ref{lem:LB_H_localize} to obtain an environment $\nu'$ such that $H_{\mathrm{detect}}^{(m,\nu)}\leq H_{\mathrm{detect}}^{(m,\nu')}\leq 4H_{\mathrm{detect}}^{(m,\nu)}$, $\frac{1}{2}H_{\mathrm{localize}}^{(m,\nu)} \leq H_{\mathrm{localize}}^{(m,\nu')}\leq H_{\mathrm{localize}}^{(m,\nu)}$, and $\mathbb{P}_{\pi,\nu'}\left(\mathcal{T}_{\pi}\geq \chi \right)\geq \delta$, with
\begin{align*}
\chi= &\max\left(\frac{1}{2}H_{\mathrm{detect}}^{(m,\nu)}\log\left(\frac{1}{6\delta}\right),H_{\mathrm{localize}}^{(m,\nu)}\log\left(\frac{1}{8\delta}\right)+\sum_{i=1}^m \frac{1}{\Delta_{i}^{2}}\log_+\!\left(\frac{s_i}{16\eta}\right)\right) \\
	\geq & \frac{1}{4}H_{\mathrm{detect}}^{(m,\nu)}\log\left(\frac{1}{8\delta}\right)+\frac{1}{2}H_{\mathrm{localize}}^{(m,\nu)}\log\left(\frac{1}{8\delta}\right)+\frac{1}{2}\sum_{i=1}^m \frac{1}{\Delta_{i}^{2}}\log_+\!\left(\frac{s_i}{16\eta}\right) \enspace,
\end{align*}
which concludes the proof.
\end{proof}

Now, we can deduce Corollary~\ref{cor:main_LB} from Theorem~\ref{thm:main_LB}, and Theorem 5.2 from~\citet{pmlr-v267-lazzaro25a}.

\begin{proof}[Proof of Corollary~\ref{cor:main_LB}]
First, observe that for any environment $\nu$ with $m\geq 1$ change points, the condition $\vartheta_i>2\eta$ for $i=2,\dots,m-1$ allows us to apply Theorem~5.2 from~\citet{pmlr-v267-lazzaro25a}, which states that for any $(\delta,\eta,m)$-correct algorithm $\pi$,
\begin{align}
\mathbb{E}_{\pi,\nu}[\mathcal{T}_{\pi}]\geq 4H_{\mathrm{localize}}^{(m,\nu)}\log\left(\frac{1}{4\delta}\right) \label{eq:exp_LB_1} \enspace.
\end{align}
Second, we obtain the lower bound in $\mathbb{E}[\mathcal{T}_{\pi}]$ independent of $\delta$ by applying Theorem~\ref{thm:main_LB} with $\delta=1/16$, observing that $\nu$ is in particular $(1/16,\eta,m)$-correct, as $\delta\leq 1/16$. From Theorem~\ref{thm:main_LB},  there exists an environment $\nu'$ such that $H_{\mathrm{detect}}^{(m,\nu)}\leq H_{\mathrm{detect}}^{(m,\nu')}\leq 4H_{\mathrm{detect}}^{(m,\nu)}$, $\frac{1}{2}H_{\mathrm{localize}}^{(m,\nu)} \leq H_{\mathrm{localize}}^{(m,\nu')}\leq H_{\mathrm{localize}}^{(m,\nu)}$, and such that
\begin{align*}
\mathbb{P}_{\pi,\nu'}\left(\mathcal{T}_{\pi}\geq\tfrac{1}{4}H_{\mathrm{detect}}^{(m,\nu)}\log\left(2\right)+\tfrac{1}{2}H_{\mathrm{localize}}^{(m,\nu)}\log\left(2\right)+\tfrac{1}{2}\sum_{i=1}^m \frac{1}{\Delta_{i}^{2}}\log_+\!\left(\frac{s_i}{16\eta}\right)\right)\geq 1/16  \enspace.
\end{align*}
Now, from Markov's inequality, we can lower bound $\mathbb{E}_{\pi,\nu'}[\mathcal{T}_{\pi}]$ as
\begin{align}
\mathbb{E}_{\pi,\nu'}[\mathcal{T}_{\pi}]&\geq \frac{1}{16}\left(\frac{1}{4}H_{\mathrm{detect}}^{(m,\nu)}\log\left(2\right)+\frac{1}{2}H_{\mathrm{localize}}^{(m,\nu)}\log\left(2\right)+\frac{1}{2}\sum_{i=1}^m \frac{1}{\Delta_{i}^{2}}\log_+\!\left(\frac{s_i}{16\eta}\right)\right)\label{eq:exp_LB_2}
\end{align}
Now, the final lower bound from Corollary~\ref{cor:main_LB} follows from Equations~\eqref{eq:exp_LB_1},\eqref{eq:exp_LB_2}, and the standard inequality $\max(a,b)\geq (a+b)/2$ for $a,b\geq 0$.

\end{proof}

\subsection{Proof of Lemma~\ref{lem:LB_H_localize}}

\begin{proof}[Proof of Lemma~\ref{lem:LB_H_localize}]
\textbf{Proof roadmap.}\quad
The proof has four steps.

First, we construct a family ${(\nu_J)}_J$, indexed by $J=(j_1,\dots,j_m)\in \prod_{i=1}^m \{0,\dots,\alpha_i-1\}$, by shifting each change point $x_i^*$ over $\alpha_i \sim s_i/8\eta$ candidate positions. We pay attention to preserve roughly the length of the change points,  so that the complexities $H_{\mathrm{detect}}^{(m,\nu)}$ is preserved (up to a numerical constant) under any of the considered shifts. Second, for each $i\in\{1,\dots,m\}$, we define an alternative environment $\nu_{J^{(-i)}}$ by removing from $\nu_J$ the $i$-th jump.
Third, we compare the behavior of the algorithm under $\nu_J$ and $\nu_{J^{(-i)}}$. For that, we consider the event $\xi_{i,j_i}=\{\mathcal{T}_{\pi}\leq\chi,\ |\hat x_i-x_i^{J}|\leq\eta\}$, where $x_i^{J}$ is the position of the $i$-th change point in environment $\nu_J$, and $\chi$ is the largest $(1-\delta)$-quantile of the budget under the family of environments ${(\nu_J)}_J$. We show that $\xi_{i,j_i}$ is likely under $\nu_J$ and unlikely on average under $\nu_{J^{(-i)}}$. These bounds allow us to lower bound the $\KL$ between the respective probability distributions, thanks to data-processing inequalities. Finally, we compute the KL divergence between the probability under $\nu_J$ and $\nu_{J^{(-i)}}$, which leads to a  lower bound on the quantile budget $\chi$, which translates Lemma~\ref{lem:LB_H_localize}.

\paragraph{Reference class of environments.}
Consider the parameters of the environment $\nu$ given by the jumps $\Delta_1,\dots,\Delta_m$, the initial mean $\mu_0$, and the change point positions $x_1^*,\dots,x_m^*$, and the function $f$ defined by Equation~\eqref{eq:model}. Recall that $x^*_0=0$ and $x^*_{m+1}=1$. Consider the length of the change points $\vartheta_0=1$, for $i=1,\dots,m-1$ $\vartheta_i=x_{i+1}^*-x_i^*$, and $\vartheta_m=1$.

Fix $\eta'>\eta$, which will be used to shift the change points. For $i=1,\dots,m$, define
\[
\alpha_i:=1 \vee \left\lfloor\frac{s_i}{8\eta'}\right\rfloor \enspace,
\]
as the number of possible positions for the change point $i$ in the environments that we will construct. Observe that $\alpha_i\geq 1$ for all $i=1,\dots,m$.

Let $J=(j_1,\dots,j_m)$ be a vector of size $m$, so that  for each $i=1,\dots,m$,  $0\leq j_i \leq \alpha_i-1$.
For each $i=1,\dots,m$, define
\begin{equation}
x_i^{J}= \frac{1}{2} x_i^*+ j_i \cdot 2\eta'
\end{equation}
By convention, $x_{0}^{J}:=0$ and $x_{m+1}^{J}:=1$.

For each $J=(j_1,\dots,j_m)$, define $\nu_{J}$ as the environment with the same jumps as $\nu$, and where the position of the change points are given by the vector  ${(x_i^{J})}_{i=1,\dots,m}$. Note that $\nu_{J}$ has also $m$ change points and that the jumps are unchanged. Moreover, for each $i=1,\dots,m-1$, the change points $x_i^{J}$ and $x_{i+1}^{J}$ are spaced by at least $\vartheta_i/4$ and at most $(3/4)\vartheta_i$, so that the complexities $H_{\mathrm{detect}}^{(m,\nu)}$ are preserved (up to a numerical constant) under any of the considered shifts. We point to Lemma~\ref{lem:complexity_nu_J} for a formal statement of this fact.

Denote as $\mathbb{P}_{\pi,J}$ the probability distribution induced by the interaction between an algorithm $\pi$, and the environment is $\nu_J$, and by $\mathbb{E}_{\pi,J}$ the associated expectation.

\paragraph{Alternative environments.}
Let $i\in\{1,\dots,m\}$ and fix for now $j_i\in\{0,\dots,\alpha_i-1\}$. Fix any $J\in\prod_{k=1}^{m} \{0,\dots,\alpha_k-1\}$, such that $j_i$ is the $i$-th component of $J$. Denote as $J^{(-i)}$ the vector obtained from $J$ by removing the $i$-th component $j_i$. Define $\nu_{J^{(-i)}}$ as the environment obtained from $\nu_J$ by setting the jump $\Delta_i$ to zero, so that $\nu_{J^{(-i)}}$ has $m-1$ change points. Note that $\nu_{J^{(-i)}}$ and $\nu_J$ only differ for arms in the interval $[x_i^{J},x_{i+1}^{J})$, and are identical elsewhere. Observe moreover that $\nu_{J^{(-i)}}$ does not depend on the index $j_i$, as the jump at change point $i$ is set to zero.

Denote as $\mathbb{P}_{\pi,J^{(-i)}}$ the probability distribution of the observations when the algorithm is $\pi$, and the environment is $\nu_{J^{(-i)}}$ and by $\mathbb{E}_{\pi,J^{(-i)}}$ the associated expectation.

\paragraph{Bound on total variation.}
Let $\pi$ be an algorithm for this problem, which is $(\delta,\eta)$-correct, and let $\hat x_1,\dots,\hat x_m$ be the corresponding estimates of the change points. Denote as $\mathcal{T}_{\pi}$ the stopping time of $\pi$.
Define $\chi$ as the supremum of the $(1-\delta)$-quantile of the budget over the family of probability distributions $\mathbb{P}_{\pi,J}$ for $J\in\prod_{i=1}^m\{0,\dots,\alpha_i-1\}$. Namely,
\begin{equation}\label{def:chi_localize}
\chi=\inf\{t>0: \forall J\in\prod_{i=1}^m\{0,\dots,\alpha_i-1\}, \mathbb{P}_{\pi,J}(\mathcal{T}_{\pi}
\geq t)\leq \delta\}\enspace .
\end{equation}

Let $i\in\{1, \dots,m\}$. Let $J$ be a vector of $m$ indices, whose $i$-th component is $j_i$.
 Define the event
\begin{equation*}
\xi_{i,j_i}=\{\mathcal{T}_{\pi}\leq \chi,\; |\hat x_i - x_i^{J}|\leq \eta\}\enspace .
\end{equation*}

Recall that the environment $\nu_{J}$ contains $m$ change points, and the $i$-th change point is $x_i^{J}$. Then, it follows from the $(\delta,\eta)$-correctness of $\pi$ that
\[
\mathbb{P}_{\pi,J}(|\hat x_i - x_i^{J}|> \eta) \leq \delta\enspace.
\]
Moreover, by definition of $\chi$ (see~\eqref{def:chi_localize}), it holds that
\[
\mathbb{P}_{\pi,J}(\mathcal{T}_{\pi} \geq \chi)\leq \delta\enspace.
\]

Therefore, by union bound, we have that, for all $J$ such that $j_i$ is the $i$-th component of $J$,
\begin{equation}\label{eq:TV_bound_reference}
\mathbb{P}_{\pi,J}(\xi_{i,j_i})\geq 1-2\delta \enspace.
\end{equation}

Consider now the environment $\nu_{J^{(-i)}}$, which differs from $\nu_J$ by the fact that the jump at change point $i$ is set to zero. Under this environment, the change point $i$ is not identifiable, and the algorithm should not be able to identify it in finite time. In particular, we have the following bound, stated as Lemma~\ref{lem:TV_bound_alternative}, whose proof is deferred at the end of this section, and which hold as long as $\eta<1/8$.

\begin{equation}\label{eq:TV_bound_alternative}
	\text{For all } J^{(-i)}, \text{ it holds that } \quad \mathbb{P}_{\pi,J^{(-i)}}(\mathcal{T}_{\pi}\leq\chi)\leq 2\delta\enspace.
\end{equation}

Let $j_i\in \{0,\dots \alpha_i-1\}$ and recall that the position of the change point $i$ in the environment $\nu_J$ is $x_i^{J}=x_i^*+j_i\cdot 2\eta' $. In particular, for any $\tilde{\jmath}_i\in \{0,\dots,\alpha_i-1\}$, with $\tilde{\jmath}_i\neq j_i$, if $\tilde J$ is such that its $i$-th component is $\tilde{\jmath}_i$, then, one has $|x_i^{J}-x_i^{\tilde J}|\geq 2\eta'>2\eta$. Therefore, the events $\xi_{i,j_i}$ and $\xi_{i,\tilde{\jmath}_i}$ are disjoint. Moreover, observe that the environment $\nu_{J^{(-i)}}$ does not depend on the index $j_i$, as the jump at change point $i$ is set to zero.

Therefore, we can write, for a fixed $i$ and for a fixed vector of indices $J^{(-i)}$,
\begin{align*}
	\frac{1}{\alpha_i}\sum_{j_i=0}^{\alpha_i-1} \mathbb{P}_{\pi,J^{(-i)}}(\xi_{i,j_i}) & = \frac{1}{\alpha_i}\mathbb{E}_{\pi,J^{(-i)}}\left[\sum_{j_i=0}^{\alpha_i-1} \mathbbm{1}_{\xi_{i,j_i}}\right] = \frac{1}{\alpha_i}\mathbb{P}_{\pi,J^{(-i)}}\left[{\displaystyle\bigsqcup_{j_i=0}^{\alpha_i-1} \xi_{i,j_i}}\right] \\
	& \leq \frac{1}{\alpha_i} \mathbb{P}_{\pi,J^{(-i)}}(\mathcal{T}_{\pi}\leq \chi) \\
	& \leq \frac{2\delta}{\alpha_i}\enspace,
\end{align*}
where we use that, $\forall j_i$, $\xi_{i,j_i}\subset \{\mathcal{T}_{\pi}\leq \chi\}$, and the last inequality follows from Equation~\eqref{eq:TV_bound_alternative}.

On the other hand, it follows from Equation~\eqref{eq:TV_bound_reference}, averaging over $j_i$, that,
\begin{equation*}
\frac{1}{\alpha_i}\sum_{j_i=0}^{\alpha_i-1} \mathbb{P}_{\pi,J}(\xi_{i,j_i})\geq 1-2\delta\enspace,
\end{equation*}
where we point out that the distribution $\mathbb{P}_{\pi,J}$ on the left sum depends on the index $j_i$.

From these two bounds, using monotonicity of $\kl(\cdot,\cdot)$ in each argument (valid as soon as  $2\delta/\alpha_i \leq 1-2\delta$, which holds for $\delta\leq 1/4$ and $\alpha_i\geq 1$), and joint convexity of the $\kl$ divergence (e.g., \citealp[Section~4.1]{polyanskiy2014lecture}), we write
\begin{align}
\kl\left(1-2\delta,\frac{2\delta}{\alpha_i}\right) &  \leq  \kl\left(\frac{1}{\alpha_i}\sum_{j_i=0}^{\alpha_i-1}\mathbb{P}_{\pi,J}(\xi_{i,j_i}),\frac{1}{\alpha_i}\sum_{j_i=0}^{\alpha_i-1}\mathbb{P}_{\pi,J^{(-i)}}(\xi_{i,j_i})\right) \nonumber \\
& \leq  \frac{1}{\alpha_i}\sum_{j_i=0}^{\alpha_i-1}\kl\left(\mathbb{P}_{\pi,J}(\xi_{i,j_i}),\mathbb{P}_{\pi,J^{(-i)}}(\xi_{i,j_i})\right) \label{eq:kl_bound}
\enspace.
\end{align}

Observe that the bound~\eqref{eq:kl_bound} holds for any components $J^{(-i)}={(j_k)}_{k\ne i}$. In particular, we can average over $J^{(-i)}$ to write
\begin{align*}
\kl\left(1-2\delta,\frac{2\delta}{\alpha_i}\right) & \leq \frac{1}{\alpha_i}\sum_{j_i=0}^{\alpha_i-1} \frac{1}{\prod_{k\ne i} \alpha_k}\sum_{J^{(-i)}\in \prod_{k\ne i} \{0,\dots,\alpha_k-1\}} \kl\left(\mathbb{P}_{\pi,J}(\xi_{i,j_i}),\mathbb{P}_{\pi,J^{(-i)}}(\xi_{i,j_i})\right) \\
& = \frac{1}{\prod_{k=1}^m \alpha_k}\sum_{J\in \prod_{k=1}^m \{0,\dots,\alpha_k-1\}} \kl\left(\mathbb{P}_{\pi,J}(\xi_{i,j_i}),\mathbb{P}_{\pi,J^{(-i)}}(\xi_{i,j_i})\right) \\
& = \mathbb{E}_{J\sim\mathcal{U}}\left[\kl\left(\mathbb{P}_{\pi,J}(\xi_{i,j_i}),\mathbb{P}_{\pi,J^{(-i)}}(\xi_{i,j_i})\right)\right] \enspace,
\end{align*}
where the last equality is the notation for the expectation over $J$ when $J$ is uniformly distributed in $\prod_{k=1}^m \{0,\dots,\alpha_k-1\}$.

Now, $\xi_{i,j_i}$ is measurable with respect to the observations of $\pi$ up to time $\chi$, so that the probability of $\xi_{i,j_i}$ under $\mathbb{P}_{\pi,J}$ and $\mathbb{P}_{\pi,J^{(-i)}}$ only depends on the distribution of the observations up to time $\chi$. Consider then the procedure $\tilde\pi$ defined as follows: for $t=1,\dots,\chi$, $\tilde\pi$ behaves like $\pi$. If at time $\chi$, the algorithm $\pi$ did not stop, then $\tilde\pi$ stops anyway and returns an error. For any environment, the distribution of the observations of $\tilde\pi$ up to time $\chi$ is the same as the distribution of the observations of $\pi$ up to time $\chi$. Moreover, the algorithm $\tilde\pi$ stops at time $\chi$ at the latest, so that the budget of $\tilde\pi$ is at most $\chi$. Therefore, we can rewrite the above inequality, and use the data-processing inequality for the KL divergence, to write
\begin{equation}
\kl\left(1-2\delta,\frac{2\delta}{\alpha_i}\right) \leq \mathbb{E}_{J\sim\mathcal{U}}\left[\KL\left(\mathbb{P}_{\tilde\pi,J},\mathbb{P}_{\tilde\pi,J^{(-i)}}\right)\right] \enspace. \label{eq:main_KL_bound}
\end{equation}

\paragraph{Bound on the $\KL$ divergence.}
Fix $i\in\{1,\dots,m\}$ and $J\in \prod_{k=1}^{m} \{0,\dots,\alpha_k-1\}$.
We want to bound the KL divergence $\KL\left(\mathbb{P}_{\tilde\pi,J},\mathbb{P}_{\tilde\pi,J^{(-i)}}\right)$.

Now, observe that the environments $\nu_J$ and $\nu_{J^{(-i)}}$ only differ by a constant magnitude $\Delta_i$ on the interval $[x_i^{J},x_{i+1}^{J})$, and are identical elsewhere. Denote as $f_J$ for the reward function characterizing $\nu_{J}$, and $f_{J^{(-i)}}$ for the reward function characterizing $\nu_{J^{(-i)}}$. One has
\[
f_J-f_{J^{(-i)}}=\Delta_i \mathbbm{1}\{x\in [x_i^{J},x_{i+1}^{J})\}\enspace.
\]
Recall that the noise is assumed to be Gaussian with variance $1$.

 Therefore, we can use the decomposition of the KL divergence for bandit problems, under a continuous action space, to write
\begin{align*}
\KL\left(\mathbb{P}_{\tilde\pi,J},\mathbb{P}_{\tilde\pi,J^{(-i)}}\right) & = \mathbb{E}_{\tilde\pi,J}\left[\sum_{t=1}^{\mathcal{T}_{\tilde\pi}} \KL\left(\mathcal{N}(f_J(x_t),1),\mathcal{N}(f_{J^{(-i)}}(x_t),1)\right)\right] \\
& = \mathbb{E}_{\tilde\pi,J}\left[\sum_{t=1}^{\mathcal{T}_{\tilde\pi}} \frac{{(f_J(x_t)-f_{J^{(-i)}}(x_t))}^2}{2}\right] \\
& = \frac{\Delta_i^2}{2}\mathbb{E}_{\tilde\pi,J}\left[\sum_{t=1}^{\mathcal{T}_{\tilde\pi}} \mathbbm{1}\{x_t\in [x_i^{J},x_{i+1}^{J})\}\right] \enspace.
\end{align*}

Now, averaging over $J$ uniformly distributed in $\prod_{k=1}^m \{0,\dots,\alpha_k-1\}$, from Equation~\eqref{eq:main_KL_bound}, we obtain
\begin{align*}
	\frac{2}{\Delta_i^2}\kl\left(1-2\delta,\frac{2\delta}{\alpha_i}\right) & \leq \mathbb{E}_{J\sim\mathcal{U}}\left[\mathbb{E}_{\tilde\pi,J}\left[\sum_{t=1}^{\mathcal{T}_{\tilde\pi}} \mathbbm{1}\{x_t\in [x_i^{J},x_{i+1}^{J})\}\right]\right] \enspace.
\end{align*}
Observe that the averaging measure $\mathbb{E}_{J\sim\mathcal{U}}\mathbb{E}_{\tilde\pi,J}$ is the same for each index $i$. Then, we sum over $i=1,\dots,m$
\begin{align*}
\sum_{i=1}^m \frac{2}{\Delta_i^2}\kl\left(1-2\delta,\frac{2\delta}{\alpha_i}\right) & \leq \sum_{i=1}^m \mathbb{E}_{J\sim\mathcal{U}}\left[\mathbb{E}_{\tilde\pi,J}\left[\sum_{t=1}^{\mathcal{T}_{\tilde\pi}} \mathbbm{1}\{x_t\in [x_i^{J},x_{i+1}^{J})\}\right]\right] \\
& = \mathbb{E}_{J\sim\mathcal{U}}\left[\mathbb{E}_{\tilde\pi,J}\left[\sum_{t=1}^{\mathcal{T}_{\tilde\pi}} \sum_{i=1}^m \mathbbm{1}\{x_t\in [x_i^{J},x_{i+1}^{J})\}\right]\right] \\
& \leq \mathbb{E}_{J\sim\mathcal{U}}\left[\mathbb{E}_{\tilde\pi,J}\left[\mathcal{T}_{\tilde\pi}\right]\right] \\
& \leq \chi\enspace,
\end{align*}
where the second inequality follows from the fact that the intervals $[x_i^{J},x_{i+1}^{J})$ are disjoint, and the last inequality follows from the fact that, under $\tilde \pi$, the budget $\mathcal{T}_{\tilde\pi}$ is at most $\chi$.

By definition of $\chi$ as an infimum, this means that there exists $J\in\prod_{i=1}^m\{0,\dots,\alpha_i-1\}$ such that
\begin{align*}
\mathbb{P}_{\pi,J}\left(\mathcal{T}_{\pi}\geq 2\sum_{i=1}^m \Delta_{i}^{-2} \kl\left(1-2\delta,\frac{2\delta}{\alpha_i}\right) \right)\geq \delta  \enspace,
\end{align*}
in particular, there exists an environment $\nu'=\nu_J$ such that the above holds. Moreover, by construction, the jumps of $\nu'$ are the same as the jumps of $\nu$, and for each $i\in\{1,\dots,m-1\}$, $(1/4)\vartheta_i\leq\vartheta'_i\leq (3/4)\vartheta_i$, as verified in Lemma~\ref{lem:complexity_nu_J}.

Finally, use  the bound $ \kl\left(1-2\delta,\frac{2\delta}{\alpha_i}\right) \geq \frac{1}{2}\log\left(\frac{\alpha_i}{8\delta}\right)$ from technical Lemma~\ref{lem:kl_technical},	and with the bound $\log(\alpha_i) \geq \log_+\!\left(\frac{s_i}{16\eta'}\right)$, which follows from the definition of $\alpha_i$. As $\eta'>\eta$ is arbitrary, we can conclude by taking the limit $\eta'\to\eta$.
\end{proof}

\begin{lemma}\label{lem:complexity_nu_J}
For all $J\in\prod_{i=1}^m \{0,\dots,\alpha_i-1\}$, it holds that $\nu_J$ is a valid environment with $m$ change points, with the same jumps at $\nu$, and with spacing  $(1/4)\vartheta_i \leq \vartheta_i^J \leq (3/4)\vartheta_i$ for each $i=1,\dots,m-1$.
In particular, $H_{\mathrm{localize}}^{\nu_J}=H_{\mathrm{localize}}^{\nu}$, and $\frac{4}{3} H_{\mathrm{detect}}^{\nu} \leq H_{\mathrm{detect}}^{\nu_J}\leq 4H_{\mathrm{detect}}^{\nu}$.
\end{lemma}

\begin{proof}

First, we verify that the environment $\nu_J$ is valid and contains $m$ change points. For validity, we only need the positions of the change points $(x_i^J)_{i=1}^m$ to be distinct and to belong to $(0,1)$.

Consider $x_m^J=\frac{1}{2}x_m^*+j_m\cdot 2\eta'$. One has $j_m\cdot 2\eta'\leq \alpha_m \cdot 2\eta' \leq \frac{s_m}{8\eta'} \cdot 2\eta'\leq 1/4$, and $x_m^*\leq 1$, so that  $x_m^J<1$.

Moreover, for each $i=1,\dots,m-1$, the change points $x_i^{J}$ and $x_{i+1}^{J}$ are spaced by at least $\vartheta_i/4$ and at most $(3/4)\vartheta_i$. Indeed, one has $x_{i+1}^{J}-x_i^{J}=\frac{1}{2}(x_{i+1}^*-x_i^*)+(j_{i+1}-j_i)\cdot 2\eta'=\frac{1}{2}\vartheta_i+(j_{i+1}-j_i)\cdot 2\eta'$. We will show that $|j_{i+1}-j_i|\cdot 2\eta' \leq \frac{\vartheta_i}{4}$, which will conclude the proof.

By assumption on $J$, we have that $j_{i+1}-j_i\in\{-(\alpha_i-1),\dots,\alpha_{i+1}-1\}$, so that $|j_{i+1}-j_i|\cdot 2\eta' \leq((\alpha_i-1) \vee (\alpha_{i+1}-1)) \cdot 2\eta'$. Now, by definition of $\alpha_i$ and $\alpha_{i+1}$, we have $\alpha_i \vee \alpha_{i+1}-1 = 1 \vee \left\lfloor\frac{\vartheta_{i-1}\wedge \vartheta_i}{8\eta'}\right\rfloor \vee \left\lfloor\frac{\vartheta_{i}\wedge \vartheta_{i+1}}{8\eta'}\right\rfloor -1\leq 0 \vee \left(\frac{\vartheta_{i}}{8\eta'}-1\right)$. Therefore, $|j_{i+1}-j_i|\cdot 2\eta' \leq 0\vee \left(\frac{\vartheta_i}{4}-2\eta'\right)\leq \frac{\vartheta_i}{4}$, which conclude the proof.

\end{proof}

\begin{lemma}\label{lem:TV_bound_alternative}

\begin{equation*}
	\text{For all } J^{(-i)}, \text{ it holds that } \quad \mathbb{P}_{\pi,J^{(-i)}}(\mathcal{T}_{\pi}\leq\chi)\leq 2\delta\enspace.
\end{equation*}
\end{lemma}

\begin{proof}[Proof of Lemma~\ref{lem:TV_bound_alternative}]

We can prove it by considering two environments, with one vanishing change point of jump magnitude $\epsilon$. Denote as $\nu_{\epsilon}^{(1)}$ the environment obtained from $\nu_{J^{(-i)}}$ by adding a signal of magnitude $\epsilon$ on the interval $[x^*_{m}+l',1)$, and as $\nu_{\epsilon}^{(2)}$ the environment obtained from $\nu_{J^{(-i)}}$ by adding a signal of magnitude $\epsilon$ on the interval $[1-l',1)$. Here, we take $0<l'<1/8-\eta$, which is possible by assumption on $\eta<1/8$. These choice imply that the $m$-th change point under $\nu_{\epsilon}^{(1)}$, and the $m$-th change point under $\nu_{\epsilon}^{(2)}$ are spaced by at least $2\eta$.
Indeed $x^*_m+l'+\eta<1-l'-\eta$, which hold for $x^*_m\leq 3/4 \leq 1-l'-2\eta$.   Note that $\nu_{\epsilon}^{(1)}$ and $\nu_{\epsilon}^{(2)}$ have $m$ valid change points.

Consider the decomposition of the event $\{\mathcal{T}_{\pi}\leq\chi\}$ as
\begin{equation}\label{eq:decomposition_xii}
	\{\mathcal{T}_{\pi}\leq\chi\} = \{\mathcal{T}_{\pi}\leq	\chi,\; \hat x_m \leq \frac{1}{2}(x^*_{m}+1)\} \cup \{\mathcal{T}_{\pi}\leq \chi,\; \hat x_m > \frac{1}{2}(x^*_{m}+1)\}\enspace .
\end{equation}

Under $\nu_{\epsilon}^{(1)}$, the change point $m$ is at position $x^*_m+l'<\frac{1}{2}(x^*_{m}+1)-\eta$. Then, by the $(\delta,\eta)$-correctness of $\pi$, it holds that
\begin{equation}\label{eq:bounds_nu-}
\mathbb{P}_{\pi,\nu_{\epsilon}^{(1)}}\left(\hat x_m > \frac{1}{2}(x^*_{m}+1)\right)\leq \mathbb{P}_{\pi,\nu_{\epsilon}^{(1)}}(|\hat x_m- (x^*_m+l')|>\eta)   \leq \delta\enspace.
\end{equation}
Similarly, under $\nu_{\epsilon}^{(2)}$, the change point $m$ is at position $1-l'>\frac{1}{2}(x^*_{m}+1)+\eta$, so that
\begin{equation}\label{eq:bounds_nu+}
\mathbb{P}_{\pi,\nu_{\epsilon}^{(2)}}\left(\hat x_m \leq \frac{1}{2}(x^*_{m}+1)\right) \leq \mathbb{P}_{\pi,\nu_{\epsilon}^{(2)}}(|\hat x_m- (1-l')|>\eta)   \leq \delta\enspace.
\end{equation}

Moreover, when $\epsilon$ goes to zero, the total variation distance between $\mathbb{P}_{\pi,\nu_{\epsilon}^{(1)}}$ and $\mathbb{P}_{\pi,J^{(-i)}}$, goes to zero when we restrict to events measurable with respect to the observations up to the finite time $\chi$. Consider any event $A$ measurable with respect to the observations up to time $\chi$. Then, we can write, by decomposition of the KL divergence,
\begin{align*}
\KL(\mathbb{P}_{\pi,\nu_{\epsilon}^{(1)}}(A),\mathbb{P}_{\pi,J^{(-i)}}(A)) & \leq \mathbb{E}_{\pi,\nu_{\epsilon}^{(1)}}\left[\sum_{t=1}^{\chi} \KL(\mathcal{N}(f_{\epsilon}^{(1)}(x_t),1),\mathcal{N}(f_{J^{(-i)}}(x_t),1))\right] \\
& =\frac{\epsilon^2}{2}\mathbb{E}_{\pi,\nu_{\epsilon}^{(1)}}\left[\sum_{t=1}^{\chi} \mathbbm{1}\{x_t\in [x^*_{m}+l',1-l'\}\right] \\
& \leq \frac{\epsilon^2}{2}\chi \enspace,
\end{align*}
Now, from the Pinsker's inequality, we have that
\begin{equation*}
|\mathbb{P}_{\pi,\nu_{\epsilon}^{(1)}}(A)-\mathbb{P}_{\pi,J^{(-i)}}(A)| \leq \sqrt{\frac{1}{2}\KL(\mathbb{P}_{\pi,\nu_{\epsilon}^{(1)}}(A),\mathbb{P}_{\pi,J^{(-i)}}(A))} \leq \epsilon\sqrt{\chi}/2\enspace,
\end{equation*}
which goes to zero as $\epsilon$ goes to zero. We proceed similarly for $\nu_{\epsilon}^{(2)}$.

Now, using the decomposition of the event $\{\mathcal{T}_{\pi}\leq\chi\}$ given by Equation~\eqref{eq:decomposition_xii}, we can write
\begin{align*}
\mathbb{P}_{\pi,J^{(-i)}}(\mathcal{T}_{\pi}\leq	\chi) & = \mathbb{P}_{\pi,J^{(-i)}}(\mathcal{T}_{\pi}\leq	\chi,\; \hat x_m \leq \tfrac{1}{2}(x^*_{m}+1)) +\mathbb{P}_{\pi,J^{(-i)}}(\mathcal{T}_{\pi}\leq \chi,\; \hat x_m > \tfrac{1}{2}(x^*_{m}+1)) \\
& = \lim_{\epsilon\to 0}\mathbb{P}_{\pi,\nu_{\epsilon}^{(2)}}(\mathcal{T}_{\pi}\leq\chi,\; \hat x_m \leq \tfrac{1}{2}(x^*_{m}+1)) \\ &\qquad+ \lim_{\epsilon\to 0}\mathbb{P}_{\pi,\nu_{\epsilon}^{(1)}}(\mathcal{T}_{\pi}\leq\chi,\; \hat x_m > \tfrac{1}{2}(x^*_{m}+1))\\
& \leq \lim_{\epsilon\to 0}\mathbb{P}_{\pi,\nu_{\epsilon}^{(2)}}(\hat x_m \leq \tfrac{1}{2}(x^*_{m}+1)) + \lim_{\epsilon\to 0}\mathbb{P}_{\pi,\nu_{\epsilon}^{(1)}}(\hat x_m > \tfrac{1}{2}(x^*_{m}+1))\\
& \leq 2\delta\enspace,
\end{align*}
where the last inequality follows from the bounds~\eqref{eq:bounds_nu-},\eqref{eq:bounds_nu+} we obtained under $\nu_{\epsilon}^{(1)}$ and $\nu_{\epsilon}^{(2)}$. This proves Equation~\eqref{eq:TV_bound_alternative}.
\end{proof}

\begin{lemma}\label{lem:kl_technical}
	For $\delta\in(0,1/4)$ and $\alpha\geq 1$, one has
	\begin{equation*}
	\kl\left(1-2\delta,\frac{2\delta}{\alpha}\right) \geq \frac{1}{2}\log\left(\frac{\alpha}{8\delta}\right) \enspace.
	\end{equation*}
\end{lemma}

\begin{proof}[Proof of Lemma~\ref{lem:kl_technical}]
	We can write
	\begin{align*}
	\kl\left(1-2\delta,\frac{2\delta}{\alpha}\right)
	& \geq \kl\left(1/2,\frac{2\delta}{\alpha}\right)  \\
	& = \frac{1}{2}\log\left(\frac{1/2}{2\delta/\alpha}\right) + \frac{1}{2}\log\left(\frac{1/2}{1-2\delta/\alpha}\right) \\
	& = \frac{1}{2}\log\left(\frac{\alpha}{8\delta}\right) + \frac{1}{2}\log\left(\frac{1}{(1-2\delta/\alpha)}\right) \\
	& \geq \frac{1}{2}\log\left(\frac{\alpha}{8\delta}\right)\enspace,
	\end{align*}
	where the first inequality follows from the fact that $1-2\delta \geq 1/2 \geq 2\delta/\alpha$, and the last inequality follows from the fact that $1-2\delta/\alpha\leq 1$.
\end{proof}

\subsection{Proof of Lemma~\ref{lem:LB_H_detect}}

\paragraph{Proof roadmap.}
The proof proceeds in four steps. The overall novelty lies in a two-level reduction: from bandit detection of multiple change points to an instance with two change points, then from the $\delta$-PAC bandit setting with quantile guarantees to a batch detection problem amenable to classical information-theoretic tools.

First, we construct a family of environments ${(\tilde\nu^{(s)})}_{s\in[0,1/4]}$ from $\nu$, and show that $H_{\mathrm{detect}}$ and $H_{\mathrm{localize}}$ are preserved up to absolute constants (Lemma~\ref{lem:complexity_nu_s}).
Second, we introduce an alternative environment $\bar\nu$ (with two removed jumps).
Third, we consider the event $\xi=\{\mathcal{T}_{\pi}\leq\chi\}$, where $\chi$ is the maximum $(1-\delta)$-quantile of the budget under all the environments $\tilde\nu^{(s)}$. This event has a probability at least $\delta$ under $\nu'=\tilde\nu^{(s)}$ for all $s\in[0,1/4]$. We show in Lemma~\ref{lem:TV_bound} that it has a probability smaller than $2\delta$ under $\bar\nu$, because $\bar\nu$ contains less than $m=N$ change points, so that the algorithm should not stop.
Fourth, we convert this total-variation lower bound into a KL lower bound using the Bretagnolle--Huber inequality.
Finally, we upper-bound the averaged KL by the expected number of pulls in the interval where $\bar\nu$ and $\tilde\nu^{(s)}$ differ, which yields a lower bound on $\chi$; the conclusion follows by selecting $\nu'=\tilde\nu^{(s)}$ for a suitable $s$.

\begin{proof}[Proof of Lemma~\ref{lem:LB_H_detect} ]
\textbf{Reference class of environments.}\quad
Consider an environment $\nu$ with $m\geq 2$ change points, with parameters $\mu_0$, $\Delta_1,\dots,\Delta_m$, $\vartheta_0,\dots,\vartheta_m$.

Observe that,
\begin{equation*}
H_{\mathrm{detect}}^{\nu}:=H_{\mathrm{detect}}^{(m,\nu)}=\displaystyle \max_{i=1,\dots,m} \frac{1}{\mathcal{E}^2_i} =  \max_{i=1,\dots,m-1} \frac{1}{\vartheta_i}\frac{1}{\Delta_i^2\wedge \Delta_{i+1}^2} \enspace.
\end{equation*}

The quantity  $\frac{1}{\vartheta_i}\frac{1}{\Delta_i^2\wedge \Delta_{i+1}^2}$ measures the difficulty of detecting a point in the interval between  $x_i^*$ and $x_{i+1}^*$.
Consider $i^*=\arg\max_{i=1,\dots,m-1} \frac{1}{\vartheta_i}\frac{1}{\Delta_i^2\wedge \Delta_{i+1}^2}$, and assume without loss of generality\footnote{The other case follows from a symmetric construction} that $|\Delta_{i^*}|\leq |\Delta_{i^*+1}|$, so that $H_{\mathrm{detect}}^{\nu}=\frac{1}{\vartheta_{i^*}}\frac{1}{\Delta_{i^*}^2}$.

From $\nu$, we construct a family of environments $\tilde\nu^{(s)}$ parametrized by some $s\in[0,1/4]$. Define $\tilde \Delta=\sqrt{2}(|\Delta_{i^*}|\wedge |\Delta_{i^*+1}|)=\sqrt{2}|\Delta_{i^*}|$ and $\tilde\vartheta=\vartheta_{i^*}/2$.

Let $s\in[0,1/4]$, define the environment $\tilde\nu^{(s)}$ with parameters $\tilde\Delta_1,\dots,\tilde\Delta_m$, the initial mean $\mu_0$ and the length of the change points denoted as $\tilde\vartheta^{(s)}_0,\dots,\tilde\vartheta^{(s)}_m$, where
\begin{itemize}
	\item the initial mean is $\mu_0$ and initial position is $\tilde x_1^*=x_1^*/2$;
	\item for $i\in\{1,\dots,m\}\setminus\{i^*,i^*+1\}$, $\tilde\Delta_i=\Delta_i$;
	\item $\tilde\Delta_{i^*}=-\tilde\Delta_{i^*+1}=\tilde\Delta$;
	\item for $i\in\{1,\dots,m-1\}\setminus\{i^*-1,i^*,i^*+1\}$, $\tilde\vartheta^{(s)}_i=\vartheta_i/2$;
	\item  $\tilde\vartheta^{(s)}_{i^*}=\vartheta_{i^*}/2$; $\tilde\vartheta^{(s)}_{i^*-1}=\vartheta_{i^*-1}/2+s$ and $\tilde\vartheta^{(s)}_{i^*+1}=\vartheta_{i^*+1}/2+\tfrac{1}{2}-s$
\end{itemize}

For notational purposes, denote as $\tilde l=\sum_{j=0}^{i^*-1}\tilde\vartheta^{(0)}_j$ for the position of the change point $i^*$ under $\tilde\nu^{(0)}$. Observe that by definition, $\tilde l=\sum_{j=0}^{i^*-1}\vartheta_j/2=x^*_{i^*}/2$. Recall that $\tilde{\vartheta}=\vartheta_{i^*}/2$. Define also $\tilde{r}= \tilde l+\tilde{\vartheta}+1/4$ for the position of the change point $i^*+1$ under $\tilde\nu^{(1/4)}$. Observe that $\tilde r=\sum_{j=0}^{i^*} \tilde{\vartheta}^{(1/4)}=\sum_{j=0}^{i^*}\vartheta_j/2+1/4=x^*_{i^*+1}/2+1/4$. Moreover, one has $|\tilde{r}-\tilde{l}| = \tilde{\vartheta}+1/4$. We denote as  $\mathbb{P}_{\pi,s}$ the probability distribution induced by the interaction between algorithm $\pi$ and  the environment $\tilde\nu^{(s)}$. We will also denote by $\mathbb{E}_{\pi,s}$ the corresponding expectation.

\begin{remark}
	The family of environments ${(\tilde\nu^{(s)})}_s$  is designed to preserve, up to a multiplicative constant $
	2$ the complexities $H_{\mathrm{detect}}$ and $H_{\mathrm{localize}}$, that is  $H_{\mathrm{detect}}^{\tilde\nu^{(s)}}$ and $H_{\mathrm{localize}}^{\tilde\nu^{(s)}}$ are of the same order as  $H_{\mathrm{detect}}^{\nu}$ and $H_{\mathrm{localize}}^{\nu}$, respectively. This is proved in Lemma~\ref{lem:complexity_nu_s} and is a direct consequence of the construction of $\tilde\nu^{(s)}$ and the definition of $H_{\mathrm{detect}}$ and $H_{\mathrm{localize}}$.
\end{remark}

\paragraph{Alternative environment.}
Define $\bar\nu$ as the environment obtained from  $\tilde\nu^{(0)}$ by setting the jumps $\tilde\Delta_{i^*}$ and $\tilde\Delta_{i^*+1}$ to zero, so that $\bar\nu$ has $m-2$ change points.
Note that $\bar\nu$ does not depend on the parameter $s$.
Interestingly, the environment $\tilde{\nu}^{(s)}$ and $\bar\nu$ only differ for arms in the interval $[\tilde l, \tilde r]$.

Denote as $\mathbb{\bar P}_{\pi}$ the probability distribution of the observations when the environment is $\bar\nu$ and by $\mathbb{\bar E}_{\pi}$ the associated expectation, with some algorithm $\pi$.

\begin{remark}
	This construction is a reduction scheme. We reduce the problem of detecting $m$ change points to the signal detection problem of detecting the presence of a signal of magnitude $\tilde \Delta$, on an interval of length $\tilde \vartheta$, planted somewhere in the interval $[\tilde l, \tilde r]$.
\end{remark}

\paragraph{Bound in total variation.}
Let $\pi$ be an algorithm, which is $(\delta,\eta)$-correct for the change point detection problem with $m$ change points, and let $\hat x_1,\dots,\hat x_m$ be the corresponding estimates of the change points. Denote as $\mathcal{T}_{\pi}$ the stopping time of $\pi$.
Define $\chi$ as the supremum of the $(1-\delta)$-quantile of the budget over the family of probability distributions $\mathbb{P}_{\pi,s}$ for $s\in[0,1/4]$. Namely,
\begin{equation*}
	\chi=\inf\left\{t>0:\sup_{s\in[0,1/4]}\mathbb{P}_{\pi,s}(\mathcal{T}_{\pi}>t)\leq\delta\right\}\enspace .
\end{equation*}

Define the event
\begin{equation*}
	\xi=\{\mathcal{T}_{\pi}\leq \chi\}\enspace .
\end{equation*}

By definition of $\chi$, it holds that $\mathbb{P}_{\pi,s}(\xi^c)\leq\delta$ for all $s\in[0,1/4]$. Then, under $\bar{\nu}$, the algorithm is not able to identify $m$ change points in finite time, as $\bar{\nu}$ only contains $m-2$ change points. In particular, one can prove the following bound, proved in Lemma~\ref{lem:TV_bound} on the probability of the event $\xi$ under $\bar{\nu}$:
\begin{equation}
\mathbb{\bar P}_{\pi}(\xi)\leq 2\delta\enspace.
\end{equation}

This bound is a consequence of the fact that, under $\bar{\nu}$, the algorithm $\pi$ should typically not stop, as it cannot identify $m$ change points. In particular, the event $\xi$ is an event on which $\pi$ stops, and therefore should have a small probability under $\bar{\nu}$. The proof follows from proving that $\bar{\nu}$ is as close as possible to two different environments with $m$ change points, under which the position of the change points are different.

Now, observe that the event $\xi$ is measurable with respect to the first $\chi$ observations. Consider the modified algorithm $\tilde\pi$ defined as follows. For $t=1,\ldots,\chi$, $\tilde \pi$ uses the rules of $\pi$. If at time $\chi$, the algorithm $\pi$ did not stop, $\tilde \pi$ stops anyway and returns an error. Observe that $\xi$ is exactly the event on which $\tilde\pi$ does not return an error. In particular, the event $\xi$ is measurable with respect to the observations of $\tilde\pi$, and has the same probability under $\pi$ and $\tilde \pi$ (facing any environment). Therefore, $\mathbb{\bar P}_{\tilde\pi}(\xi)\leq 2\delta$ and $\mathbb{P}_{\tilde\pi,s}(\xi)\geq 1-\delta$ for all $s\in[0,1/4]$.

We can bound the total variation distance between $\mathbb{\bar P}_{\tilde\pi}$ and $\mathbb{P}_{\tilde\pi,s}$ for $s\in[0,1/4]$ as
\begin{align*}
\mathrm{TV}(\mathbb{\bar P}_{\tilde\pi},\mathbb{ P}_{\tilde\pi,s}) & \geq \mathbb{\bar P}_{\tilde\pi}(\xi^c) - \mathbb{P}_{\tilde\pi,s}(\xi^c)  \\
& \geq 1 - 2\delta - \delta = 1 - 3\delta\enspace .
\end{align*}

\paragraph{Data processing inequality.}
Classically, one can bound the total variation distance between $\mathbb{\bar P}_{\tilde\pi}$ and $\mathbb{P}_{\tilde\pi,s}$ by the Kullback--Leibler divergence between these two distributions, using the Bretagnolle--Huber inequality as
\begin{equation*}
\mathrm{TV}(\mathbb{\bar P}_{\tilde\pi},\mathbb{P}_{\tilde\pi,s}) \leq 1-\frac{1}{2}\exp\left(-\KL(\mathbb{\bar P}_{\tilde\pi},\mathbb{P}_{\tilde\pi,s})\right)\enspace .
\end{equation*}
Rearranging this inequality, and using the previous bound on the total variation distance, we obtain, for any $s\in[0,1/4]$,
\begin{equation}\label{eq:BH_1}
\KL(\mathbb{\bar P}_{\tilde\pi},\mathbb{P}_{\tilde\pi,s}) \geq \log\left(\frac{1}{2(1-\mathrm{TV}(\mathbb{\bar P}_{\tilde\pi},\mathbb{P}_{\tilde\pi,s}))}\right)\geq \log\left(\frac{1}{6\delta}\right)\enspace .
\end{equation}

\paragraph{Lower bound on the budget.}
Finally, we can bound the Kullback--Leibler divergence between $\mathbb{\bar P}_{\tilde\pi}$ and $\mathbb{P}_{\tilde\pi,s}$, using the decomposition of the KL divergence for bandit problems, under a continuous action space. Observe that the environments $\bar\nu$ and $\tilde\nu^{(s)}$ only differ by a constant magnitude $\tilde{\Delta}$ on the interval $[\tilde l+s,\tilde l+s+\tilde\vartheta)$, and are identical elsewhere. For notation, denote as $\bar{f}$ and $\tilde{f}^{(s)}$ the mean reward functions of $\bar\nu$ and $\tilde\nu^{(s)}$, respectively. Recall that the reward distributions are Gaussian with variance $1$. Denote also as $a_t$ the arm pulled at time $t$ by $\tilde\pi$. Then, the reward distribution at time $t$ under $\bar\nu$ is $\mathcal{N}(\bar f(a_t),1)$ and the reward distribution at time $t$ under $\tilde\nu^{(s)}$ is $\mathcal{N}(\tilde f^{(s)}(a_t),1)$.

Therefore, we can write
\begin{align*}
\KL(\mathbb{\bar P}_{\tilde\pi},\mathbb{P}_{\tilde\pi,s}) & = \mathbb{\bar E}_{\tilde\pi}\left[\sum_{t=1}^{\mathcal{T}_{\tilde\pi}} \KL\left(\mathcal{N}(\bar f(a_t),1),\mathcal{N}(\tilde f^{(s)}(a_t),1)\right)\right] \\
& = \mathbb{\bar E}_{\tilde\pi}\left[\sum_{t=1}^{\mathcal{T}_{\tilde\pi}} \frac{{(\tilde f^{(s)}(a_t)-\bar f(a_t))}^2}{2}\right] \\
& = \	\frac{\tilde{\Delta}^2}{2}\mathbb{\bar E}_{\tilde\pi}\left[\sum_{t=1}^{\mathcal{T}_{\tilde\pi}} \mathbbm{1}\{a_t\in [\tilde l+s,\tilde l+s+\tilde\vartheta)\}\right] \enspace .
\end{align*}

Now, averaging over $s$ uniformly distributed in $[0,1/4]$, we obtain
\begin{align*}
\frac{1}{1/4}\int_0^{1/4} \KL(\mathbb{\bar P}_{\tilde\pi},\mathbb{P}_{\tilde\pi,s}) ds & = \frac{1}{1/4}\frac{\tilde{\Delta}^2}{2}\mathbb{\bar E}_{\tilde\pi}\left[\sum_{t=1}^{\mathcal{T}_{\tilde\pi}} \int_0^{1/4} \mathbbm{1}\{a_t\in [\tilde l+s,\tilde l+s+\tilde\vartheta)\} ds\right] \\
& = 2\tilde{\Delta}^2\mathbb{\bar E}_{\tilde\pi}\left[\sum_{t=1}^{\mathcal{T}_{\tilde\pi}} \int_0^{1/4} \mathbbm{1}\{s\in [a_t-\tilde l - \tilde\vartheta, a_t-\tilde l)\} ds\right] \\
& \leq 2\tilde\vartheta\tilde{\Delta}^2\mathbb{\bar E}_{\tilde\pi}\left[\mathcal{T}_{\tilde\pi}  \right] \\
\enspace,
\end{align*}
where the last inequality follows from the fact that the length of the interval $[a_t-\tilde l - \tilde\vartheta, a_t-\tilde l)$ is at most $\tilde\vartheta$. Now, using the fact that, under $\tilde \pi$, the budget $\mathcal{T}_{\tilde\pi}$ is at most $\chi$, we can write
\begin{equation}
\log \left(\frac{1}{6\delta}\right) \leq \frac{1}{1/4}\int_0^{1/4} \KL(\mathbb{\bar P}_{\tilde\pi},\mathbb{P}_{\tilde\pi,s}) ds \leq 2\tilde\vartheta\tilde{\Delta}^2\chi \enspace ,
\end{equation}
where we used the bound on the KL divergence from Equation~\eqref{eq:BH_1} in the first inequality.

Rearranging yields
\begin{equation}
\chi \geq \frac{1}{2\tilde\vartheta\tilde{\Delta}^2}\log\left(\frac{1}{6\delta}\right)= \frac{1}{2\vartheta_{i^*}}\frac{1}{\Delta_{i^*}^2\wedge \Delta_{i^*+1}^2}\log\left(\frac{1}{6\delta}\right)= \frac{1}{2} H_{\mathrm{detect}}^{\nu}\log\left(\frac{1}{6\delta}\right) \enspace .
\end{equation}

By definition of $\chi$, this means that there exists $s\in[0,1/4]$ such that
\[
\mathbb{P}_{\pi,s}\left(\mathcal{T}_{\pi}\geq \frac{1}{2}H_{\mathrm{detect}}^{\nu}\log\left(\frac{1}{6\delta}\right) \right)\geq \delta \enspace .
\]

The proof of Lemma~\ref{lem:LB_H_detect} is concluded by observing that, by Lemma~\ref{lem:complexity_nu_s}, the environment $\tilde\nu^{(s)}$ has the desired properties.
It remains to prove Lemma~\ref{lem:TV_bound} and Lemma~\ref{lem:complexity_nu_s}.
\end{proof}

\begin{lemma}\label{lem:complexity_nu_s}
For all $s\in[0,1/4]$ it holds that $\tilde\nu^{(s)}$ is a valid environment with $m$ change points spaced by at least $2\eta$ and that
\[
 H_{\mathrm{detect}}^{\nu} \leq H_{\mathrm{detect}}^{\tilde\nu^{(s)}}\leq 2H_{\mathrm{detect}}^{\nu}\quad\text{and}\quad \frac{1}{2}H_{\mathrm{localize}}^{\nu}\leq H_{\mathrm{localize}}^{\tilde\nu^{(s)}}\leq H_{\mathrm{localize}}^{\nu}\enspace .
\]
\end{lemma}

\begin{proof}[Proof of the Lemma~\ref{lem:complexity_nu_s}]
It is clear that $\tilde\nu^{(s)}$ is a valid environment with $m$ change points, and that the change points are spaced by at least $2\eta$, because $\tilde\vartheta^{(s)}_i\geq \vartheta_i/2\geq 2\eta$ for all $i=1,\dots,m-1$.

By definition of $\tilde\nu^{(s)}$, it holds that $\tilde\Delta_{i^*}=-\tilde\Delta_{i^*+1}=\sqrt{2}\Delta_{i^*}$ and $\tilde\vartheta^{(s)}_{i^*}=\vartheta_{i^*}/2$. Then, one has
\begin{align*}
H_{\mathrm{localize}}^{\tilde\nu^{(s)}}=\sum_{i=1}^m \frac{1}{\tilde\Delta_i^2} = \sum_{\substack{i=1\\i\ne i^*,i^*+1}}^m \frac{1}{\Delta_i^2}+ \frac{2}{\tilde \Delta^2}
	 = \sum_{\substack{i=1\\i\ne i^*,i^*+1}}^m \frac{1}{\Delta_i^2}+ \frac{1}{\Delta_{i^*}^2} \enspace.
\end{align*}
Moreover, it holds that $\frac{1}{2}\left(\frac{1}{\Delta_{i^*}^2}+\frac{1}{\Delta_{i^*+1}^2}\right) \leq \frac{1}{\Delta_{i^*}^2}\leq \frac{1}{\Delta_{i^*}^2}+\frac{1}{\Delta_{i^*+1}^2}$, because $|\Delta_{i^*}|\leq |\Delta_{i^*+1}|$. Then, one has $\frac{1}{2}H_{\mathrm{localize}}^{\nu}\leq H_{\mathrm{localize}}^{\tilde\nu^{(s)}}\leq H_{\mathrm{localize}}^{\nu}$.

Recall that we assumed that $\Delta_{i^*}\leq \Delta_{i^*+1}$, so that
\begin{align*}
& H_{\mathrm{detect}}^{\nu}=\max_{i=1,\dots,m-1} \frac{1}{\vartheta_i}\frac{1}{\Delta_i^2\wedge \Delta_{i+1}^2}= \frac{1}{\vartheta_{i^*}}\frac{1}{\Delta_{i^*}^2} \\
& H_{\mathrm{detect}}^{\tilde\nu^{(s)}}=\displaystyle \max_{i=1,\dots,m-1}
\frac{1}{\tilde\vartheta^{(s)}_i}\frac{1}{\tilde\Delta_i^2\wedge \tilde\Delta_{i+1}^2} \enspace.
\end{align*}

We consider the different cases for $i=1,\dots,m-1$ to control $\frac{1}{\tilde\vartheta^{(s)}_i}\frac{1}{\tilde\Delta_i^2\wedge \tilde\Delta_{i+1}^2}$.

\noindent Assume that $i\ne i^*-1,i^*,i^*+1$, so that $\tilde\vartheta^{(s)}_i=\vartheta_i/2$ and $\tilde\Delta_i=\Delta_{i}$ and $\tilde\Delta_{i+1}=\Delta_{i+1}$. Then, one has
\[ \frac{1}{\tilde\vartheta^{(s)}_i}\frac{1}{\tilde\Delta_i^2\wedge \tilde\Delta_{i+1}^2}=\frac{2}{\vartheta_i}\frac{1}{\Delta_i^2\wedge \Delta_{i+1}^2}\leq 2 H_{\mathrm{detect}}^{\nu} \enspace.
\]

\noindent If $i=i^*-1$, one has $\tilde\vartheta^{(s)}_{i^*-1}=\vartheta_{i^*-1}/2+s$, $\tilde\Delta_{i^*-1}=\Delta_{i^*-1}$ and $\tilde\Delta_{i^*}=\sqrt{2}\Delta_{i^*}$. Then, one has
\[
\frac{1}{\tilde\vartheta^{(s)}_{i^*-1}}\frac{1}{\tilde\Delta_{i^*-1}^2\wedge \tilde\Delta_{i^*}^2}= \frac{1}{(\vartheta_{i^*-1}/2+s)} \frac{1}{\Delta_{i^*-1}^2\wedge {(\sqrt{2}\Delta_{i^*})}^2} \leq \frac{2}{\vartheta_{i^*-1}}\frac{1}{\Delta_{i^*-1}^2\wedge \Delta_{i^*}^2}\leq 2 H_{\mathrm{detect}}^{\nu} \enspace,
\]
as $s\geq 0$ and $\sqrt{2}|\Delta_{i^*}|\geq |\Delta_{i^*}|$.

\noindent If $i=i^*$, one has $\tilde\vartheta^{(s)}_{i^*}=\vartheta_{i^*}/2$, $\tilde\Delta_{i^*}=\sqrt{2}\Delta_{i^*}$ and $\tilde\Delta_{i^*+1}=-\sqrt{2}\Delta_{i^*}$. Then,
\[
\frac{1}{\tilde\vartheta^{(s)}_{i^*}}\frac{1}{\tilde\Delta_{i^*}^2\wedge \tilde\Delta_{i^*+1}^2}= \frac{1}{(\vartheta_{i^*}/2)} \frac{1}{{(\sqrt{2}\Delta_{i^*})}^2}= \frac{1}{\vartheta_{i^*}}\frac{1}{\Delta_{i^*}^2}=H_{\mathrm{detect}}^{\nu} \enspace.
\]
\noindent Finally, if $i=i^*+1$, one has $\tilde\vartheta^{(s)}_{i^*+1}=\vartheta_{i^*+1}/2+1/2-s\geq \frac{1}{4}$ because $s\leq 1/4$. Also, $\tilde\Delta_{i^*+1}=-\sqrt{2}\Delta_{i^*}$ and $\tilde\Delta_{i^*+2}=\Delta_{i^*+2}$. Then,
\begin{align*}
 \frac{1}{\tilde\vartheta^{(s)}_{i^*+1}}\frac{1}{\tilde\Delta_{i^*+1}^2\wedge \tilde\Delta_{i^*+2}^2} & = \frac{1}{(\vartheta_{i^*+1}/2+1/2-s)} \frac{1}{{(-\sqrt{2}\Delta_{i^*})}^2\wedge \Delta_{i^*+2}^2} \\
& \leq \frac{4}{\Delta_{i^*}^2\wedge \Delta_{i^*+2}^2}\leq 4 \max_{i=1,\dots,m} \frac{1}{\Delta_i^2}\leq 4 H_{\mathrm{detect}}^{\nu} \enspace,
\end{align*}
where one uses that $\displaystyle \max_{i=1,\dots,m} \frac{1}{\Delta_i^2}\leq H_{\mathrm{detect}}^{\nu}$ because $\vartheta_i\leq 1$ for all  $i=1,\dots,m-1$.

The inequality $H_{\mathrm{detect}}^{\nu} \leq H_{\mathrm{detect}}^{\tilde\nu^{(s)}}\leq 2H_{\mathrm{detect}}^{\nu}$ follows from the previous inequalities, observing that the case $i=i^*$ gives the lower bound and the other cases give the upper bound.
\end{proof}

\begin{lemma}\label{lem:TV_bound}
It holds that
\begin{equation*}
\mathbb{\bar P}_{\pi}(\xi)\leq 2\delta\enspace.
\end{equation*}
\end{lemma}
\begin{proof}[Proof of Lemma~\ref{lem:TV_bound}]
Let $e=(1/8-\eta)>0$, which is positive because $\eta\leq 1/8$. By definition of $\tilde l$ and $\tilde r$, one has $\tilde r - \tilde l = \tilde \vartheta + 1/4 > 2\eta + 2e$. In particular, the points $\tilde l+e$ and $\tilde r-e$ are spaced by at least $2\eta$, so that $\tilde l+e+\eta< \tfrac{1}{2}(\tilde l + \tilde r)<\tilde r - e - \eta$.

Let $\epsilon>0$ and consider $\bar{\nu}^{(1)}_{\epsilon}$ as the environment obtained from $\bar{\nu}$ by adding a plateau of height $\epsilon$, and length $e$ on the interval $[\tilde l,\tilde{l}+e]$. That is, we consider a change point of jump magnitude $+\epsilon$ at $\tilde{l}$, and another one of jump $-\epsilon$ at $\tilde l +e$. Consider also $\bar{\nu}^{(2)}_{\epsilon}$ as the environment obtained from $\bar{\nu}$ by adding a plateau of height $\epsilon$, and length $e$ on the interval $[\tilde r-e,\tilde{r}]$. Note that $\bar{\nu}^{(1)}_{\epsilon}$ and $\bar{\nu}^{(2)}_{\epsilon}$ have $m$ valid change points. Moreover, when $\epsilon$ goes to zero, the total variation distance between $\bar{\nu}^{(1)}_{\epsilon}$ and $\bar{\nu}$, and between $\bar{\nu}^{(2)}_{\epsilon}$ and $\bar{\nu}$ goes to zero\footnote{Using the same argument as in Lemma~\ref{lem:TV_bound_alternative}}. Denote as $\mathbb{P}_{\epsilon}^{(1)}$ and $\mathbb{P}_{\epsilon}^{(2)}$ the probability distributions of the observations when the environment is $\bar{\nu}^{(1)}_{\epsilon}$ and $\bar{\nu}^{(2)}_{\epsilon}$, respectively, and when the algorithm is $\pi$.

Consider the decomposition of the event $\xi$ as
\[
\xi=\left\{\mathcal{T}_{\pi}\leq \chi,\; \frac{\hat{x}_{i^*}+\hat{x}_{i^*+1}}{2}\leq \frac{\tilde l +\tilde r}{2}\right\}\sqcup \left\{\mathcal{T}_{\pi}\leq \chi,\; \frac{\hat{x}_{i^*}+\hat{x}_{i^*+1}}{2}> \frac{\tilde l +\tilde r}{2} \right\}
\]

Under $\bar \nu^{(1)}_{\epsilon}$, the change point $i^*+1$ is at position $\tilde l + e$. Moreover, it holds that $\tilde l + e +\eta< \frac{\tilde l + \tilde r}{2}$.  Then, by the $(\delta,\eta,m)$-correctness of $\pi$, it holds that
\begin{equation}
\mathbb{P}_{\epsilon}^{(1)}\left(\frac{\hat{x}_{i^*}+\hat{x}_{i^*+1}}{2}>\frac{\tilde l +\tilde r}{2}\right) \leq \mathbb{P}_{\epsilon}^{(1)}\left(\hat{x}_{i^*+1}>(\tilde l + e) +\eta\right)
 \leq \delta\enspace.
\end{equation}
Similarly, under $\bar \nu^{(2)}_{\epsilon}$, the change point $i^*$ is at position $\tilde r - e$. Moreover, it holds that $\tilde r - e - \eta\geq \frac{\tilde l + \tilde r}{2}$. Then, by the $(\delta,\eta)$-correctness of $\pi$, it holds that
\begin{equation}
\mathbb{P}_{\epsilon}^{(2)}\left(\frac{\hat{x}_{i^*}+\hat{x}_{i^*+1}}{2}\leq\frac{\tilde l +\tilde r}{2}\right)  \leq \mathbb{P}_{\epsilon}^{(2)}\left(\hat{x}_{i^*}<(\tilde r - e) -\eta\right)  \leq \delta\enspace.
\end{equation}

Overall, we have that, taking the limit $\epsilon\to 0$,
\begin{align*}
\mathbb{\bar P}_{\pi}(\xi) & = \lim_{\epsilon\to 0}\left\{ \mathbb{P}_{\epsilon}^{(1)}\left(\mathcal{T}_{\pi}\leq \chi,\dfrac{\hat{x}_{i^*}+\hat{x}_{i^*+1}}{2}>\dfrac{\tilde l +\tilde r}{2}\right) + \mathbb{P}_{\epsilon}^{(2)}\left(\mathcal{T}_{\pi}\leq \chi,\frac{\hat{x}_{i^*}+\hat{x}_{i^*+1}}{2}\leq\frac{\tilde l +\tilde r}{2}\right)\right\}\\
& \leq \lim_{\epsilon\to 0}\mathbb{P}_{\epsilon}^{(1)}\left(\frac{\hat{x}_{i^*}+\hat{x}_{i^*+1}}{2}>\frac{\tilde l +\tilde r}{2}\right) + \lim_{\epsilon\to 0}\mathbb{P}_{\epsilon}^{(2)}\left(\frac{\hat{x}_{i^*}+\hat{x}_{i^*+1}}{2}\leq\frac{\tilde l +\tilde r}{2}\right)\\
& \leq 2\delta\enspace,
\end{align*}
which concludes the proof of Lemma~\ref{lem:TV_bound}.
\end{proof}

\section{Relating the continuous and the discrete bandit change point problems}\label{sec:continuousdiscrete}
Assume we have an algorithm that is able to localize $N$ change points in a $K$ armed bandit change point problem. This means that if we have $K$ arms with corresponding $1$-subGaussian distribution $\tilde\nu_k$ and mean $\tilde\mu_k$ for $k\in[K]$ and there is $\mathcal{C}=\{1\leq k_1^*<k_2^*<\dots<k_m^*<K\}$ such that $\tilde\mu_k\neq \tilde\mu_{k+1}$ for $k\in\mathcal{C}$, the algorithm is able to identify a subset of $\mathcal{C}$ with cardinality $m$.

When we are now dealing with a continuous bandit change point problem as introduced in Section~\ref{sec:setting}, we can still tackle it by using such an algorithm.

If we want to localize change points in $[0,1]$ with precision $\eta$, let us first build a discrete bandit, by considering $\tilde \nu_k=\nu_{(k-1)\eta}$ for $k=1,\dots,\lfloor1/\eta\rfloor+1$. If we now localize a change point $k^*$, it means that $\tilde \mu_{k^*}\neq \tilde\mu_{k^*+1}$ and therefore $\mu_{(k^*-1)\eta}\neq\mu_{k^*\eta}$. This implies, that in the continuous setting there must be a change point $x^*$ with $(k^*-1)\eta< x^*\leq k^*\eta$, and using $c=(k^*-1)\eta$ guarantees that $|x^*-c|\leq \eta$.

While we use this discretization throughout our experiments, we want to remark that conversely it is possible to transform a discrete change point identification problem for $K$ arms to a continuous problem by constructing a respective step function. Here, localizing change points with precision $\eta<1/2K$ leads to exact change point identification in the discrete case.

\end{document}